\documentclass[11pt]{article}

\usepackage[final]{acl}

\usepackage{times}
\usepackage{latexsym}

\usepackage[T1]{fontenc}

\usepackage[utf8]{inputenc}

\usepackage{microtype}

\usepackage{inconsolata}

\usepackage{graphicx}
\usepackage{amsmath}
\usepackage{amsthm}
\usepackage{booktabs}
\usepackage{multicol}
\usepackage{adjustbox}
\usepackage{subcaption}
\usepackage{multirow}
\usepackage{algpseudocode}
\usepackage{algorithm}
\usepackage{amssymb}
\usepackage{makecell}
\usepackage[multiple]{footmisc}
\usepackage{resizegather}
\usepackage{microtype}
\usepackage{bbm}
\usepackage{enumitem}
\usepackage{hyperref}
\usepackage{dsfont}
\usepackage{bbold}
\usepackage{float}

\usepackage{geometry}

\newtheorem{theorem}{Theorem}
\newtheorem{proposition}{Proposition}

\theoremstyle{definition}
\newtheorem{assumption}{Assumption}

\title{Which Metrics Save the Most Human Annotation? \\ Prediction-Powered Evaluation and Meta-Evaluation}

\author{Mingqi Gao$^{1}$\qquad Anthony Sicilia$^{2}$\qquad Weiyan Shi$^{1}$\vspace{6pt}\\
Northeastern University$^{1}$\quad West Virginia University$^{2}$\vspace{4pt}\\
\texttt{\{gao.mingqi, we.shi\}@northeastern.edu, \quad  anthony.sicilia@mail.wvu.edu}}

\begin{document}
\maketitle
\begin{abstract}
Across various non-verifiable tasks, human evaluation is reliable but expensive, while automatic metrics are more scalable but often biased. Building on prediction-powered inference (PPI), we propose \textit{prediction-powered evaluation}, a framework that combines limited human judgments with large-scale automatic scores to obtain data-efficient system comparisons that are provably unbiased. We develop parametric and non-parametric procedures, analyze the efficiency trade-off between paired and unpaired designs, and validate the framework on six WMT datasets. We further introduce the \textit{Prediction-Powered Saving Ratio (PPSR)}, a meta-metric that measures how much human annotation an automatic metric can save when used within prediction-powered evaluation. PPSR directly targets metric utility for prediction-powered evaluation and yields more discriminative and stable metric rankings than existing system-level meta-metrics. Overall, our new paradigm reframes automatic metrics as tools for reducing human annotation cost rather than replacing human judgment, and applies broadly to non-verifiable tasks\footnote{Our code is available at \url{https://github.com/CHATS-lab/ppi-eval}}.
\end{abstract}

\section{Introduction}
\label{sec:intro}

System comparison is a central problem in AI evaluation: we often want to determine which model or system performs better, and by how much. For non-verifiable tasks such as machine translation (MT), where output quality is inherently subjective \citep{DBLP:journals/corr/abs-2506-00103}, human evaluation is typically regarded as the gold standard. However, human evaluation is expensive, resulting in limited sample sizes, which often leave system comparisons statistically underpowered \citep{DBLP:conf/emnlp/CardHKJMJ20,DBLP:conf/emnlp/HowcroftR21}. In contrast, automatic metrics, including LLM-based judges, can be applied at large scale and low cost \citep{DBLP:conf/acl/WeiJ20}. However, automatic metrics do not always match gold standard human evaluations. These metrics are unreliable to use alone, but it is wasteful to ignore them entirely.

We argue that this dilemma arises largely because automatic evaluation has traditionally been framed as a proxy for human evaluation \citep{DBLP:conf/acl/MathurBC20}. As a result, evaluation is typically conducted in a \textit{human-only} or \textit{auto-only} manner, even when both are reported separately. Relatively little work combines the two within a single procedure, such as using human judgments to debias automatic evaluation \citep{DBLP:conf/emnlp/GaoXWC24}. Prediction-powered inference (PPI,  \citealp{doi:10.1126/science.adi6000,DBLP:journals/corr/abs-2311-01453,DBLP:journals/corr/abs-2501-09731,DBLP:journals/corr/abs-2603-16041}), a line of work from the statistical literature, offers a formal way to combine human and automatic evaluation. Specifically, it shifts the goal of automatic evaluation from \textit{human approximation} to \textit{variance reduction}. This increases statistical power and lowers the amount of human annotation needed for accurate system comparison.

\begin{figure*}
    \centering
    \includegraphics[width=\linewidth]{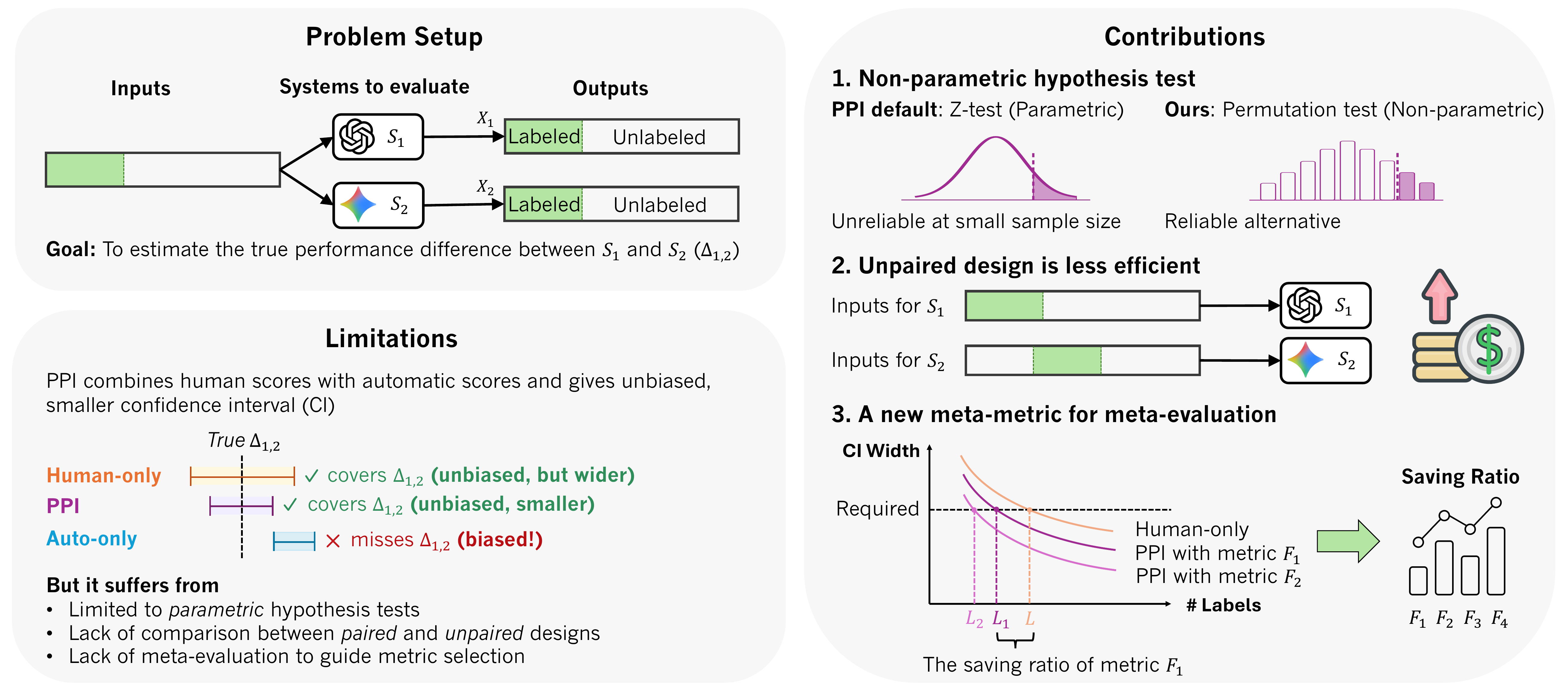}
    \caption{\small An illustration of the problem setup, the limitations of the default PPI methods, and our contributions.}
    \label{fig:teaser}
    \vspace{-5mm}
\end{figure*}

However, applying PPI to MT system comparison is not straightforward. We identify three design axes that have long been central to the MT community. The first is \textit{parametric vs. non-parametric inference}: standard PPI yields parametric hypothesis tests and confidence intervals, whereas WMT shared tasks have traditionally relied on non-parametric procedures, such as the sign test \citep{DBLP:conf/wmt/Callison-BurchK12}, bootstrap resampling \citep{DBLP:conf/wmt/BojarCFGHHJKLMN16}, and the Wilcoxon signed-rank test \citep{kocmi-etal-2025-findings}. The second is \textit{paired vs. unpaired designs}: standard PPI naturally leads to a paired design, while WMT shared tasks use both paired designs \citep{DBLP:conf/wmt/KocmiABBDFFFGGH24,kocmi-etal-2025-findings} and unpaired designs \citep{DBLP:conf/wmt/BarraultBBCFGGH20,DBLP:conf/wmt/AkhbardehABBCCC21,DBLP:conf/wmt/KocmiBBDFFGGGHKKMMNNNPP22,DBLP:conf/wmt/KocmiABBDFFFGGH23} depending on year. The last is \textit{meta-evaluation}---given many automatic metrics to choose from, which should be used? While existing meta-metrics, such as SPA \citep{DBLP:conf/wmt/0001MDK24}, answer this question by minimizing disagreement with human-evaluation, these are ill-formed for PPI because PPI naturally produces unbiased estimates.

To address these design axes, we make the following contributions. \textbf{(1)} We apply PPI to system comparison with pointwise automatic metrics (\S\ref{sec:method-paired-param}), and we validate its effectiveness on six WMT datasets (\S\ref{sec:exp-basic}). Prediction-powered evaluation is unbiased, making it more reliable than automatic-only evaluation, and yet more data-efficient than human-only evaluation. \textbf{(2)} Motivated by traditional WMT evaluation, we propose a non-parametric hypothesis test (\S\ref{sec:method-non-param}) and empirically validate it (\S\ref{sec:exp-nonparam}). This approach addresses settings where the approximation conditions required by standard PPI fail to hold (e.g., when the sample size is small). \textbf{(3)} We provide a theoretical comparison of the data efficiency of paired and unpaired designs, identifying key drivers of their variance differences, and empirically show that the paired design is generally superior on WMT (\S\ref{sec:exp-gap}). \textbf{(4)} Building on our data-efficiency analysis, we propose the \textbf{Prediction-Powered Saving Ratio} (PPSR), a new meta-metric for prediction-powered evaluation (\S\ref{sec:ppsr}). This meta-metric reframes traditional meta-evaluation in terms of the amount of human annotation saved by a given automatic metric, rather than focusing on how well automatic metrics align with human annotations. PPSR exhibits better discriminative power (\S\ref{sec:exp-ppsr-disc}) and ranking stability (\S\ref{sec:exp-ppsr-stability}) than existing system-level meta-metrics. 

More broadly, prediction-powered evaluation reframes the long-standing trade-off between cheap, biased automatic evaluation and expensive, gold-standard human evaluation as a smooth annotation-savings curve, with PPSR providing the principled instrument to read off where any given automatic metric sits on that curve. Notably, the framework applies beyond MT to various non-verifiable tasks.

\section{Related Work}

\paragraph{PPI for evaluation.}
Existing PPI-based evaluation work divides by whether the automatic metric is \emph{pairwise}, targeting statistics such as win rate and Bradley--Terry scores \citep{DBLP:conf/nips/ChatziSTR24,DBLP:conf/icml/BoyeauALYMJ25,DBLP:conf/icml/Zhou0Z25}, or \emph{pointwise}, where prior work  estimates the quality score of a single model \citep{chaganty-etal-2018-price,DBLP:conf/naacl/Saad-FalconKPZ24,DBLP:conf/nips/FischMHDGC24,DBLP:conf/iclr/DornerNH25,DBLP:journals/corr/abs-2509-22957,DBLP:journals/corr/abs-2601-05420}. A closely related line uses Rogan--Gladen \citep{Rogan1978EstimatingPF}, a different statistical framework for the same targets \citep{DBLP:journals/corr/abs-2511-21140,DBLP:journals/corr/abs-2601-05420,DBLP:journals/corr/abs-2601-20913}. We focus on the pointwise case, the dominant class in MT/NLG, but our target is the score \emph{difference} between systems, which is what MT system comparison requires and which the above lines do not directly address. We also study two design axes crucial to MT evaluation: parametric vs.\ non-parametric and paired vs.\ unpaired.

\paragraph{Hypothesis tests in NLP evaluation.}
\citet{DBLP:conf/acl/ReichartDBS18} formalized algorithm comparison in NLP as hypothesis testing. The dominant parametric choice for MT/NLG is the paired $t$-test \citep{DBLP:journals/csl/LeeGMK21}; non-parametric alternatives including Wilcoxon signed-rank \citep{kocmi-etal-2025-findings} and rank-sum \citep{DBLP:conf/wmt/KocmiABBDFFFGGH23}, paired bootstrap \citep{DBLP:conf/emnlp/Koehn04}, and the paired permutation test (also known as the approximate randomization test \citep{DBLP:conf/wmt/GrahamMB14}), are widely used and often preferred. All these tests are designed for human-only or auto-only settings. The prediction-powered tests we study are designed for the combined setting and cover both parametric and non-parametric variants.

\paragraph{Meta-metrics in MT and NLG.}
Numerous meta-metrics have been proposed for meta-evaluating automatic metrics, combining various grouping strategies and agreement functions \citep{DBLP:conf/naacl/GaoHLW25,DBLP:conf/emnlp/DeutschFF23,DBLP:conf/emnlp/Perrella0CBN24,DBLP:conf/emnlp/DiIanniD25}. For system comparison, the common system-level choices, including Pearson's $r$ \citep{DBLP:conf/wmt/MaWBG19}, Spearman's $\rho$ \citep{DBLP:conf/wmt/MachacekB13}, Kendall's $\tau$ \citep{DBLP:conf/wmt/MathurWFMB20}, pairwise accuracy \citep{DBLP:conf/wmt/KocmiFGJMM21}, and soft pairwise accuracy (SPA) \citep{DBLP:conf/wmt/0001MDK24}, all directly measure agreement with human judgments. PPSR is fundamentally different: derived from PPI, it quantifies the proportion of human judgments an automatic metric can save in system comparison while preserving the same statistical power. Quantities analogous to PPSR appear in \citet{chaganty-etal-2018-price} and \citet{DBLP:conf/icml/Zhou0Z25} from a control-variate perspective, but are not formalized as meta-metrics.

\section{Problem and Background}
\label{sec:problem_background}

\subsection{System Comparison}

Suppose there are $N$ systems to be compared, denoted by $\{S_i\}_{i=1}^N$, and $M$ source inputs (i.e. segments), denoted by $\{Q_j\}_{j=1}^M$. The output of the $i$-th system on the $j$-th input is denoted by $X_{ij}=S_i(Q_j)$.
For all the system outputs on $L$ of the inputs, we conduct human evaluation and obtain their human scores $Y_{ij}=H(X_{ij})$, denoted by $\{Y_{ij}\}_{i=1,j=1}^{N,L}$. The human scores of the remaining $U=M-L$ outputs of each system are not available. In contrast, we obtain metric scores for all system outputs using a pointwise automatic metric $F$, denoted by $\{F(X_{ij})\}_{i=1,j=1}^{N,M}$\footnote{For convenience, we omit possible additional inputs to $F$ other than $X$, such as the source input $Q$.}.

For any pair of systems, such as $S_1$ and $S_2$, our goal is to estimate their true human score difference $\delta_{1,2}=\mathbb{E}[Y_1-Y_2]$ \citep{DBLP:conf/acl/WeiJ20}. \textbf{Human-only evaluation} uses only human scores:
\begin{equation}
\widehat{\delta_{1,2}^{H}}=\tfrac{1}{L}\sum\nolimits_{j=1}^L (Y_{1j} - Y_{2j}).
\end{equation}
If $\{Y_{1j}\}_{j=1}^L$ are regarded as independent and identically distributed (i.i.d.) random variables and $\{Y_{2j}\}_{j=1}^L$ are regarded as i.i.d. random variables\footnote{Here $Y_1$ and $Y_2$ are random variables with the same distribution as $\{Y_{1j}\}_{j=1}^L$ and $\{Y_{2j}\}_{j=1}^L$ separately. The same applies to $F(X_1)$ and $F(X_2)$.}, it's easy to know  $\widehat{\delta_{1,2}^{H}}$ is an unbiased estimator with $\mathrm{Var}[\widehat{\delta_{1,2}^{H}}] =\frac{1}{L}\mathrm{Var}[Y_1-Y_2]$.

\textbf{Auto-only evaluation} uses only metric scores:
\begin{equation}
\widehat{\delta_{1,2}^{A}}
=
\tfrac{1}{M}\sum\nolimits_{j=1}^M
\big(F(X_{1j}) - F(X_{2j})\big).
\end{equation}
This estimator is \textbf{biased} because the automatic metric $F$ is generally biased. Its variance $\mathrm{Var}[\widehat{\delta_{1,2}^{A}}] = \frac{1}{M}\mathrm{Var}[F(X_{1})-F(X_{2})]$ is always small because $M$ is generally very large.

\subsection{Meta-Evaluation: Metric Comparison}

The comparisons above focus on MT systems. We now turn to meta-evaluation, where the objects of comparison are automatic metrics. Suppose that we wish to compare $K$ automatic metrics, denoted by $\{F_k\}_{k=1}^{K}$. For each metric $F_k$, a meta-metric $C$ evaluates its performance based on the metric scores $\{F_k(X_{ij})\}_{i=1,j=1}^{N,L}$ and the corresponding human scores $\{Y_{ij}\}_{i=1,j=1}^{N,L}$. Formally, the meta-evaluation score of $F_k$ can be written as $C\!\left(\left\{ \bigl(Y_{ij}, F_k(X_{ij})\bigr) \right\}_{i=1,j=1}^{N,L}\right).$ Typically, $C$ returns a a scalar score, with larger values indicating better performance of $F_k$ with respect to the quantity measured by $C$. These meta-evaluation scores can therefore be used both to compare pairs of automatic metrics and to produce an overall ranking among them. Important criteria for evaluating a meta-metric include the interpretability of the quantity it measures \citep{DBLP:conf/emnlp/Perrella0CBN24}, its ability to discriminate among automatic metrics, and the stability of the resulting metric rankings \citep{DBLP:conf/wmt/0001MDK24}.

\subsection{Prediction-Powered Inference}
\label{sec:bg-ppi}

Prediction-powered inference (PPI) combines a small labeled sample with a larger unlabeled sample equipped with predictions from an arbitrary predictive model \citep{doi:10.1126/science.adi6000}. Consider a generic estimation problem with target $\theta = \mathbb{E}[Z]$, where $Z$ is expensive to observe. Let $W$ be cheaply observed information and let $g(W)$ be a prediction of $Z$. Suppose we observe $L$ labeled examples $\{(W_j,Z_j)\}_{j=1}^{L}$ and $U$ additional unlabeled examples $\{W_j\}_{j=L+1}^{L+U}$, for which only $g(W_j)$ is available. A classical estimator of $\theta$ that only uses labeled examples is $\widehat{\theta^{H}}=\frac{1}{L}
\sum_{j=1}^{L} Z_j$. In contrast, a prediction-powered estimator of $\theta$ is

{\small
\begin{equation}
\widehat{\theta^{PP}}
=
\frac{\lambda}{U}
\sum_{j=L+1}^{L+U}
g(W_j)
+
\frac{1}{L}
\sum_{j=1}^{L}
\left(
Z_j-\lambda g(W_j)
\right),
\label{eq:generic_ppi}
\end{equation}
}
where $\lambda$ is a tunable parameter. $\widehat{\theta^{PP}}$ remains unbiased for $\theta$ for any fixed $\lambda$, even when $g$ is biased. An appropriate choice of $\lambda$ can guarantee that $\mathrm{Var}[\widehat{\theta^{PP}}]\leq \mathrm{Var}[\widehat{\theta^{H}}]$ \citep{DBLP:journals/corr/abs-2311-01453}, which means smaller confidence intervals and hypothesis tests with higher power.

\section{Prediction-Powered Evaluation for MT}
\label{sec:method}

We now instantiate the generic PPI estimator from \S\ref{sec:bg-ppi} for paired MT system comparison and derive the corresponding parametric confidence intervals and hypothesis tests (\S\ref{sec:method-paired-param}). We then present two extensions: a paired non-parametric hypothesis test (\S\ref{sec:method-non-param}) and an instantiation of PPI for unpaired MT system comparison (\S\ref{sec:method-unpaired}).

\subsection{Paired Parametric Design (PPI Default)}
\label{sec:method-paired-param}

Map the generic PPI setup of \S\ref{sec:bg-ppi} to system comparison by setting $Z_j = Y_{1j} - Y_{2j}$ (the expensive human score difference) and $g(W_j) = F(X_{1j}) - F(X_{2j})$ (the cheap metric score difference). The target $\theta = \mathbb{E}[Z] = \delta_{1,2}$ is the true human score difference. Substituting into \eqref{eq:generic_ppi} yields the \textbf{prediction-powered paired estimator}

{\small
\begin{equation}
\begin{split}
& \widehat{\delta_{1,2}^{PP}} =
\frac{\lambda}{U}\!\!\!\sum_{j=L+1}^{L+U}\!\!\!
\big(F(X_{1j}) - F(X_{2j})\big) \\
& + \frac{1}{L}\sum_{j=1}^{L}
\Big[(Y_{1j} - Y_{2j}) - \lambda\big(F(X_{1j}) - F(X_{2j})\big)\Big].
\end{split}
\label{eq:ppi-paired}
\end{equation}
}
At $\lambda = 0$, \eqref{eq:ppi-paired} reduces to the human-only estimator $\widehat{\delta_{1,2}^{H}}$; at $\lambda = 1$, it equals the auto-only estimator on the unlabeled set plus a rectifier term estimated on the labeled set. For any fixed $\lambda$, $\widehat{\delta_{1,2}^{PP}}$ is \textbf{unbiased} for $\delta_{1,2}$. Proposition~\ref{prop:paired-mean-var} in Appendix~\ref{app:paired-theory} provides the proof. Human scores and automatic scores need not share a scale (e.g., 0--100 vs.\ 0--1).

\paragraph{Variance and optimal $\lambda$.} Following
\citet{DBLP:journals/corr/abs-2311-01453,DBLP:journals/corr/abs-2603-16041}, the variance-minimizing tuning parameter $\lambda^\star$ has a closed form (Proposition~\ref{prop:paired-opt} in Appendix~\ref{app:paired-theory} provides the proof):
$$\lambda^\star=\frac{\mathrm{Cov}[Y_1 - Y_2,F(X_1) - F(X_2)]}{\bigl(1+\frac{L}{U}\bigr)\mathrm{Var}[F(X_1) - F(X_2)]},$$ yielding

{\small
\begin{equation}
\begin{split}
& \mathrm{Var}\!\left[\widehat{\delta_{1,2}^{PP}}\right]
= \frac{1}{L}\mathrm{Var}[Y_1 - Y_2] \\
& \qquad -\frac{U \big(\mathrm{Cov}[Y_1 - Y_2,\, F(X_1) - F(X_2)]\big)^2}
     {L(L+U)\, \mathrm{Var}[F(X_1) - F(X_2)]}.
\end{split}
\label{eq:ppi-var}
\end{equation}
}
The second term in \eqref{eq:ppi-var} is the variance reduction PPI buys over the human-only paired estimator. It is non-negative, and vanishes if and only if the metric score differences are uncorrelated with the human score differences. Thus we have $\mathrm{Var}[\widehat{\delta_{1,2}^{PP}}] \le
\mathrm{Var}[\widehat{\delta_{1,2}^{H}}]$, so PPI is never worse than human-only evaluation under the optimal $\lambda^\star$, even when $F$ is highly biased or weakly informative. 

In practice, the optimal value $\lambda^\star$ cannot be computed directly because $\mathrm{Cov}[Y_1-Y_2, F(X_1)-F(X_2)]$ and $\mathrm{Var}[F(X_1)-F(X_2)]$ are unknown. Following \citet{DBLP:journals/corr/abs-2311-01453}, we therefore replace $\lambda^\star$ with its empirical estimate $\hat{\lambda}$, computed using the sample covariance $\widehat{\mathrm{Cov}}[Y_1-Y_2, F(X_1)-F(X_2)]$ and sample variance $\widehat{\mathrm{Var}}[F(X_1)-F(X_2)]$. Note that we do not need to clip $\hat \lambda$ to $[0, 1]$ here.

\paragraph{Confidence intervals and hypothesis test.} Applying the central limit theorem yields the $100(1-\alpha)\%$ CI for human-only and prediction-powered cases\footnote{Note that for auto-only evaluation, we can have similar confidence intervals and hypothesis tests (Alg.~\ref{alg:auto-only-pair-z-test} in the appendix), but neither provides statistical guarantees.}:

{\small
\begin{equation*}
\widehat{\delta_{1,2}^{g}} \pm z_{1-\alpha/2}\sqrt{\widehat{\mathrm{Var}}
[\widehat{\delta_{1,2}^{g}}]}, \quad g=\{H, PP\},
\end{equation*}
}
where $\widehat{\mathrm{Var}}[\cdot]$ plugs
sample variance. Theorem~\ref{thm:paired-ci} and Theorem~\ref{thm:paired-plugin} in Appendix~\ref{app:paired-theory} provide the proof. The corresponding parametric paired $Z$-tests are summarized as Alg.~\ref{alg:human-only-pair-z-test} and  Alg.~\ref{alg:prediction-power-pair-z-test} in the Appendix.

\subsection{Paired Non-Parametric Test}
\label{sec:method-non-param}

The parametric CI and $Z$-test of \S\ref{sec:method-paired-param} inherit the standard PPI asymptotic guarantee only when both (1) $\widehat{\delta_{1,2}^{PP}}$ is well approximated by a normal distribution and (2) the population variance and population covariance are known or estimated with high accuracy. In MT system comparison this is not always the case: human scores are discrete, and the sample size is small, etc. WMT shared tasks have for two decades preferred non-parametric procedures possibly for these reasons, but the existing prediction-powered literature has no off-the-shelf non-parametric test for paired system comparison\footnote{\citet{DBLP:journals/corr/abs-2405-18379} proposed using the bootstrap to construct prediction-powered confidence intervals, but it cannot be directly converted into $p$-values.}.

To fill this gap, we extend the classical paired permutation test (Alg.~\ref{alg:human-only-pair-perm-test} and Alg.~\ref{alg:Auto-only-pair-perm-test} in the appendix) \citep{DBLP:conf/wmt/GrahamMB14} to the prediction-powered setting (Alg.~\ref{alg:prediction-power-pair-perm-test} in the appendix). Unlike the $Z$-test of \S\ref{sec:method-paired-param}, the permutation test avoids the normal approximation but requires stronger symmetry assumptions. Beyond the null hypothesis $\delta_{1,2}\leq 0$, the human-only test assumes that $Y_1-Y_2$ is symmetric about $\delta_{1,2}$ \citep{permutation_textbook}, while the prediction-powered test assumes that $(Y_1-Y_2,F(X_1)-F(X_2))$ is jointly centrally symmetric about $(\delta_{1,2},\mathbb{E}[F(X_1)-F(X_2)])$. We state these assumptions as ``asserts'' in the algorithms.

\subsection{Unpaired Design}
\label{sec:method-unpaired}

Sections~\ref{sec:method-paired-param} and~\ref{sec:method-non-param} assume that both systems are evaluated on the same source inputs, and $\widehat{\delta_{1,2}^{PP}}$ is built from per-segment differences. The unpaired design is the natural fallback when the two systems are evaluated on disjoint inputs.

\paragraph{Unpaired estimators.} Let $I_i^L, I_i^U$ index the labeled and
unlabeled inputs for system $S_i$ (the paired design requires
$I_1^L = I_2^L, I_1^U = I_2^U$). The unpaired human-only estimator,

{\small
\begin{equation*}
\widehat{\delta_{1,2}^{H,\mathrm{Un}}}
=
\frac{1}{|I_1^L|}\sum_{j\in I_1^L} Y_{1j}
-
\frac{1}{|I_2^L|}\sum_{j\in I_2^L} Y_{2j},
\end{equation*}
}
treats the two systems' scores as independent. Similarly, the unpaired
prediction-powered estimator instantiates the generic PPI of
\eqref{eq:generic_ppi} separately for each system, with two tuning
parameters $\lambda_1, \lambda_2$:

{\small
\begin{equation*}
\begin{split}
\widehat{\delta_{1,2}^{PP,\mathrm{Un}}}
&=
\frac{\lambda_1}{|I_1^U|}\!\sum_{j\in I_1^U}\!\!\! F(X_{1j})
+
\frac{1}{|I_1^L|}\!\sum_{j\in I_1^L}\!\!\bigl(Y_{1j} - \lambda_1 F(X_{1j})\bigr) \\
&\quad -
\frac{\lambda_2}{|I_2^U|}\!\sum_{j\in I_2^U}\!\!\! F(X_{2j})
-
\frac{1}{|I_2^L|}\!\sum_{j\in I_2^L}\!\!\bigl(Y_{2j} - \lambda_2 F(X_{2j})\bigr).
\end{split}
\end{equation*}
}
Both estimators are unbiased. The variance-minimizing
$(\lambda_1^\star, \lambda_2^\star)$ have closed forms analogous to
\S\ref{sec:method-paired-param}. Similarly, we have

{\small
\begin{equation*}
\mathrm{Var}\left[\widehat{\delta_{1,2}^{PP,\mathrm{Un}}}\right]
\leq
\mathrm{Var}\left[\widehat{\delta_{1,2}^{H,\mathrm{Un}}}\right].
\end{equation*}
}
The proof is in Proposition~\ref{prop:unpaired} (Appendix~\ref{app:unpaired}). The unpaired design also yields confidence intervals and hypothesis tests, which we omit for brevity.

However, whether the paired or unpaired design is more efficient remains theoretically ambiguous. Proposition~\ref{prop:gap} in Appendix~\ref{app:paired_vs_unpaired} proves that for both human-only and prediction-powered settings, the relative efficiency depends on the covariance structure of the data. Therefore, we investigate this question empirically in \S\ref{sec:exp-gap}.

\section{PPSR: A New Meta-Metric}
\label{sec:ppsr}

In \S\ref{sec:method}, we restrict our focus to a single automatic metric $F$. In practice, however, we have a set of metrics $\{F_k\}_{k=1}^K$	to choose from. In this section, we derive a meta-metric to determine which of these automatic metrics are most effective for prediction-powered evaluation (\S\ref{sec:ppsr-def}), interpret it as an annotation saving ratio (\S\ref{sec:ppsr-interp}), and contrast it with existing meta-metrics (\S\ref{sec:ppsr-relation}).

\subsection{Derivation and Definition}
\label{sec:ppsr-def}

When $U \gg L$, the relative variance reduction in \eqref{eq:ppi-var} simplifies cleanly:
\begin{equation*}
\begin{split}
\frac{\mathrm{Var}[\widehat{\delta_{1,2}^{H}}] -
      \mathrm{Var}[\widehat{\delta_{1,2}^{PP}}]}
     {\mathrm{Var}[\widehat{\delta_{1,2}^{H}}]}
\;\approx\;
\mathrm{Corr}\!\left[Y_1 - Y_2,\; F_k(X_1) - F_k(X_2)\right]^2. 
\end{split}
\label{eq:ppsr-pair}
\end{equation*}
Thus, for a given pair of systems, the relative variance reduction achieved by prediction-powered evaluation over human-only evaluation is approximately equal to the squared Pearson correlation between the human and metric \emph{score differences} for that pair. Moreover, this relative variance reduction is equal to the fraction of human annotations saved by using metric $F_k$, relative to human-only evaluation, while maintaining the same variance. Proposition~\ref{prop:savings} in Appendix~\ref{app:savings} provides the proof. This result can also be viewed as an application of the rule of thumb proposed by \citet{DBLP:journals/corr/abs-2603-16041}.

Averaging this quantity over all $\binom{N}{2}$ pairs of systems and replacing the population Pearson correlations with their sample Pearson estimates yields the \textbf{Prediction-Powered Saving Ratio (PPSR)}:
\begin{equation*}
\begin{split}
\binom{N}{2}^{-1}\sum_{p=1}^{N-1}\sum_{q=p+1}^{N} r\!\left(\left\{\left(Y_{pj}-Y_{qj},\; F_k(X_{pj})-F_k(X_{qj})\right)\right\}_{j=1}^{L}\right)^2.
\end{split}
\label{eq:ppsr}
\end{equation*}
PPSR takes values in $[0, 1]$, and requires only the labeled examples to compute.

\subsection{Interpretation}
\label{sec:ppsr-interp}
PPSR represents the average fraction of human annotations saved by using metric $F_k$ in prediction-powered evaluation, relative to human-only evaluation, while maintaining approximately the same confidence-interval width or test power across system pairs. A PPSR of $0.4$ for a metric $F_k$ on a given dataset means that replacing human-only evaluation with PPI using $F_k$ yields the same statistical conclusion with $\approx 40\%$ fewer human-annotated segments on average. Empirical resampling experiments in Appendix~\ref{app:annotation_saving_vs_ppsr} confirm that PPSR strongly correlates with actual human annotation savings.

PPSR should \textbf{not} be interpreted as an agreement measure because it squares the Pearson correlation and therefore discards its sign. This property highlights an advantage of prediction-powered evaluation over automatic-only evaluation: even a metric that is negatively correlated with human judgments can still reduce variance and save human annotations in prediction-powered evaluation. 

\subsection{Relation to Existing Meta-Metrics}
\label{sec:ppsr-relation}

Interestingly, though PPSR is a system-level meta-metric, that is, it reflects the ability of automatic metrics to perform system comparison, it is closely related to segment-level meta-metrics. Compared with existing schemes \citep{DBLP:conf/emnlp/DeutschFF23} including \textit{No Grouping}, \textit{Group by Item}, and \textit{Group by System}, PPSR can be viewed as a segment-level meta-metric under a \textit{Group by System Pair} scheme. Notably, PDP \citep{DBLP:conf/emnlp/DiIanniD25}, a new segment-level meta-metric, also uses the score differences between systems on the same segment, but its grouping strategy is very different from PPSR. Please read Appendix~\ref{app:comparison_segment_metametric} for more details.

Existing system-level meta-metrics typically aggregate scores at the system level first, by averaging human scores and metric scores over inputs for each system, and then compute a correlation between the resulting $N$ pairs of system-level scores: 

{\small
\begin{equation*}
c\!\left(\left\{\left(\frac{1}{L}\sum_{j=1}^L Y_{ij},\; \frac{1}{L}\sum_{j=1}^L F_k(X_{ij})\right)\right\}_{i=1}^{N}\right), 
\end{equation*}
}
where $c$ can be Pearson's $r$, Spearman's $\rho$, or Kendall's $\tau$\footnote{\citet{DBLP:conf/wmt/0001MDK24} showed that Pairwise Accuracy (PA) and system-level Kendall's $\tau_a$ are equivalent up to a linear scaling and shift. However, to handle possible ties, we consistently use $\tau_b$ for Kendall's $\tau$.}. These meta-metrics lack discriminative power and produce unstable metric rankings since the number of systems $N$ is typically small.

To address this, \citet{DBLP:conf/wmt/0001MDK24} proposed SPA, which leverages the discrepancy between $p$-values from human-only and auto-only paired permutation tests as additional signal:

{\small
\begin{equation*}
\binom{N}{2}^{-1}\sum_{p=1}^{N-1}\sum_{q=p+1}^{N} 1 - \left|p^H_{p,q} - p^A_{p,q}\right|,
\end{equation*}
}
It has been adopted in the latest WMT metric shared tasks \citep{DBLP:conf/wmt/FreitagMDLAR0BK24,lavie-etal-2025-findings}. However, since p-values are not effect sizes or calibrated probabilities, this is best understood as a heuristic measure of agreement.

\section{Experiments}
\label{sec:experiments}

We evaluate the framework on six WMT22--24 datasets and organize the experiments around the design axes and meta-evaluation question that structure the paper. After describing the datasets (\S\ref{sec:datasets}), we first verify that prediction-powered evaluation delivers on its basic promise for MT system comparison (\S\ref{sec:exp-basic}). We then examine the two design axes introduced in \S\ref{sec:method}, parametric vs.\ non-parametric test (\S\ref{sec:exp-nonparam}) and paired vs.\ unpaired design (\S\ref{sec:exp-gap}), to quantify the cost of each design choice. Finally, we assess the discriminative power (\S\ref{sec:exp-ppsr-disc}) and ranking stability (\S\ref{sec:exp-ppsr-stability}) of our new meta-metric, PPSR, and report its metric rankings (\S\ref{sec:exp-ppsr-result}).

\subsection{Datasets}
\label{sec:datasets}

We used the MT Metrics Eval V2 toolkit\footnote{\url{https://github.com/google-research/mt-metrics-eval}} to obtain all WMT data from 2022 to 2025. After carefully inspecting the data, we selected en-de, en-ru, and zh-en from WMT22; en-zh and ja-en from WMT23; and cs-uk from WMT24. These datasets contain relatively large numbers of source inputs across all systems, making them more suitable for treating the data as a population when evaluating confidence intervals and hypothesis tests. Table~\ref{tab:dataset_summary} summarizes the basic statistics of these datasets. For each input, the human score of each system output and the scores assigned by all automatic metrics are available. We use all available automatic metrics, so our dataset configuration is not identical to that used in the metric shared tasks. For reference-based metrics, we follow the dataset guidelines and use the scores computed with ref-A. For WMT24 cs-uk, we additionally exclude sentinel metrics and three pairs of automatic metrics that produce identical outputs.

\begin{table}[]
\centering
\footnotesize
\setlength{\tabcolsep}{4pt}
\begin{tabular}{lcrrr}
\toprule
Dataset & Human & \#Inputs & \#Sys. & \#Met. \\
\midrule
WMT22 en-de & MQM & 1315 & 14 & 31 \\
WMT22 en-ru & MQM & 1315 & 15 & 30 \\
WMT22 zh-en & MQM & 1875 & 14 & 31 \\
WMT23 en-zh & DA-SQM & 1098 & 15 & 34 \\
WMT23 ja-en & DA-SQM & 1120 & 17 & 35 \\
WMT24 cs-uk & ESA & 1955 & 11 & 25 \\
\bottomrule
\end{tabular}
\caption{\small Datasets used in the experiments. Sys.
and Met. denote the number of systems and automatic metrics.}
\label{tab:dataset_summary}
\vspace{-5mm}
\end{table}

\subsection{PPI Works for MT System Comparison}
\label{sec:exp-basic}
\paragraph{Setup.} For each system pair $S_p$ and $S_q$, we define the finite population as the full set of paired examples for which human scores are available, and take the mean human score difference over this full set as the population effect size $\delta_{p,q}$. The effect size is expressed in the original scale of the human scores and is not normalized. We then evaluate confidence intervals and hypothesis tests by resampling from this finite population. In each trial, we sample $U+L$ examples without replacement from the full set and randomly split them into $L$ labeled examples, whose human scores are used, and $U$ unlabeled examples, whose human scores are held out. The sampled data are used to construct confidence intervals and conduct hypothesis tests, which are then evaluated against the population effect size defined above. We sweep $L$ from 20 to 200 while fixing $U=800$. For each $(L,U)$ configuration, we repeat this procedure for 1000 trials. For automatic metrics, we use MetricX: \texttt{MetricX-XXL-20} for WMT22, \texttt{MetricX-23} for WMT23, and \texttt{MetricX-24} for WMT24.

\paragraph{Hypothesis Test.} We record the empirical power of the auto-only, the human-only, and the prediction-powered paired Z-test for each system pair, i.e., the percentage of trials in which the one-sided null hypothesis is rejected at $\alpha = 0.05$. Since each system pair has a different true human effect, this produces a power curve. As shown in Figure~\ref{fig:basic_power_wmt24_cs_uk}, the auto-only test has lower power than the human-only test for some human effect sizes, indicating bias in the automatic metric. If the automatic metric were unbiased, the auto-only test, which uses all $U+L$ examples, would be expected to have greater power than the human-only test, which uses only the $L$ labeled examples, across all human effect sizes. By contrast, the power curves of human-only and prediction-powered tests align closely, with prediction-powered tests consistently achieving higher power. 

\begin{figure}
    \centering
    \includegraphics[width=\linewidth]{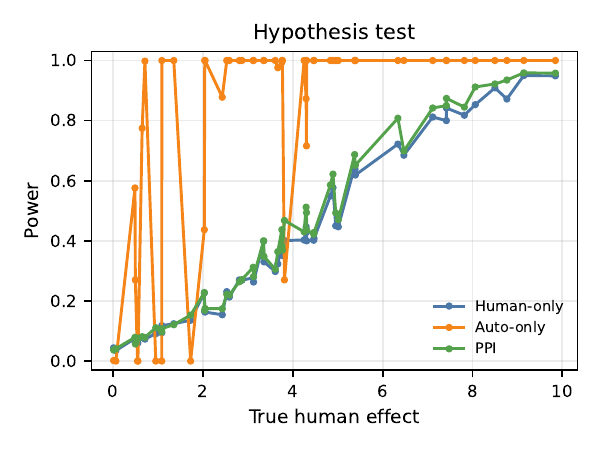}
    \caption{\small Empirical power curves at $L=80,U=800$ on WMT24 cs-uk. MetricX is used as the automatic metric. The fact that the empirical power of the auto-only test is lower than that of the human-only test for some effect sizes suggests that even a strong metric such as MetricX exhibits substantial bias for some system pairs. Similar trends are observed for other values of $L$ and other datasets. The full results are shown in Figure~\ref{fig:basic_power_l40} and~\ref{fig:basic_power_l80} in the appendix.}
    \label{fig:basic_power_wmt24_cs_uk}
    \vspace{-5mm}
\end{figure}

\paragraph{Confidence Intervals.} To more clearly illustrate the bias of automatic-only evaluation, we include an additional LLM-based metric, GEMBA \citep{DBLP:conf/wmt/KocmiF23} in the confidence interval experiments, whose score range matches that of the human scores. 
The auto-only confidence intervals based on GEMBA are very narrow on average, but their empirical coverage is only approximately 30\%, far below the nominal 95\% level. In contrast, the prediction-powered confidence intervals with GEMBA maintain coverage close to the nominal 95\% level while remaining consistently smaller than the human-only confidence intervals. Figure~\ref{fig:gemba_basic_interval_width} and~\ref{fig:gemba_basic_interval_coverage} in the appendix shows the results.

\begin{figure}
    \centering
    \includegraphics[width=\linewidth]{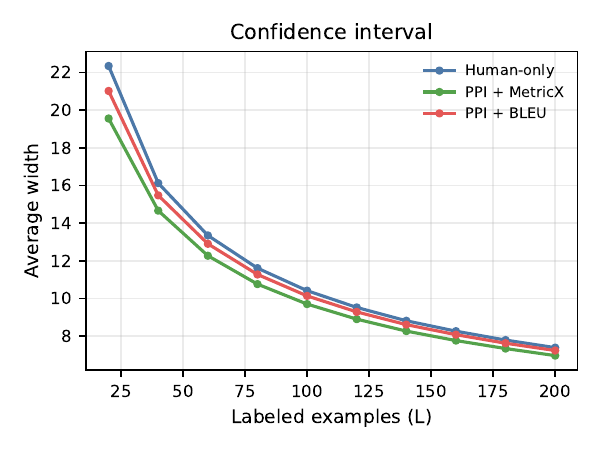}
    \caption{\small Average two-sided 95\% CI width with $U=800$ on WMT24 cs-uk. A smaller CI means better data efficiency. Results are averaged across system pairs. The trend is the same across datasets. Figure~\ref{fig:basic_interval_width} in the appendix shows full results. }
    \label{fig:basic_interval_wmt24_cs_uk}
    \vspace{-5mm}
\end{figure}

We also conduct a preliminary analysis of how the choice of automatic metric affects efficiency gains. Figure~\ref{fig:basic_interval_wmt24_cs_uk} compares the average widths of prediction-powered confidence intervals with MetricX and BLEU. Both are consistently smaller than the human-only baseline, although MetricX produces substantially larger reductions. This observation motivates the development of PPSR. For both metrics, the empirical coverage remains close to the nominal 95\% level, as shown in Figure~\ref{fig:basic_interval_coverage} in the appendix.

\paragraph{Summary.} Overall, PPI works for MT system comparisons. On the one hand, prediction-powered evaluation is unbiased, which is a clear advantage over auto-only evaluation. On the other hand, it is more data-efficient than human-only evaluation.

\subsection{Non-Parametric vs. Parametric Test}
\label{sec:exp-nonparam}

In this section, we investigate two questions: (1) whether the prediction-powered paired permutation test achieves higher power than the human-only paired permutation test, and (2) how the paired $Z$-test and paired permutation test compare in terms of power and Type I error under both the human-only and prediction-powered settings. We follow the setup described in \S\ref{sec:exp-basic} and use $B=1000$ random sign-flip permutations for paired permutation tests.

Consistent with the earlier results for the paired $Z$-test, the human-only and prediction-powered paired permutation tests exhibit nearly identical power trends, while the prediction-powered test consistently achieves higher power. Figures~\ref{fig:perm_only_power_l40},~\ref{fig:perm_only_power_l80},~\ref{fig:perm_only_power_l120} in the appendix present their power curves at different labeled sample sizes.

The paired $Z$-test achieves higher power than the paired permutation test in both the human-only and prediction-powered settings, especially when the labeled sample size is small, as shown in Figure~\ref{fig:z_vs_perm_power_gain_by_l} in the appendix. However, the Type I error simulations show that the paired $Z$-test can be poorly calibrated at small labeled sample sizes in the prediction-powered setting (Figure~\ref{fig:type_i_error_ppi_rho0.7_nu3}). This behavior arises from both error in the normal approximation and finite-sample instability in the plug-in variance and covariance estimators. See Appendix~\ref{app:simulation_type_i} for details of the Type I error simulations. Taken together, these results suggest that \textbf{the paired $Z$-test can be anti-conservative in small samples, whereas the paired permutation test provides a  reliable nonparametric alternative}.

\begin{figure}[t]
\centering
\includegraphics[width=\linewidth]{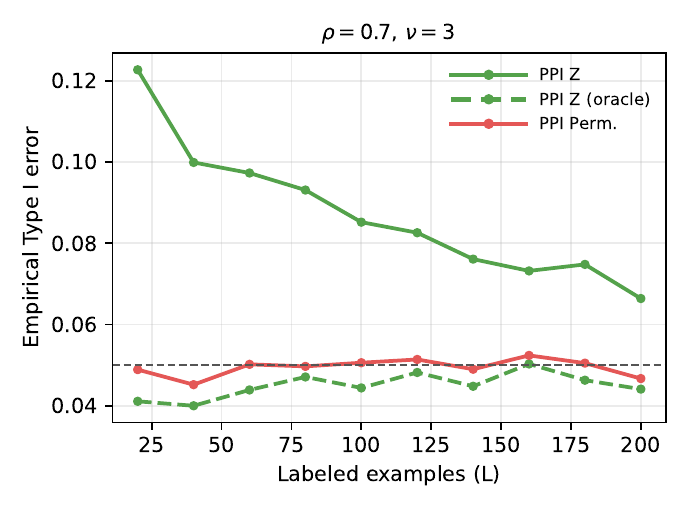}
\caption{\small Type I error under a zero-mean Student's $t$ distribution with $\rho=0.7$ and $\nu=3$. The nominal Type I error rate is 0.05. ``PPI Z'' refers to the original prediction-powered paired $Z$-test using plug-in variance and covariance estimates, whereas ``PPI Z (oracle)'' uses the true variance and covariance. ``PPI Perm'' refers to the prediction-powered paired permutation test. Full results are shown in Figure~\ref{fig:type_i_error_ppi} in the appendix.}
\label{fig:type_i_error_ppi_rho0.7_nu3}
\end{figure}

\begin{table}[]
\centering
\small
\setlength{\tabcolsep}{5pt}
\begin{tabular}{lrr@{\hskip 8pt}rr}
\toprule
& \multicolumn{2}{c}{Human-only} & \multicolumn{2}{c}{Prediction-powered} \\
\cmidrule(lr){2-3}\cmidrule(lr){4-5}
Dataset  & Var.\ Inc. & Pos.  & Var.\ Inc. & Pos. \\
\midrule
WMT22 en-de  & 1.28 & 100\%  & 1.11 & 100\% \\
WMT22 en-ru  & 1.81 & 100\%  & 1.67 & 100\% \\
WMT22 zh-en  & 1.28 & 100\%  & 1.02 & 100\% \\
WMT23 en-zh  & 0.16 & 98\%   & 0.11 & 92\%  \\
WMT23 ja-en  & 0.18 & 100\%  & 0.12 & 97\%  \\
WMT24 cs-uk  & 0.56 & 100\%  & 0.33 & 100\% \\
\bottomrule
\end{tabular}
\caption{\small Paired-versus-unpaired variance comparison for human-only and prediction-powered evaluation. ``Var.\ Inc.'' is the average relative variance increase from using the unpaired design; ``Pos.'' is the fraction of cases where it is positive (i.e., where the unpaired design increases variance). For human-only, each case is one system pair; for prediction-powered, each case is one system-pair/metric combination.}
\label{tab:paired_unpaired_variance}
\vspace{-5mm}
\end{table}

\begin{table*}[t]
\small
\centering
\begin{tabular}{lrrrrrr|rrrrrr}
\toprule
& \multicolumn{6}{c|}{Distinct values ($\uparrow$)}
& \multicolumn{6}{c}{Significant comparisons ($\uparrow$)} \\
\cmidrule(lr){2-7}\cmidrule(lr){8-13}
Dataset & Max & $r$ & $\rho$ & $\tau$ & SPA & PPSR
& Max & $r$ & $\rho$ & $\tau$ & SPA & PPSR \\
\midrule
WMT22 en-de & 31 & \textbf{31} & 23 & 14 & \textbf{31} & \textbf{31}
& 465 & 355 & 284 & 196 & 306 & \textbf{412} \\
WMT22 en-ru & 30 & \textbf{30} & 23 & 16 & \textbf{30} & \textbf{30}
& 435 & 354 & 289 & 270 & 314 & \textbf{373} \\
WMT22 zh-en & 31 & \textbf{31} & 27 & 18 & \textbf{31} & \textbf{31}
& 465 & 402 & 348 & 315 & 358 & \textbf{422} \\
WMT23 en-zh & 34 & \textbf{34} & 26 & 20 & \textbf{34} & \textbf{34}
& 561 & 493 & 423 & 409 & 440 & \textbf{507} \\
WMT23 ja-en & 35 & \textbf{35} & 22 & 13 & \textbf{35} & \textbf{35}
& 595 & 389 & 266 & 272 & 338 & \textbf{500} \\
WMT24 cs-uk & 25 & \textbf{25} & 17 & 14 & \textbf{25} & \textbf{25}
& 300 & 234 & 208 & 197 & 225 & \textbf{247} \\
\bottomrule
\end{tabular}
\caption{\small Discriminative power of system-level meta-metrics. Max gives the largest possible value for each criterion. Bold indicates the highest observed value in each row.}
\label{tab:saving_ratio_discriminative_power}
\vspace{-5mm}
\end{table*}

\subsection{Paired vs. Unpaired Design}
\label{sec:exp-gap}

As noted in \S\ref{sec:method-unpaired}, whether the paired or unpaired design is more efficient depends on the data. For human-only evaluation, we estimate the relative variance increase $(\widehat{\mathrm{Var}}\left[\widehat{\delta_{1,2}^{H,\mathrm{Un}}}\right]-\widehat{\mathrm{Var}}\left[\widehat{\delta_{1,2}^{H}}\right])/\widehat{\mathrm{Var}}\left[\widehat{\delta_{1,2}^{H}}\right]$ for each system pair in the selected datasets. Similarly, for prediction-powered evaluation, we estimate $(\widehat{\mathrm{Var}}\left[\widehat{\delta_{1,2}^{PP,\mathrm{Un}}}\right] - \widehat{\mathrm{Var}}\left[\widehat{\delta_{1,2}^{PP}}\right])/\widehat{\mathrm{Var}}\left[\widehat{\delta_{1,2}^{PP}}\right]$ for each system pair and automatic metric. 

As shown in Table~\ref{tab:paired_unpaired_variance}, the relative variance increase is positive in nearly all cases, indicating that paired designs are generally more efficient than unpaired designs. Therefore, we recommend \textbf{using the paired design for both human-only and prediction-powered evaluation whenever feasible}. However, this effect is weaker in the prediction-powered setting than in the human-only setting. See Appendix~\ref{app:paired_vs_unpaired} for in-depth analysis.

\subsection{PPSR Has Higher Discriminative Power}
\label{sec:exp-ppsr-disc}

A meta-metric with higher discriminative power is more likely to distinguish between automatic metrics. This is desirable in meta-evaluation, as an excess of ties makes comparison difficult.

We evaluate discriminative power in two ways. First, we count how many distinct values each meta-metric assigns to the automatic metrics in a dataset. Second, for every pair of automatic metrics, we follow the practice of \citet{DBLP:conf/wmt/0001MDK24} to apply the PERM-INPUTS permutation test \citep{DBLP:journals/tacl/DeutschDR21} with $B=1000$ and count how many pairwise differences are statistically significant at $p \leq 0.05$. Table~\ref{tab:saving_ratio_discriminative_power} summarizes the results. We find that PPSR achieves the highest discriminative power among all system-level meta-metrics on each dataset. By comparison, SPA has higher discriminative power than system-level $\rho$ and $\tau$, but lower than system-level $r$.

\subsection{PPSR Produces More Stable Rankings}
\label{sec:exp-ppsr-stability}

Another dimension to assess meta-metrics is whether the induced ranking of automatic metrics remains stable when only a subset of data is available. This matters because meta-evaluation studies often compare many metrics on a limited test set.

We use resampling to assess ranking stability. For each dataset and each meta-metric, we first obtain the ranking of automatic metrics using the full dataset. We then sample $L$ inputs without replacement (keeping the system set fixed), recompute the ranking produced by each meta-metric on the sampled data, and measure Kendall's $\tau$ between the sampled ranking and the full-data ranking. This procedure is repeated 1000 times for each $L \in \{100, 200, \ldots, 1000\}$, and we report the average Kendall's $\tau$. Figure~\ref{fig:saving_ratio_ranking_stability_cs_uk} shows PPSR achieves the highest average ranking stability across input size. We emphasize that the higher discriminative power and ranking stability does not imply that PPSR can replace other meta-metrics, because their meanings and intended use cases are substantially different.

\begin{figure}[]
\centering
\includegraphics[width=\linewidth]{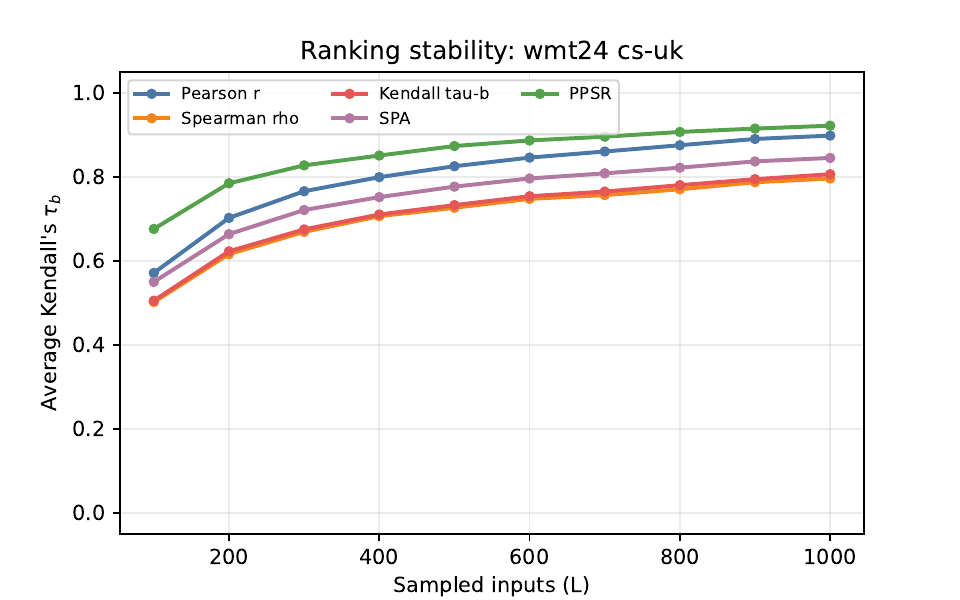}
\caption{\small Ranking stability of system-level meta-metrics on WMT24 cs-uk. Higher values indicate more stable metric rankings. Note that for meta-evaluation, only labeled examples are used and the unlabeled examples are unused. The trend is the same across all datasets. Full results for all datasets are shown in Figure~\ref{fig:saving_ratio_ranking_stability} in the appendix.}
\label{fig:saving_ratio_ranking_stability_cs_uk}
\vspace{-5mm}
\end{figure}

\subsection{Metric Rankings Produced By PPSR}
\label{sec:exp-ppsr-result}

Table~\ref{tab:system_metric_score_ranks_wmt22_en_de}--\ref{tab:segment_metric_score_ranks_wmt24_cs_uk} in the appendix reports the scores and rankings of all automatic metrics under PPSR and other meta-metrics. We find that though PPSR is a system-level meta-metric designed for system comparison, its metric rankings differ substantially from those produced by existing system-level meta-metrics. Instead, its metric rankings align more closely with segment-level meta-metrics, particularly PDP. This supports the interpretation that PPSR measures a related yet distinct quality in automatic metrics.

\section{Conclusion}

We show the advantages of prediction-powered evaluation over human-only and auto-only approaches on WMT data, providing practical recommendations across parametric vs.\ non-parametric inference and paired vs.\ unpaired designs. Building on this framework, we introduce PPSR, a meta-metric offering greater discriminative power and ranking stability than existing system-level alternatives. More broadly, our framework applies to other non-verifiable tasks where both human judgments and automatic metrics are pointwise.

\section*{Limitations}

For simplicity, we do not model inter-annotator disagreement as an additional source of variance in human scores, as considered in prior work \citep{chaganty-etal-2018-price,DBLP:conf/acl/WeiJ20}.

\citet{DBLP:journals/corr/abs-2505-20178} pointed out that reusing the same labeled data both to estimate the $\hat\lambda$, and to estimate the empirical variance to construct the confidence intervals while treating $\hat \lambda$ as fixed can lead to overly optimistic empirical variance. We do not address this issue in the present work.

We also assume that system-level scores are obtained by averaging over inputs. As a result, our framework may not apply directly to the automatic metrics that compute system-level scores using different aggregation methods.

Finally, our experiments focus on machine translation. The applicability of our methods to other domains and tasks remains to be further validated.

\section*{Ethical Considerations}

The MT Metrics Eval V2 toolkit is licensed under Apache-2.0, and our use of it complies with the terms of the license. AI-based writing tools were used to assist with language polishing. Specifically, we prompted Claude and ChatGPT to generate suggested revisions, which we then manually reviewed and edited. 

\section*{Acknowledgments}
We thank Pengcheng Su for insightful feedback and the reviewers in the ARR May 2026 cycle for their valuable suggestions.

\bibliography{custom}

\appendix

\onecolumn

\section{Theoretical Properties of the Paired Prediction-Powered Estimator}
\label{app:paired-theory}

Throughout, $\delta_{1,2}=\mu_1-\mu_2$ with $\mu_i=\mathbb{E}[Y_i]$ and $\nu_i=\mathbb{E}[F(X_i)]$. For the paired design write
\[
D_j^Y = Y_{1j}-Y_{2j},
\qquad
D_j^F = F(X_{1j})-F(X_{2j}),
\]
so the paired prediction-powered estimator with tuning parameter $\lambda$ is
\[
\widehat{\delta_{1,2}^{PP}}
=\frac{1}{U}\sum_{j=L+1}^{L+U}\lambda D_j^F
+\frac{1}{L}\sum_{j=1}^{L}\bigl(D_j^Y-\lambda D_j^F\bigr),
\]
and the human-only estimator is the special case
$\widehat{\delta_{1,2}^{H}}=\widehat{\delta_{1,2}^{PP}}\big|_{\lambda=0}=\frac1L\sum_{j=1}^{L}D_j^Y$.

\begin{assumption}[Paired sampling]
\label{ass:paired}
The labeled pairs $\{(D_j^Y,D_j^F)\}_{j=1}^{L}$ are i.i.d., the unlabeled values $\{D_j^F\}_{j=L+1}^{L+U}$ are i.i.d., the labeled and unlabeled samples are mutually independent, and all relevant second moments are finite.
\end{assumption}

\begin{proposition}[Unbiasedness and variance]
\label{prop:paired-mean-var}
Under Assumption~\ref{ass:paired}, for any fixed $\lambda$ the estimator $\widehat{\delta_{1,2}^{PP}}$ is unbiased,
$\mathbb{E}[\widehat{\delta_{1,2}^{PP}}]=\delta_{1,2}$, with variance
\[
\mathrm{Var}\bigl[\widehat{\delta_{1,2}^{PP}}\bigr]
=\lambda^2\Bigl(\frac1U+\frac1L\Bigr)\mathrm{Var}[D^F]-\frac{2\lambda}{L}\mathrm{Cov}[D^Y,D^F]
+\frac1L\mathrm{Var}[D^Y]. 
\]

\end{proposition}

\begin{proof}
For the mean,
\begin{equation*}
\begin{split}
\mathbb{E}[\widehat{\delta_{1,2}^{PP}}]
&=
\mathbb{E}\left[
\frac{1}{U}\sum_{j=L+1}^{L+U}
\lambda \big(F(X_{1j}) - F(X_{2j})\big)
+
\frac{1}{L}\sum_{j=1}^L
\left[
(Y_{1j} - Y_{2j})
-
\lambda \big(F(X_{1j}) - F(X_{2j})\big)
\right]
\right] \\
&=
\frac{\lambda}{U}\sum_{j=L+1}^{L+U}
\mathbb{E}\left[
F(X_{1j}) - F(X_{2j})
\right]
+
\frac{1}{L}\sum_{j=1}^L
\mathbb{E}\left[
Y_{1j} - Y_{2j}
\right]
-
\frac{\lambda}{L}\sum_{j=1}^L
\mathbb{E}\left[
F(X_{1j}) - F(X_{2j})
\right] \\
&= \lambda \nu_1 - \lambda \nu_2 + \mu_1 - \mu_2 - \lambda \nu_1 + \lambda \nu_2 \\
&=\mu_1-\mu_2 \\
&=\delta_{1,2}.
\end{split}
\end{equation*}
For the variance, using independence of the labeled and unlabeled blocks,
\begin{equation*}
\begin{split}
\mathrm{Var}[\widehat{\delta_{1,2}^{PP}}]
&=
\mathrm{Var}\left[
\frac{1}{U}\sum_{j=L+1}^{L+U}
\lambda \big(F(X_{1j}) - F(X_{2j})\big)
\right]
+
\mathrm{Var}\left[
\frac{1}{L}\sum_{j=1}^L
\left[
(Y_{1j} - Y_{2j})
-
\lambda \big(F(X_{1j}) - F(X_{2j})\big)
\right]
\right] \\
&=
\frac{\lambda^2}{U}
\mathrm{Var}\big[
F(X_{1}) - F(X_{2})
\big]
+
\frac{1}{L}
\Big(
\mathrm{Var}[Y_{1} - Y_{2}]
+
\lambda^2
\mathrm{Var}\big[
F(X_{1}) - F(X_{2})
\big] \\
&\qquad
- 2\lambda
\mathrm{Cov}\big[
Y_{1} - Y_{2},\,
F(X_{1}) - F(X_{2})
\big]
\Big) \\
&=
\lambda^2
\left(
\frac{1}{U}+\frac{1}{L}
\right)
\mathrm{Var}\big[
F(X_1) - F(X_2)
\big]
-
\frac{2\lambda}{L}
\mathrm{Cov}\big[
Y_1 - Y_2,\,
F(X_1) - F(X_2)
\big]
+
\frac{1}{L}
\mathrm{Var}[Y_1 - Y_2],
\end{split}
\end{equation*}
which is the claimed expression in the notation $D^Y=Y_1-Y_2$, $D^F=F(X_1)-F(X_2)$.
\end{proof}

\begin{proposition}[Optimal tuning and variance reduction]
\label{prop:paired-opt}
The variance in Proposition~\ref{prop:paired-mean-var} is an upward-opening quadratic in $\lambda$, minimized at
\[
\lambda^\star=\frac{\mathrm{Cov}[D^Y,D^F]}{\bigl(1+\frac{L}{U}\bigr)\mathrm{Var}[D^F]},
\]
at which
\begin{align*}
\mathrm{Var}\bigl[\widehat{\delta_{1,2}^{PP}}\bigr]
&=\mathrm{Var}\bigl[\widehat{\delta_{1,2}^{H}}\bigr]
-\frac{U\bigl(\mathrm{Cov}[D^Y,D^F]\bigr)^2}{L(L+U)\,\mathrm{Var}[D^F]}\le\mathrm{Var}\bigl[\widehat{\delta_{1,2}^{H}}\bigr].   
\end{align*}
Hence the optimally tuned prediction-powered estimator never has larger variance than the human-only estimator.
\end{proposition}

\begin{proof}
The variance is a quadratic in $\lambda$ with positive leading coefficient $\bigl(\frac1U+\frac1L\bigr)\mathrm{Var}[D^F]$, so it is minimized at
\begin{equation*}
\lambda^\star
=
\frac{
\mathrm{Cov}\big[
Y_1 - Y_2,\,
F(X_1) - F(X_2)
\big]
}{
\left(1+\frac{L}{U}\right)
\mathrm{Var}\big[
F(X_1) - F(X_2)
\big]
}.
\end{equation*}
Substituting $\lambda^\star$ into the variance gives
\begin{equation*}
\begin{split}
\mathrm{Var}[\widehat{\delta_{1,2}^{PP}}]
& = \frac{1}{L}\mathrm{Var}[Y_1 - Y_2]  - \frac{
U\left(\mathrm{Cov}\big[Y_1 - Y_2,\, F(X_1) - F(X_2)\big] \right)^2
}{
L(L+U)
\mathrm{Var}\big[
F(X_1) - F(X_2)
\big]
} \\
& = \mathrm{Var}[\widehat{\delta_{1,2}^{H}}] - \frac{
U\left(\mathrm{Cov}\big[Y_1 - Y_2,\, F(X_1) - F(X_2)\big] \right)^2
}{
L(L+U)
\mathrm{Var}\big[
F(X_1) - F(X_2)
\big]
} \\
& \leq \mathrm{Var}[\widehat{\delta_{1,2}^{H}}],
\end{split}
\end{equation*}
since the subtracted term is nonnegative.
\end{proof}

\begin{theorem}[Asymptotic confidence-interval validity]
\label{thm:paired-ci}
Suppose Assumption~\ref{ass:paired} holds, the relevant second moments are
strictly positive, and $L,U\to\infty$ with $L/U\to v\in[0,\infty)$. Then for any
fixed $\lambda$,
\[
\frac{\widehat{\delta_{1,2}^{g}}-\delta_{1,2}}
{\sqrt{\mathrm{Var}[\widehat{\delta_{1,2}^{g}}]}}
\xrightarrow{d}\mathcal{N}(0,1),
\qquad g\in\{H,PP\}.
\]
Replacing the variance by any consistent estimator $\widehat{\mathrm{Var}}$, the interval $\widehat{\delta_{1,2}^{g}}\pm z_{1-\alpha/2} \sqrt{\widehat{\mathrm{Var}}[\widehat{\delta_{1,2}^{g}}]}$
is an asymptotically valid $100(1-\alpha)\%$ confidence interval for $\delta_{1,2}$.
\end{theorem}

\begin{proof}
Fix $\lambda$ and write $\sigma_F^2=\mathrm{Var}[D^F]$ and
$\tau^2=\mathrm{Var}[D^Y-\lambda D^F]$, both strictly positive by assumption.
Since $\mathbb{E}[\widehat{\delta_{1,2}^{PP}}]=\delta_{1,2}$
(Proposition~\ref{prop:paired-mean-var}) and
$\lambda\mathbb{E}[D^F]+\mathbb{E}[D^Y-\lambda D^F]=\mathbb{E}[D^Y]=\delta_{1,2}$,
the centered estimator splits into two centered blocks,
\begin{equation*}
\widehat{\delta_{1,2}^{PP}}-\delta_{1,2}
=\underbrace{\frac{1}{U}\sum_{j=L+1}^{L+U}
\lambda\bigl(D_j^F-\mathbb{E}[D^F]\bigr)}_{=:\,\bar A_U}
\;+\;
\underbrace{\frac{1}{L}\sum_{j=1}^{L}
\Bigl[(D_j^Y-\lambda D_j^F)-\mathbb{E}[D^Y-\lambda D^F]\Bigr]}_{=:\,\bar B_L},
\end{equation*}
where $\bar A_U$ depends only on the unlabeled block and $\bar B_L$ only on the labeled block, so $\bar A_U$ and $B_L$ are independent. 

We scale by $\sqrt{L}$ to put both blocks on a common $1/\sqrt{L}$ rate. For the
labeled block, the central limit theorem gives
\begin{equation*}
\sqrt{L}\,\bar B_L
=\frac{1}{\sqrt{L}}\sum_{j=1}^{L}
\Bigl[(D_j^Y-\lambda D_j^F)-\mathbb{E}[D^Y-\lambda D^F]\Bigr]
\xrightarrow{d}\mathcal{N}\!\bigl(0,\tau^2\bigr).
\end{equation*}
For the unlabeled block, factor out the rate mismatch:
\begin{equation*}
\sqrt{L}\,\bar A_U
=\sqrt{\tfrac{L}{U}}\;\cdot\;
\frac{1}{\sqrt{U}}\sum_{j=L+1}^{L+U}
\lambda\bigl(D_j^F-\mathbb{E}[D^F]\bigr).
\end{equation*}
The central limit theorem gives
$\frac{1}{\sqrt{U}}\sum_{j=L+1}^{L+U}\lambda(D_j^F-\mathbb{E}[D^F])
\xrightarrow{d}\mathcal{N}(0,\lambda^2\sigma_F^2)$, and $\sqrt{L/U}\to\sqrt{v}$,
so by Slutsky's theorem
\begin{equation*}
\sqrt{L}\,\bar A_U\xrightarrow{d}\mathcal{N}\!\bigl(0,\,v\,\lambda^2\sigma_F^2\bigr).
\end{equation*}
Because $\bar A_U$ and $\bar B_L$ are independent, the pair
$(\sqrt{L}\,\bar A_U,\sqrt{L}\,\bar B_L)$ converges jointly to a pair of
independent normals, and hence their sum converges to the sum of those normals:
\begin{equation*}
\sqrt{L}\bigl(\widehat{\delta_{1,2}^{PP}}-\delta_{1,2}\bigr)
=\sqrt{L}\,\bar A_U+\sqrt{L}\,\bar B_L
\xrightarrow{d}\mathcal{N}\!\bigl(0,\,V_\infty\bigr),
\qquad
V_\infty:=v\,\lambda^2\sigma_F^2+\tau^2.
\end{equation*}
It remains to normalize by the exact variance. By
Proposition~\ref{prop:paired-mean-var},
$\mathrm{Var}[\widehat{\delta_{1,2}^{PP}}]
=\frac{\lambda^2\sigma_F^2}{U}+\frac{\tau^2}{L}$, so
\begin{equation*}
L\,\mathrm{Var}[\widehat{\delta_{1,2}^{PP}}]
=\frac{L}{U}\lambda^2\sigma_F^2+\tau^2
\longrightarrow v\,\lambda^2\sigma_F^2+\tau^2=V_\infty\in(0,\infty).
\end{equation*}
Writing
\begin{equation*}
\frac{\widehat{\delta_{1,2}^{PP}}-\delta_{1,2}}
{\sqrt{\mathrm{Var}[\widehat{\delta_{1,2}^{PP}}]}}
=\frac{\sqrt{L}\bigl(\widehat{\delta_{1,2}^{PP}}-\delta_{1,2}\bigr)}
{\sqrt{L\,\mathrm{Var}[\widehat{\delta_{1,2}^{PP}}]}},
\end{equation*}
the numerator converges in distribution to $\mathcal{N}(0,V_\infty)$ and the
denominator converges to $\sqrt{V_\infty}>0$, so Slutsky's
theorem yields
\begin{equation*}
\frac{\widehat{\delta_{1,2}^{PP}}-\delta_{1,2}}
{\sqrt{\mathrm{Var}[\widehat{\delta_{1,2}^{PP}}]}}
\xrightarrow{d}\mathcal{N}(0,1).
\end{equation*}
The human-only estimator $\widehat{\delta_{1,2}^{H}}$ is the special case $\lambda=0$, for which
$\bar A_U=0$, $\widehat{\delta_{1,2}^{PP}}=\frac1L\sum_{j=1}^{L}D_j^Y
=\widehat{\delta_{1,2}^{H}}$, $\tau^2=\mathrm{Var}[D^Y]$, and
$V_\infty=\mathrm{Var}[D^Y]$. Now a direct application of the
central limit theorem to the labeled block gives the limit.

Finally, for either $g\in\{H,PP\}$ and any consistent variance estimator,
\begin{equation*}
\frac{\widehat{\delta_{1,2}^{g}}-\delta_{1,2}}
{\sqrt{\widehat{\mathrm{Var}}[\widehat{\delta_{1,2}^{g}}]}}
=\frac{\widehat{\delta_{1,2}^{g}}-\delta_{1,2}}
{\sqrt{\mathrm{Var}[\widehat{\delta_{1,2}^{g}}]}}
\cdot
\sqrt{\frac{\mathrm{Var}[\widehat{\delta_{1,2}^{g}}]}
{\widehat{\mathrm{Var}}[\widehat{\delta_{1,2}^{g}}]}}
\xrightarrow{d}\mathcal{N}(0,1),
\end{equation*}
since the first factor converges to $\mathcal{N}(0,1)$ and the second converges in
probability to $1$. This establishes the validity of the stated interval.
\end{proof}

\begin{theorem}[Plug-in validity with data-driven tuning]
\label{thm:paired-plugin}
Maintain the hypotheses of Theorem~\ref{thm:paired-ci}: Assumption~\ref{ass:paired} holds, the relevant second moments are strictly positive, and $L,U\to\infty$ with $L/U\to v\in[0,\infty)$. Let $\widehat{\mathrm{Cov}}[D^Y,D^F]$ and
$\widehat{\mathrm{Var}}[D^F]$ be consistent moment estimators computed from the labeled pairs, and define the data-driven tuning parameter
\[
\hat\lambda
=\frac{\widehat{\mathrm{Cov}}[D^Y,D^F]}
{\bigl(1+\frac{L}{U}\bigr)\widehat{\mathrm{Var}}[D^F]},
\]
together with the plug-in point estimator $\widehat{\delta_{1,2}^{PP}}(\hat\lambda)$ and the plug-in variance estimator obtained by substituting $\hat\lambda$ and consistent moment estimators into Proposition~\ref{prop:paired-mean-var},
\[
\widehat{\mathrm{Var}}\bigl[\widehat{\delta_{1,2}^{PP}}(\hat\lambda)\bigr]
=\hat\lambda^2\Bigl(\tfrac1U+\tfrac1L\Bigr)\widehat{\mathrm{Var}}[D^F]
-\frac{2\hat\lambda}{L}\widehat{\mathrm{Cov}}[D^Y,D^F]
+\frac1L\widehat{\mathrm{Var}}[D^Y].
\]
Then
\[
\frac{\widehat{\delta_{1,2}^{PP}}(\hat\lambda)-\delta_{1,2}}
{\sqrt{\widehat{\mathrm{Var}}\bigl[\widehat{\delta_{1,2}^{PP}}(\hat\lambda)\bigr]}}
\xrightarrow{d}\mathcal{N}(0,1),
\]
so $\widehat{\delta_{1,2}^{PP}}(\hat\lambda)\pm z_{1-\alpha/2}
\sqrt{\widehat{\mathrm{Var}}[\widehat{\delta_{1,2}^{PP}}(\hat\lambda)]}$
is an asymptotically valid $100(1-\alpha)\%$ confidence interval for
$\delta_{1,2}$.
\end{theorem}

\begin{proof}
Write $\sigma_F^2=\mathrm{Var}[D^F]>0$ and $c=\mathrm{Cov}[D^Y,D^F]$, and let $\lambda_0=c/\bigl((1+v)\sigma_F^2\bigr)$ denote the population limit of the tuning parameter. Since the moment estimators are consistent, we have $\widehat{\mathrm{Cov}}[D^Y,D^F]\xrightarrow{p}c$ and
$\widehat{\mathrm{Var}}[D^F]\xrightarrow{p}\sigma_F^2$, while
$1+\frac{L}{U}\to 1+v$; because $\sigma_F^2>0$ keeps the denominator bounded away from zero, the continuous mapping theorem gives $\hat\lambda\xrightarrow{p}\lambda_0$, so $\hat{\lambda} - \lambda_0 = o_p(1)$.

The estimator is affine in its tuning parameter. Grouping by powers of $\lambda$,
\[
\widehat{\delta_{1,2}^{PP}}(\lambda)
=\underbrace{\frac1L\sum_{j=1}^{L}D_j^Y}_{=\widehat{\delta_{1,2}^{H}}}
+\lambda\,\widehat{\Delta},
\qquad
\widehat{\Delta}:=\frac1U\sum_{j=L+1}^{L+U}D_j^F-\frac1L\sum_{j=1}^{L}D_j^F,
\]
so that
$\widehat{\delta_{1,2}^{PP}}(\hat\lambda)-\widehat{\delta_{1,2}^{PP}}(\lambda_0)
=(\hat\lambda-\lambda_0)\,\widehat{\Delta}$. The coefficient $\widehat{\Delta}$ is
a difference of two independent sample averages of i.i.d.\ copies of $D^F$, with
$\mathbb{E}[\widehat{\Delta}]=0$ and
$\mathrm{Var}[\widehat{\Delta}]=\bigl(\tfrac1U+\tfrac1L\bigr)\sigma_F^2$. Applying the $\sqrt{L}$ scaling of Theorem~\ref{thm:paired-ci} to $\widehat{\Delta}$, the central limit theorem on each block together with $\sqrt{L/U}\to\sqrt v$ and Slutsky's theorem, gives
\[
\sqrt{L}\,\widehat{\Delta}
\xrightarrow{d}\mathcal{N}\!\bigl(0,(1+v)\sigma_F^2\bigr),
\qquad\text{so in particular }\sqrt{L}\,\widehat{\Delta}=O_p(1).
\]
Consequently the substitution of $\hat\lambda$ for $\lambda_0$ perturbs the estimator only through the product of a vanishing consistency gap and a bounded centered term,
\[
\sqrt{L}\Bigl(\widehat{\delta_{1,2}^{PP}}(\hat\lambda)
-\widehat{\delta_{1,2}^{PP}}(\lambda_0)\Bigr)
=(\hat\lambda-\lambda_0)\,\sqrt{L}\,\widehat{\Delta}
=o_p(1)\cdot O_p(1)=o_p(1).
\]
Equivalently,
\[\sqrt{L}\Bigl(\widehat{\delta_{1,2}^{PP}}(\hat\lambda) -\widehat{\delta_{1,2}^{PP}}(\lambda_0)\Bigr)\xrightarrow{p} 0.\]
The limit law now follows by comparison with the fixed-$\lambda$ result. Because $\lambda_0$ is a fixed constant, Theorem~\ref{thm:paired-ci} applies at $\lambda=\lambda_0$ and gives
\[
\sqrt{L}\bigl(\widehat{\delta_{1,2}^{PP}}(\lambda_0)-\delta_{1,2}\bigr)
\xrightarrow{d}\mathcal{N}\!\bigl(0,V_\infty(\lambda_0)\bigr),
\qquad
V_\infty(\lambda_0)=(1+v)\lambda_0^2\sigma_F^2-2\lambda_0 c+\mathrm{Var}[D^Y]>0,
\]
Adding $\sqrt{L}\Bigl(\widehat{\delta_{1,2}^{PP}}(\hat\lambda) -\widehat{\delta_{1,2}^{PP}}(\lambda_0)\Bigr)\xrightarrow{p} 0$ above, via Slutsky's theorem, transfers this limit to the plug-in estimator,
\[
\sqrt{L}\bigl(\widehat{\delta_{1,2}^{PP}}(\hat\lambda)-\delta_{1,2}\bigr)
\xrightarrow{d}\mathcal{N}\!\bigl(0,V_\infty(\lambda_0)\bigr).
\]

It remains to check that the plug-in variance estimator recovers the same constant. Scaling by $L$,
\[
L\,\widehat{\mathrm{Var}}\bigl[\widehat{\delta_{1,2}^{PP}}(\hat\lambda)\bigr]
=\hat\lambda^2\Bigl(\tfrac{L}{U}+1\Bigr)\widehat{\mathrm{Var}}[D^F]
-2\hat\lambda\,\widehat{\mathrm{Cov}}[D^Y,D^F]
+\widehat{\mathrm{Var}}[D^Y]
\]
is a continuous function of $\bigl(\hat\lambda,\tfrac{L}{U},\widehat{\mathrm{Var}}[D^F], \widehat{\mathrm{Cov}}[D^Y,D^F],\widehat{\mathrm{Var}}[D^Y]\bigr)$, and these converge in probability to $(\lambda_0,v,\sigma_F^2,c,\mathrm{Var}[D^Y])$ by the
consistency established above. The continuous mapping theorem therefore yields
\[
L\,\widehat{\mathrm{Var}}\bigl[\widehat{\delta_{1,2}^{PP}}(\hat\lambda)\bigr]
\xrightarrow{p}
(1+v)\lambda_0^2\sigma_F^2-2\lambda_0 c+\mathrm{Var}[D^Y]
=V_\infty(\lambda_0)>0.
\]
Writing the studentized statistic as
\[
\frac{\widehat{\delta_{1,2}^{PP}}(\hat\lambda)-\delta_{1,2}}
{\sqrt{\widehat{\mathrm{Var}}[\widehat{\delta_{1,2}^{PP}}(\hat\lambda)]}}
=\frac{\sqrt{L}\bigl(\widehat{\delta_{1,2}^{PP}}(\hat\lambda)-\delta_{1,2}\bigr)}
{\sqrt{L\,\widehat{\mathrm{Var}}[\widehat{\delta_{1,2}^{PP}}(\hat\lambda)]}},
\]
its numerator converges in distribution to $\mathcal{N}(0,V_\infty(\lambda_0))$ and its denominator converges in probability to $\sqrt{V_\infty(\lambda_0)}>0$, so a final application of Slutsky's theorem gives the standard normal limit and the stated coverage.
\end{proof}

\section{Unpaired Design}
\label{app:unpaired}

For the unpaired design let $L_i=|I_i^L|$, $U_i=|I_i^U|$, and
\begin{equation*}
\begin{split}
\widehat{\delta_{1,2}^{PP,\mathrm{Un}}}
& =\frac{1}{U_1}\sum_{j\in I_1^U}\lambda_1 F(X_{1j})
+\frac{1}{L_1}\sum_{j\in I_1^L}\bigl(Y_{1j}-\lambda_1 F(X_{1j})\bigr) \\ 
&\qquad - \frac{1}{U_2}\sum_{j\in I_2^U}\lambda_2 F(X_{2j})
-\frac{1}{L_2}\sum_{j\in I_2^L}\bigl(Y_{2j}-\lambda_2 F(X_{2j})\bigr).
\end{split}
\end{equation*}

\begin{assumption}[Unpaired sampling]
\label{ass:unpaired}
Samples from different systems are independent; for each system the labeled and unlabeled samples are drawn from the same system-specific distribution; and within each system the i.i.d.\ and finite-second-moment conditions hold.
\end{assumption}

\begin{proposition}[Unbiasedness and variance]
\label{prop:unpaired}
Under Assumption~\ref{ass:unpaired}, $\widehat{\delta_{1,2}^{PP,\mathrm{Un}}}$ is unbiased for $\delta_{1,2}$ with
\begin{equation*}
\begin{split}
\mathrm{Var}\bigl[\widehat{\delta_{1,2}^{PP,\mathrm{Un}}}\bigr] &=\lambda_1^2\bigl(\frac{1}{U_1}+\frac{1}{L_1}\bigr)\mathrm{Var}[F(X_1)]
-\frac{2\lambda_1}{L_1}\mathrm{Cov}[Y_1,F(X_1)]
+\frac{1}{L_1}\mathrm{Var}[Y_1] \\
&\qquad + \lambda_2^2\bigl(\frac{1}{U_2}+\frac{1}{L_2}\bigr)\mathrm{Var}[F(X_2)]
-\frac{2\lambda_2}{L_2}\mathrm{Cov}[Y_2,F(X_2)]
+\frac{1}{L_2}\mathrm{Var}[Y_2]
\end{split}
\end{equation*}

This is minimized at
\[\lambda_i^\star=\dfrac{\mathrm{Cov}[Y_i,F(X_i)]}{(1+L_i/U_i)\mathrm{Var}[F(X_i)]}, \quad i\in\{1,2\},\]
giving
\begin{equation*}
\begin{split}
\mathrm{Var}\bigl[\widehat{\delta_{1,2}^{PP,\mathrm{Un}}}\bigr] =\mathrm{Var}\bigl[\widehat{\delta_{1,2}^{H,\mathrm{Un}}}\bigr] 
-\frac{U_1\bigl(\mathrm{Cov}[Y_1,F(X_1)]\bigr)^2}{L_1(L_1+U_1)\,\mathrm{Var}[F(X_1)]} - \frac{U_2\bigl(\mathrm{Cov}[Y_2,F(X_2)]\bigr)^2}{L_2(L_2+U_2)\,\mathrm{Var}[F(X_2)]} \le\mathrm{Var}\bigl[\widehat{\delta_{1,2}^{H,\mathrm{Un}}}\bigr].
\end{split}
\end{equation*}
\end{proposition}

\begin{proof}
For the mean, with $\mu_i=\mathbb{E}[Y_i]$ and $\nu_i=\mathbb{E}[F(X_i)]$,
\begin{equation*}
\begin{split}
\mathbb{E}\left[\widehat{\delta_{1,2}^{PP,\mathrm{Un}}}\right]
&=
\frac{\lambda_1}{U_1}\sum_{j\in I_1^U}
\mathbb{E}[F(X_{1j})]
+
\frac{1}{L_1}\sum_{j\in I_1^L}
\mathbb{E}[Y_{1j}]
-
\frac{\lambda_1}{L_1}\sum_{j\in I_1^L}
\mathbb{E}[F(X_{1j})] \\
&\qquad
-
\frac{\lambda_2}{U_2}\sum_{j\in I_2^U}
\mathbb{E}[F(X_{2j})]
-
\frac{1}{L_2}\sum_{j\in I_2^L}
\mathbb{E}[Y_{2j}]
+
\frac{\lambda_2}{L_2}\sum_{j\in I_2^L}
\mathbb{E}[F(X_{2j})] \\
&=
\lambda_1 \nu_1 + \mu_1 - \lambda_1 \nu_1
-
\lambda_2 \nu_2 - \mu_2 + \lambda_2 \nu_2 \\
&= \mu_1-\mu_2 \\
& =\delta_{1,2}.
\end{split}
\end{equation*}
For the variance, the four blocks are mutually independent, so
\begin{equation*}
\begin{split}
\mathrm{Var}\left[\widehat{\delta_{1,2}^{PP,\mathrm{Un}}}\right]
&=
\frac{\lambda_1^2}{U_1}\mathrm{Var}[F(X_1)]
+
\frac{1}{L_1}
\left(
\mathrm{Var}[Y_1]
+
\lambda_1^2\mathrm{Var}[F(X_1)]
-
2\lambda_1\mathrm{Cov}[Y_1,F(X_1)]
\right) \\
&\qquad
+
\frac{\lambda_2^2}{U_2}\mathrm{Var}[F(X_2)]
+
\frac{1}{L_2}
\left(
\mathrm{Var}[Y_2]
+
\lambda_2^2\mathrm{Var}[F(X_2)]
-
2\lambda_2\mathrm{Cov}[Y_2,F(X_2)]
\right) \\
&=
\sum_{i=1}^{2}\Bigl[
\lambda_i^2
\Bigl(
\frac{1}{U_i}+\frac{1}{L_i}
\Bigr)
\mathrm{Var}[F(X_i)]
-
\frac{2\lambda_i}{L_i}
\mathrm{Cov}[Y_i,F(X_i)]
+
\frac{1}{L_i}\mathrm{Var}[Y_i]
\Bigr].
\end{split}
\end{equation*}
This is a sum of two upward-opening quadratics, in $\lambda_1$ and $\lambda_2$ respectively, so it is minimized termwise at
\begin{equation*}
\lambda_i^\star
=
\frac{
\mathrm{Cov}[Y_i,F(X_i)]
}{
\left(1+\frac{L_i}{U_i}\right)
\mathrm{Var}[F(X_i)]
},
\qquad i\in\{1,2\}.
\end{equation*}
Substituting $\lambda_1^\star,\lambda_2^\star$ gives
\begin{equation*}
\begin{split}
\mathrm{Var}\left[\widehat{\delta_{1,2}^{PP,\mathrm{Un}}}\right]
&=
\mathrm{Var}\left[\widehat{\delta_{1,2}^{H,\mathrm{Un}}}\right]
-
\frac{
U_1\left(\mathrm{Cov}[Y_1,F(X_1)]\right)^2
}{
L_1(L_1+U_1)\mathrm{Var}[F(X_1)]
}
-
\frac{
U_2\left(\mathrm{Cov}[Y_2,F(X_2)]\right)^2
}{
L_2(L_2+U_2)\mathrm{Var}[F(X_2)]
}
\;\leq\;
\mathrm{Var}\left[\widehat{\delta_{1,2}^{H,\mathrm{Un}}}\right],
\end{split}
\end{equation*}
where $\mathrm{Var}[\widehat{\delta_{1,2}^{H,\mathrm{Un}}}]=\frac{1}{L_1}\mathrm{Var}[Y_1]+\frac{1}{L_2}\mathrm{Var}[Y_2]$.
\end{proof}

\section{Paired vs. Unpaired Design}
\label{app:paired_vs_unpaired}

\begin{proposition}[Paired-versus-unpaired variance gap]
\label{prop:gap}
Assume equal sizes $L_1=L_2=L$ and $U_1=U_2=U$. For human-only evaluation,
\[
\mathrm{Var}\bigl[\widehat{\delta_{1,2}^{H,\mathrm{Un}}}\bigr]
-\mathrm{Var}\bigl[\widehat{\delta_{1,2}^{H}}\bigr]
=\frac{2}{L}\mathrm{Cov}[Y_1,Y_2].
\]
For prediction-powered evaluation,
\begin{equation*}
\begin{split}
\mathrm{Var}\bigl[\widehat{\delta_{1,2}^{PP,\mathrm{Un}}}\bigr]
-\mathrm{Var}\bigl[\widehat{\delta_{1,2}^{PP}}\bigr] & =\frac1L\Bigl(
2\,\mathrm{Cov}[Y_1,Y_2] +\frac{U}{L+U}\bigl[
\mathrm{Corr}[Y_1-Y_2,F(X_1)-F(X_2)]^2\,\mathrm{Var}[Y_1-Y_2] \\
& \qquad -\mathrm{Corr}[Y_1,F(X_1)]^2\,\mathrm{Var}[Y_1] -\mathrm{Corr}[Y_2,F(X_2)]^2\,\mathrm{Var}[Y_2]
\bigr]\Bigr).
\end{split}
\end{equation*}
\end{proposition}

\begin{proof}
For human-only evaluation, the paired variance is
\[
\mathrm{Var}[\widehat{\delta_{1,2}^{H}}]
=\frac1L(\mathrm{Var}[Y_1]+\mathrm{Var}[Y_2]-2\mathrm{Cov}[Y_1,Y_2])
\]
while the unpaired variance is
\[
\mathrm{Var}[\widehat{\delta_{1,2}^{H,\mathrm{Un}}}]
=\frac1L(\mathrm{Var}[Y_1]+\mathrm{Var}[Y_2]),
\]
so their difference is $\frac{2}{L}\mathrm{Cov}[Y_1,Y_2]$.
For prediction-powered evaluation, subtracting the paired variance from the unpaired variance (with $L_1=L_2=L$, $U_1=U_2=U$) gives
\begin{equation*}
\begin{split}
\mathrm{Var}\left[\widehat{\delta_{1,2}^{PP,\mathrm{Un}}}\right] - \mathrm{Var}\left[\widehat{\delta_{1,2}^{PP}}\right] &=
\frac{1}{L}\big(\mathrm{Var}[Y_1]+\mathrm{Var}[Y_2]-\mathrm{Var}[Y_1-Y_2]\big) 
-
\frac{
U\left(\mathrm{Cov}[Y_1,F(X_1)]\right)^2
}{
L(L+U)\mathrm{Var}[F(X_1)]
} \\
& \qquad -
\frac{
U\left(\mathrm{Cov}[Y_2,F(X_2)]\right)^2
}{
L(L+U)\mathrm{Var}[F(X_2)]
} + \frac{U\left(
\mathrm{Cov}\big[
Y_1-Y_2,\,
F(X_1)-F(X_2)
\big]
\right)^2
}{
L(L+U)
\mathrm{Var}\big[
F(X_1)-F(X_2)
\big]
}.
\end{split}
\end{equation*}
Using 
\[
\mathrm{Var}[Y_1]+\mathrm{Var}[Y_2]-\mathrm{Var}[Y_1-Y_2]=2\mathrm{Cov}[Y_1,Y_2]
\]
and the identity
\[
\frac{(\mathrm{Cov}[A,B])^2}{\mathrm{Var}[B]}=\mathrm{Corr}[A,B]^2\mathrm{Var}[A]
\]
applied to each covariance term yields
\begin{equation*}
\begin{split}
\mathrm{Var}\left[\widehat{\delta_{1,2}^{PP,\mathrm{Un}}}\right]
-
\mathrm{Var}\left[\widehat{\delta_{1,2}^{PP}}\right] 
&=
\frac{1}{L}
\Bigg(
2\mathrm{Cov}[Y_1,Y_2]
+
\frac{U}{L+U}
\Big(
\mathrm{Corr}\big[
Y_1-Y_2,\,
F(X_1)-F(X_2)
\big]^2
\mathrm{Var}[Y_1-Y_2] \\
&\qquad
-
\mathrm{Corr}[Y_1,F(X_1)]^2
\mathrm{Var}[Y_1]
-
\mathrm{Corr}[Y_2,F(X_2)]^2
\mathrm{Var}[Y_2]
\Big)
\Bigg).
\end{split}
\end{equation*}
\end{proof}

When $U\gg L$ one may use $\frac{U}{L+U}\approx 1$, in which case the prediction-powered gap simplifies to
\begin{equation*}
\begin{split}
\mathrm{Var}\left[\widehat{\delta_{1,2}^{PP,\mathrm{Un}}}\right]
-
\mathrm{Var}\left[\widehat{\delta_{1,2}^{PP}}\right]
&\approx
\frac{1}{L}
\Big(
2\mathrm{Cov}[Y_1,Y_2]
+
\mathrm{Corr}\big[
Y_1-Y_2,\,
F(X_1)-F(X_2)
\big]^2
\mathrm{Var}[Y_1-Y_2] \\
&\qquad
-
\mathrm{Corr}[Y_1,F(X_1)]^2
\mathrm{Var}[Y_1]
-
\mathrm{Corr}[Y_2,F(X_2)]^2
\mathrm{Var}[Y_2]
\Big).
\end{split}
\end{equation*}

\begin{table}[t]
\centering
\begin{tabular}{lrrr}
\toprule
Dataset &  $|\widehat{\mathrm{Corr}_{\delta}}|$ &  $|\widehat{\mathrm{Corr}_{\mathrm{sep}}}|$ & Smaller \\
\midrule
WMT22 en-de & 0.227 & 0.327 & 92.6\% \\
WMT22 en-ru & 0.202 & 0.271 & 81.9\% \\
WMT22 zh-en & 0.173 & 0.342 & 95.4\% \\
WMT23 en-zh & 0.244 & 0.288 & 69.7\% \\
WMT23 ja-en & 0.215 & 0.291 & 89.7\% \\
WMT24 cs-uk & 0.246 & 0.427 & 96.1\% \\
\bottomrule
\end{tabular}
\caption{Correlation diagnostics for the prediction-powered paired-versus-unpaired comparison. Here $\widehat{\mathrm{Corr}_{\mathrm{\delta}}}= \widehat{\mathrm{Corr}}[Y_1-Y_2,F(X_1)-F(X_2)]$, and $\widehat{\mathrm{Corr}_{\mathrm{sep}}}=\frac{1}{2}(\widehat{\mathrm{Corr}}[Y_1,F(X_1)]+\widehat{\mathrm{Corr}}[Y_2,F(X_2)])$. Mean values across system pairs and automatic metrics are shown. ``Smaller'' is the percentage of cases where $|\widehat{\mathrm{Corr}_{\delta}}|<|\widehat{\mathrm{Corr}_{\mathrm{sep}}}|$.}
\label{tab:paired_unpaired_pp_corr}
\end{table}

The result above shows that in addition to $\mathrm{Cov}[Y_1,Y_2]$, the three correlation terms $\mathrm{Corr}[Y_1-Y_2,F(X_1)-F(X_2)]$, $\mathrm{Corr}[Y_1,F(X_1)]$, and $\mathrm{Corr}[Y_2,F(X_2)]$ also have a critical impact on the gap. In \S\ref{sec:exp-gap}, we observe that the unpaired design has larger estimated variance than the paired design, but this effect in prediction-powered evaluation is weaker than in the human-only cases. The correlation diagnostics in Table~\ref{tab:paired_unpaired_pp_corr} explain why: the automatic metric is usually less correlated with the human score difference $Y_1-Y_2$ than with the individual human scores $Y_1$ or $Y_2$. This implies that in a paired design the power of automatic metrics is utilized less than in an unpaired design, suggesting a direction for improving automatic metrics.

\section{Saving Ratio}
\label{app:savings}

\begin{proposition}[Saving ratio]
\label{prop:savings}
For automatic metric $F$, the relative variance reduction of paired prediction-powered evaluation over human-only evaluation is
\begin{align*}
\frac{\mathrm{Var}[\widehat{\delta_{1,2}^{H}}]-\mathrm{Var}[\widehat{\delta_{1,2}^{PP}}]}
{\mathrm{Var}[\widehat{\delta_{1,2}^{H}}]} &= \frac{U}{L+U}\,
\mathrm{Corr}\bigl[Y_1-Y_2,F(X_1)-F(X_2)\bigr]^2 \\ 
& \xrightarrow{\,U\gg L\,}
\mathrm{Corr}\bigl[Y_1-Y_2,F(X_1)-F(X_2)\bigr]^2.
\end{align*}
Equivalently, if $L_{\mathrm{eq}}$ is the human-only labeled size that matches the prediction-powered variance obtained with $L$ labeled examples, then
\[
\frac{L_{\mathrm{eq}}-L}{L_{\mathrm{eq}}}
=\frac{\mathrm{Var}[\widehat{\delta_{1,2}^{H}}]-\mathrm{Var}[\widehat{\delta_{1,2}^{PP}}]}
{\mathrm{Var}[\widehat{\delta_{1,2}^{H}}]},
\]
which is precisely the fraction of human judgments saved by using metric $F$.
\end{proposition}

\begin{proof}
From $\mathrm{Var}[\widehat{\delta_{1,2}^{H}}]=\frac1L\mathrm{Var}[Y_1-Y_2]$ and Proposition~\ref{prop:paired-opt},
\begin{equation*}
\begin{split}
\mathrm{Var}\!\left[\widehat{\delta_{1,2}^{H}}\right] -   \mathrm{Var}\!\left[\widehat{\delta_{1,2}^{PP}}\right] = \frac{
U\left(
\mathrm{Cov}\big[
Y_1-Y_2,\,
F(X_1)-F(X_2)
\big]
\right)^2
}{
L(L+U)
\mathrm{Var}\big[
F(X_1)-F(X_2)
\big]
}.
\end{split}
\end{equation*}
Dividing by $\mathrm{Var}[\widehat{\delta_{1,2}^{H}}]=\frac1L\mathrm{Var}[Y_1-Y_2]$,
\begin{equation*}
\begin{split}
\frac{
\mathrm{Var}\!\left[\widehat{\delta_{1,2}^{H}}\right]
-
\mathrm{Var}\!\left[\widehat{\delta_{1,2}^{PP}}\right]
}{
\mathrm{Var}\!\left[\widehat{\delta_{1,2}^{H}}\right]
}
&=
\frac{U}{L+U}
\frac{
\left(
\mathrm{Cov}\big[
Y_1-Y_2,\,
F(X_1)-F(X_2)
\big]
\right)^2
}{
\mathrm{Var}[Y_1-Y_2]\,
\mathrm{Var}\big[
F(X_1)-F(X_2)
\big]
}\\
&=\frac{U}{L+U}
\mathrm{Corr}\big[
Y_1-Y_2,\,
F(X_1)-F(X_2)
\big]^2,
\end{split}
\end{equation*}
which tends to $\mathrm{Corr}\big[Y_1-Y_2,\, F(X_1)-F(X_2)
\big]^2$ as $\frac{U}{L+U}\to 1$ when $U\gg L$. For the equivalence, let $L_{\mathrm{eq}}$ satisfy
$\frac{1}{L_{\mathrm{eq}}}\mathrm{Var}[Y_1-Y_2]=\mathrm{Var}[\widehat{\delta_{1,2}^{PP}}]$. Combined with $\mathrm{Var}[\widehat{\delta_{1,2}^{H}}]=\frac1L\mathrm{Var}[Y_1-Y_2]$,
\begin{equation*}
\frac{
\mathrm{Var}\!\left[\widehat{\delta_{1,2}^{H}}\right]
-
\mathrm{Var}\!\left[\widehat{\delta_{1,2}^{PP}}\right]
}{
\mathrm{Var}\!\left[\widehat{\delta_{1,2}^{H}}\right]
}
= \frac{\frac{1}{L}\mathrm{Var}[Y_1-Y_2]-\frac{1}{L_{\mathrm{eq}}}\mathrm{Var}[Y_1-Y_2]}{\frac{1}{L}\mathrm{Var}[Y_1-Y_2]}=
\frac{L_{\mathrm{eq}}-L}{L_{\mathrm{eq}}},
\end{equation*}
which is the proportion of human judgment saved while achieving the same variance.
\end{proof}

\twocolumn

\newpage

\section{Algorithms of Hypothesis Tests}
\label{app:alg}

\begin{algorithm}[H]
\caption{Human-only Paired $Z$-Test}
\label{alg:human-only-pair-z-test}
\small
\begin{algorithmic}[1]
\State \textbf{Input:} human scores $\{Y_{1j},Y_{2j}\}_{j=1}^L$
\State \textbf{Output:} one-sided p-value for testing $H_0: \delta_{1,2} \le 0$ vs. $H_1: \delta_{1,2} > 0$
\State Compute
\[
\widehat{\delta_{1,2}^{H}} \gets \frac{1}{L}\sum_{j=1}^L ({\color{blue}{Y_{1j}-Y_{2j}}} )
\]

\State Compute \[
\widehat{\mathrm{Var}}[{\color{blue}{Y_1-Y_2}}]  \gets \frac{1}{L-1}\sum_{j=1}^L ({\color{blue}{Y_{1j}-Y_{2j}}}-\widehat{\delta_{1,2}^{H}})^2
\]

\State Compute \[
\widehat{\mathrm{Var}}[\widehat{\delta_{1,2}^{H}}]  \gets \frac{1}{L}\widehat{\mathrm{Var}}[{\color{blue}{Y_1-Y_2}}]
\]

\State Compute
\[
p_{1,2}^{H} \gets 1-\Phi \big (\frac{\widehat{\delta_{1,2}^{H}}}{\sqrt{\widehat{\mathrm{Var}}[\widehat{\delta_{1,2}^{H}}]}} \big)
\]
\State \Return $p_{1,2}^{H}$
\end{algorithmic}
\end{algorithm}

\begin{algorithm}[H]
\caption{Auto-only Paired $Z$-Test}
\label{alg:auto-only-pair-z-test}
\small
\begin{algorithmic}[1]
\State \textbf{Input:} metric scores $\{F(X_{1j}),F(X_{2j})\}_{j=1}^M$
\State \textbf{Output:} metric-based one-sided p-value used as a proxy for testing $H_0: \delta_{1,2} \le 0$ vs. $H_1: \delta_{1,2} > 0$
\State Compute
\[
\widehat{\delta_{1,2}^{A}} \gets \frac{1}{M}\sum_{j=1}^M ({\color{red}{F(X_{1j})-F(X_{2j})}})
\]

\State Compute 
\begin{equation*}
\resizebox{0.44\textwidth}{!}{$\widehat{\mathrm{Var}}[{\color{red}{F(X_1)-F(X_2)}}]  \gets \frac{1}{M-1}\sum_{j=1}^M ({\color{red}{F(X_{1j})-F(X_{2j})}}-\widehat{\delta_{1,2}^{A}})^2$}
\end{equation*}

\State Compute \[
\widehat{\mathrm{Var}}[\widehat{\delta_{1,2}^{A}}]  \gets \frac{1}{M} \widehat{\mathrm{Var}}[{\color{red}{F(X_1)-F(X_2)}}] 
\]

\State Compute
\[
p_{1,2}^{A} \gets 1-\Phi \big (\frac{\widehat{\delta_{1,2}^{A}}}{\sqrt{\widehat{\mathrm{Var}}[\widehat{\delta_{1,2}^{A}}]}} \big)
\]
\State \Return $p_{1,2}^{A}$
\end{algorithmic}
\end{algorithm}

\begin{algorithm}[H]
\caption{Prediction-powered Paired $Z$-Test}
\label{alg:prediction-power-pair-z-test}
\small
\begin{algorithmic}[1]
\State \textbf{Input:} system outputs and their human scores $\{X_{1j},X_{2j},Y_{1j},Y_{2j}\}_{j=1}^L$, system outputs without human scores $\{X_{1j},X_{2j}\}_{j=L+1}^{L+U}$, automatic metric $F$.
\State \textbf{Output:} one-sided p-value for testing $H_0: \delta_{1,2} \le 0$ vs. $H_1: \delta_{1,2} > 0$

\State Compute $M\gets L+U$

\State Set $\color{blue}{d_j \gets Y_{1j}-Y_{2j}}$ for $j=1,\ldots,L$ 

\State Set $\color{red}{f_j\gets F(X_{1j})-F(X_{2j})}$ for $j=1,\ldots,M$

\State Compute
\[
\bar d_L \gets \frac{1}{L}\sum_{j=1}^L d_j
\]

\State Compute \[
\widehat{\mathrm{Var}}[d]  \gets \frac{1}{L-1}\sum_{j=1}^L (d_j-\bar d_L)^2
\]

\State Compute
\begin{align*}
\bar f_L &\gets \frac{1}{L}\sum_{j=1}^{L} f_j, &
\bar f_U &\gets \frac{1}{U}\sum_{j=L+1}^{L+U} f_j, \\
\bar f_M &\gets \frac{1}{M}\sum_{j=1}^{M} f_j
\end{align*}

\State Compute \[
\widehat{\mathrm{Var}}[f]  \gets \frac{1}{M-1}\sum_{j=1}^{M} (f_j-\bar f_M)^2
\]

\State Compute \[
\widehat{\mathrm{Cov}}[d,f] 
\gets \frac{1}{L-1}\sum_{j=1}^L (d_j-\bar d_L)(f_j-\bar f_L)
\]

\State Compute \[
\hat\lambda \gets \frac{\widehat{\mathrm{Cov}}[d,f]}{(1+\frac{L}{U})\widehat{\mathrm{Var}}[f]}
\]

\State Compute \[
\widehat{\delta_{1,2}^{PP}} \gets
\hat\lambda \bar f_U 
 + \frac{1}{L}\sum_{j=1}^L (d_j - \hat\lambda f_j)
\]

\State Compute \[
\widehat{\mathrm{Var}}[\widehat{\delta_{1,2}^{PP}}]  \gets \frac{1}{L}\widehat{\mathrm{Var}}[d] -
\frac{
U\left(
\widehat{\mathrm{Cov}}\big[d,f\big]\right)^2
}{
L(L+U)
\widehat{\mathrm{Var}}\big[f\big]
}
\]

\State Compute
\[
p_{1,2}^{PP} \gets 1-\Phi \big (\frac{\widehat{\delta_{1,2}^{PP}}}{\sqrt{\widehat{\mathrm{Var}}[\widehat{\delta_{1,2}^{PP}}]}} \big)
\]
\State \Return $p_{1,2}^{PP}$
\end{algorithmic}
\end{algorithm}

\begin{algorithm}[H]
\caption{Human-only Paired Permutation Test}
\label{alg:human-only-pair-perm-test}
\small
\begin{algorithmic}[1]
\State \textbf{Input:} human scores $\{Y_{1j},Y_{2j}\}_{j=1}^L$, number of permutations $B$. 
\State \textbf{Assert}: $Y_1-Y_2$ is symmetric about $\delta_{1,2}$.
\State \textbf{Output:} one-sided p-value for testing $H_0: \delta_{1,2} \le 0$  vs. $H_1: \delta_{1,2} > 0$ 
\State Compute
\[
\widehat{\delta_{1,2}^{H}} \gets \frac{1}{L}\sum_{j=1}^L ({\color{blue}{Y_{1j}-Y_{2j}}})
\]
\For{$b = 1, \ldots, B$}
    \State Draw $\varepsilon_j^{(b)} \overset{\text{iid}}{\sim} \mathrm{Uniform}\{-1,+1\}$ for $j=1,\ldots,L$
    \State Compute
    \[
    \widehat{\delta_{1,2}^{H,(b)}} \gets \frac{1}{L}\sum_{j=1}^L \varepsilon_j^{(b)} ({\color{blue}{Y_{1j}-Y_{2j}}})
    \]
\EndFor
\State Compute
\[
p_{1,2}^{H} \gets \frac{1+\sum_{b=1}^B \mathbb{1}\!\left\{\widehat{\delta_{1,2}^{H,(b)}} \ge \widehat{\delta_{1,2}^{H}}\right\}}{B+1}
\]
\State \Return $p_{1,2}^{H}$
\end{algorithmic}
\end{algorithm}

\begin{algorithm}[H]
\caption{Auto-only Paired Permutation Test}
\label{alg:Auto-only-pair-perm-test}
\small
\begin{algorithmic}[1]
\State \textbf{Input:} metric scores $\{F(X_{1j}),F(X_{2j})\}_{j=1}^M$, number of permutations $B$
\State \textbf{Output:} one-sided p-value for testing $H_0: \delta_{1,2} \le 0$ vs. $H_1: \delta_{1,2} > 0$
\State Compute
\[
\widehat{\delta_{1,2}^{A}} \gets \frac{1}{M}\sum_{j=1}^M ({\color{red}{F(X_{1j})-F(X_{2j})}})
\]
\For{$b = 1, \ldots, B$}
    \State Draw $\varepsilon_j^{(b)} \overset{\text{iid}}{\sim} \mathrm{Uniform}\{-1,+1\}$ for $j=1,\ldots,M$
    \State Compute
    \[
    \widehat{\delta_{1,2}^{A,(b)}} \gets \frac{1}{M}\sum_{j=1}^M \varepsilon_j^{(b)} ({\color{red}{F(X_{1j})-F(X_{2j})}})
    \]
\EndFor
\State Compute
\[
p_{1,2}^{A} \gets \frac{1+\sum_{b=1}^B \mathbb{1}\!\left\{\widehat{\delta_{1,2}^{A,(b)}} \ge \widehat{\delta_{1,2}^{A}}\right\}}{B+1}
\]
\State \Return $p_{1,2}^{A}$
\end{algorithmic}
\end{algorithm}

\begin{algorithm}[H]
\caption{\small Prediction-powered Paired Permutation Test}
\label{alg:prediction-power-pair-perm-test}
\small
\begin{algorithmic}[1]
\State \textbf{Input:} system outputs and their human scores $\{X_{1j},X_{2j},Y_{1j},Y_{2j}\}_{j=1}^L$, system outputs without human scores $\{X_{1j},X_{2j}\}_{j=L+1}^{L+U}$, automatic metric $F$, number of permutations $B$
\State \textbf{Assert}: The joint distribution of
$(Y_1-Y_2,\; F(X_1)-F(X_2))$ is centrally symmetric about
$(\delta_{1,2},\; \mathbb{E}[F(X_1)-F(X_2)])$.
\State \textbf{Output:} one-sided p-value for testing $H_0: \delta_{1,2} \le 0$ vs. $H_1: \delta_{1,2} > 0$
\State Compute $M \gets L+U$
\State Set $\color{blue}{d_j \gets Y_{1j}-Y_{2j}}$ for $j=1,\ldots,L$ 

\State Set $\color{red}{f_j\gets F(X_{1j})-F(X_{2j})}$ for $j=1,\ldots,M$
\State Center the metric score differences: $\bar f_M \gets \frac{1}{M}\sum_{j=1}^{M} f_j$, then set $f_j \gets f_j-\bar f_M$ for $j=1,\dots,M$
\State Compute $\widehat{\delta_{1,2}^{PP}}$ from $\{d_j\}_{j=1}^L,\{f_j\}_{j=1}^M$ via steps~6--11 of Alg.~\ref{alg:prediction-power-pair-z-test}
\For{$b = 1, \ldots, B$}
    \State Draw $\varepsilon_j^{(b)} \overset{\text{iid}}{\sim} \mathrm{Uniform}\{-1,+1\}$ for $j=1,\ldots,M$
    \State Set $d_j^{(b)} \gets \varepsilon_j^{(b)} d_j$ for $j=1,\ldots,L$
    \State Set $f_j^{(b)} \gets \varepsilon_j^{(b)} f_j$ for $j=1,\ldots,M$
    \State Compute $\widehat{\delta_{1,2}^{PP,(b)}}$ from $\{d_j^{(b)}\}_{j=1}^L,\{f_j^{(b)}\}_{j=1}^M$ via steps~6--11 of Alg.~\ref{alg:prediction-power-pair-z-test} (re-estimating $\hat\lambda$)
\EndFor
\State Compute
\[
p_{1,2}^{PP} \gets \frac{1+\sum_{b=1}^B \mathbb{1}\!\left\{\widehat{\delta_{1,2}^{PP,(b)}} \ge \widehat{\delta_{1,2}^{PP}}\right\}}{B+1}
\]
\State \Return $p_{1,2}^{PP}$
\end{algorithmic}
\end{algorithm}

\section{Simulation Studies of Type I Error}
\label{app:simulation_type_i}

\S\ref{sec:exp-nonparam} shows that the paired $Z$-test achieves higher empirical power than the paired permutation test in both the human-only and prediction-powered settings. This raises a natural question about calibration: how do the two tests behave under the null hypothesis, and do they control the Type I error rate at the nominal level?

Answering this question requires data generated under the null hypothesis. For the null hypothesis $\delta_{1,2}\le 0$, we evaluate the test at the boundary $\delta_{1,2}=0$, where the Type I error rate is maximized. The paired permutation test additionally assumes that the distribution of $Y_1-Y_2$ is symmetric about $\delta_{1,2}$. Together, these conditions imply that $Y_1-Y_2$ must follow a distribution that is symmetric about zero.

Such null conditions are difficult to obtain from real-world data, as any two systems are likely to differ to some extent and therefore may not satisfy $\delta_{1,2}=0$. We therefore conduct a simulation study in which data are drawn from a distribution that satisfies the null assumptions of both the paired $Z$-test and the paired permutation test.

\subsection{Heavy-tailed null simulations}

To examine the error in the normal approximation of $Z$-test, we take the null distribution to be a bivariate Student's $t$ with a mean of zero. For each Monte Carlo trial, we independently draw $U+L$ observations,
\[
(d_j,f_j) \sim t_\nu(0,\Sigma),
\qquad j=1,\ldots,U+L,
\]
where
\[
\Sigma =
\begin{pmatrix}
1 & \rho \\
\rho & 1
\end{pmatrix},
\]
$\rho \in \{0.3,0.7\}$, and $\nu \in \{3,10,\infty\}$. Here, $d_j$ denotes the human score difference between the two systems, and $f_j$ denotes the corresponding metric score difference. A larger $\rho$ indicates a stronger correlation between the metric and human score differences. As $\nu$ increases, the bivariate $t$ distribution approaches a bivariate normal distribution. $\nu=\infty$ corresponds to
\[
(d_j,f_j)\sim \mathcal{N}(0,\Sigma).
\]
For finite $\nu$, we adjust the scale matrix of the $t$ distribution so that the covariance matrix of $(d_j,f_j)$ is exactly $\Sigma$.

Among the $U+L$ sampled pairs, $L$ are treated as labeled examples, for which both $d_j$ and $f_j$ are retained. The remaining $U$ are treated as unlabeled examples: their metric score differences $f_j$ are retained, whereas their human score differences $d_j$ are removed. Thus, the labeled sample consists of $\{(d_j,f_j)\}_{j=1}^{L}$, and the unlabeled sample consists of $\{f_j\}_{j=L+1}^{L+U}$. Under the null hypothesis, $\mathbb{E}[d]=0$, $\mathbb{E}[f]=0$, $\mathrm{Var}(d)=1$, $\mathrm{Var}(f)=1$, and $\mathrm{Cov}(d,f)=\rho$.

Following \S\ref{sec:exp-nonparam}, we fix $U=800$ unlabeled examples and vary the number of labeled examples over $L \in \{20,40,60,80,100,120,140,160,180,200\}$.
For each $(\rho,\nu,L)$ configuration, we run $10000$ Monte Carlo trials. Within each trial, the permutation tests use $B=1000$ random sign-flip permutations. All tests are one-sided with significance level $\alpha=0.05$.

\paragraph{The Oracle Z-Tests.}
In the human-only setting, the original paired $Z$-test uses the sample variance $\widehat{\mathrm{Var}}[d]$. In the prediction-powered setting, the original paired $Z$-test also estimates $\widehat{\mathrm{Var}}[d]$, $\widehat{\mathrm{Var}}[f]$, and $\widehat{\mathrm{Cov}}[d,f]$ from the sampled data.

To further examine the potential finite-sample instability in
the plug-in estimates, the oracle $Z$-test instead plugs in the true nuisance values $\mathrm{Var}[d]$, $\mathrm{Var}[f]$, and $\mathrm{Cov}[d,f]$ in both the human-only setting and the prediction-powered setting.

\paragraph{Results.}
In the human-only setting (Figure~\ref{fig:type_i_error_human_only}), the paired permutation test stays very close to the nominal level across all configurations, while the original paired $Z$-test is slightly anti-conservative, with a maximum Type I error of $0.059$. The Type I error of oracle paired $Z$-test is lower than the nominal level at $\rho=0.7, \nu=3$, and this deviation reflects the error in the normal approximation.

In the prediction-powered setting (Figure~\ref{fig:type_i_error_ppi} and Table~\ref{tab:type_i_error_ppi_l20}), the paired permutation test still stays very close to the nominal level across all configurations, while the original prediction-powered paired $Z$-test can be substantially anti-conservative when $L$ is small: its largest Type I error is $0.123$ at $\rho=0.7$, $\nu=3$, and $L=20$. Even in the normal setting $\nu=\infty$, it remains inflated for small $L$, reaching $0.104$ at $\rho=0.7$ and $L=20$. The type I error of oracle paired $Z$-test is still lower than the nominal level at $\rho=0.7, \nu=3$, and this deviation reflects the error in the normal approximation as in the human-only setting.

\begin{table}[]
\centering
\small
\begin{tabular}{ccccc}
\toprule
$\rho$ & $\nu$ & PPI Z & PPI Z (oracle) & PPI permutation \\
\midrule
$0.3$ & $3$ & 0.088 & 0.043 & 0.051 \\
$0.3$ & $10$ & 0.069 & 0.049 & 0.048 \\
$0.3$ & $\infty$ & 0.072 & 0.052 & 0.049 \\
$0.7$ & $3$ & 0.123 & 0.041 & 0.049 \\
$0.7$ & $10$ & 0.114 & 0.052 & 0.050 \\
$0.7$ & $\infty$ & 0.104 & 0.048 & 0.048 \\
\bottomrule
\end{tabular}
\caption{Empirical Type I error for prediction-powered tests at the smallest labeled sample size, $L=20$.}
\label{tab:type_i_error_ppi_l20}
\end{table}

Taken together, these results indicate that the Z-test's miscalibration is not driven merely by the inaccuracy of the normal approximation. \textbf{It also arises from the finite-sample instability in the plug-in variance and covariance estimators used to construct the test statistic.} This source of error appears to be more consequential in the prediction-powered setting than in the human-only setting. In the human-only setting, only $\widehat{\mathrm{Var}}[d]$ must be estimated, so the resulting distortion is relatively small. In the prediction-powered test, $\widehat{\mathrm{Var}}[f]$ and $\widehat{\mathrm{Cov}}[d,f]$ must also be estimated, making the test substantially more sensitive to small labeled samples. The permutation test seems to avoid this problem by recomputing the full statistic under each sign flip.

\begin{figure*}[]
\centering
\includegraphics[width=\linewidth]{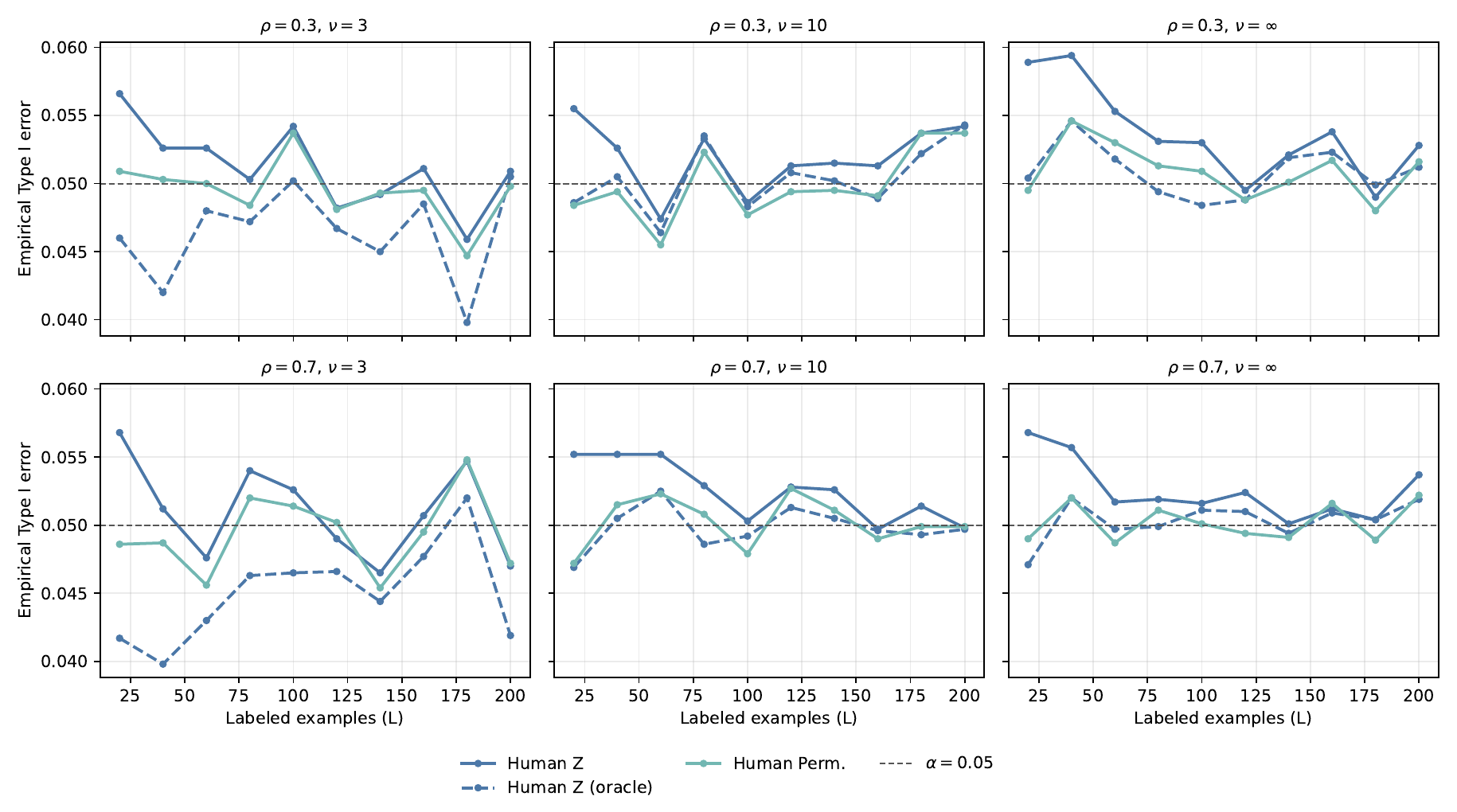}
\caption{Human-only Type I error under the null simulation. The dashed blue curve uses the true variance $\mathrm{Var}[d]=1$. The oracle and permutation tests stay close to the nominal level $\alpha=0.05$, while the original paired $Z$-test is mildly anti-conservative for small $L$.}
\label{fig:type_i_error_human_only}
\end{figure*}

\begin{figure*}[]
\centering
\includegraphics[width=\linewidth]{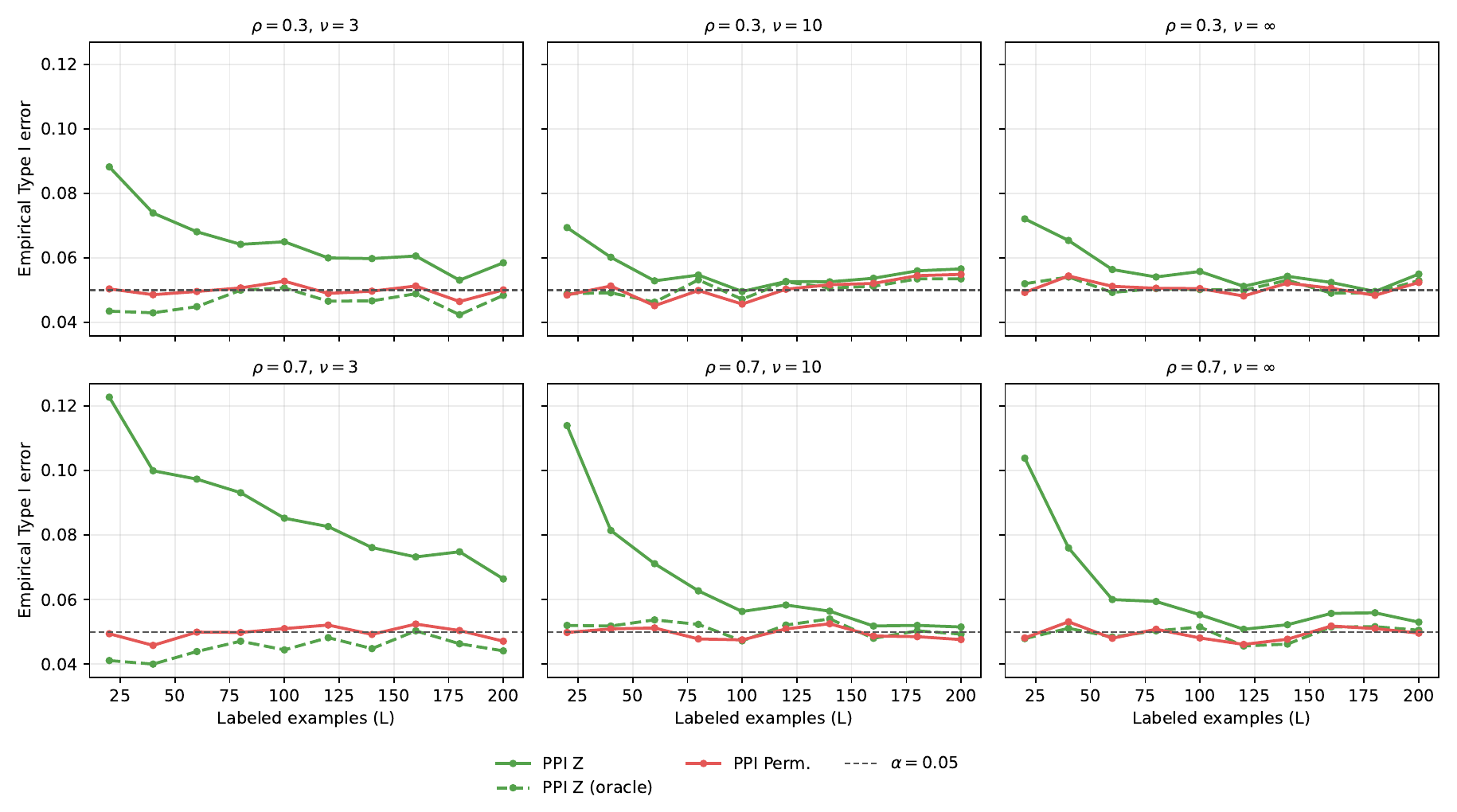}
\caption{Prediction-powered Type I error under the null simulation. The dashed green curve uses the true values of $\mathrm{Var}[d]$, $\mathrm{Var}[f]$, and $\mathrm{Cov}[d,f]$. The original prediction-powered $Z$-test can be substantially anti-conservative for small $L$, especially when $\rho=0.7$.}
\label{fig:type_i_error_ppi}
\end{figure*}

\begin{figure*}[]
\centering
\includegraphics[width=\linewidth]{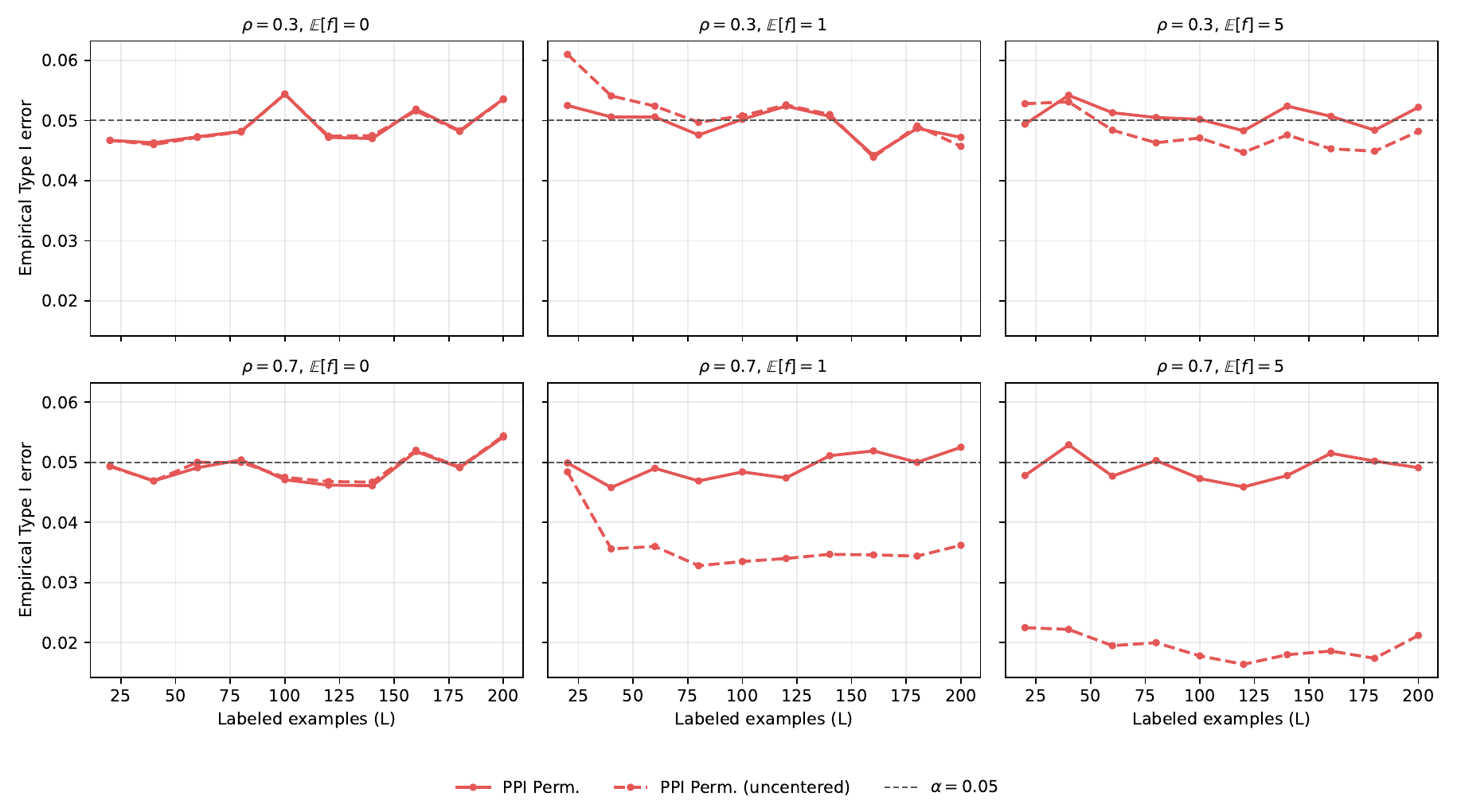}
\caption{Effect of centering metric differences in the prediction-powered paired permutation test. The simulations use a bivariate normal null with $\mathbb{E}[d]=0$ and varying $\mathbb{E}[f]$.}
\label{fig:type_i_error_metric_mean_ppi_perm_centering}
\end{figure*}

\subsection{Metric-shifted null simulations}
The preceding simulations assume $\mathbb{E}[d]=\mathbb{E}[f]=0.$ In practice, when the human score differences are centered at zero, the metric score differences need not be centered at zero  because automatic metrics are biased. Consequently, $\mathbb{E}[f]\neq 0$ is the typical case under $\mathbb{E}[d]=0$. This seems to be in tension with the sign-flip construction in the prediction-powered paired permutation test. Because the test applies the same sign flip $\varepsilon_j$ to $d_j$ and $f_j$, it requires that both distributions of $d$ and $f$ are symmetric about zero, which leads to $\mathbb{E}[f]=0$ in addition to $\mathbb{E}[d]=0$, a condition the metric has no reason to satisfy. Fortunately, centering the metric score differences in Step 7 of Alg.~\ref{alg:prediction-power-pair-perm-test} removes this dependence, rendering $\mathbb{E}[f]=0$ unnecessary.

To isolate this effect, we run an additional null simulation with $(d_j,f_j)$ drawn from a bivariate normal distribution satisfying $\mathbb{E}[d]=0$, $\mathrm{Var}(d)=1$, $\mathrm{Var}(f)=1$, and $\mathrm{Cov}(d,f)=\rho$, while varying $\mathbb{E}[f]\in \{0, 1,5\}$. All other settings remain unchanged. 

Figure~\ref{fig:type_i_error_metric_mean_ppi_perm_centering} compares the prediction-powered paired permutation test with an uncentered variant obtained by deleting the Step 7 of Alg.~\ref{alg:prediction-power-pair-perm-test}. As shown in the figure, the uncentered permutation test can become overly conservative when the metric differences have a nonzero mean, especially when the correlation between the human score differences and metric score differences is high. Centering the metric differences keeps the Type I error close to the nominal level $\alpha=0.05$ across the simulated mean shifts.

\begin{figure*}[t]
\centering
\includegraphics[width=\textwidth]{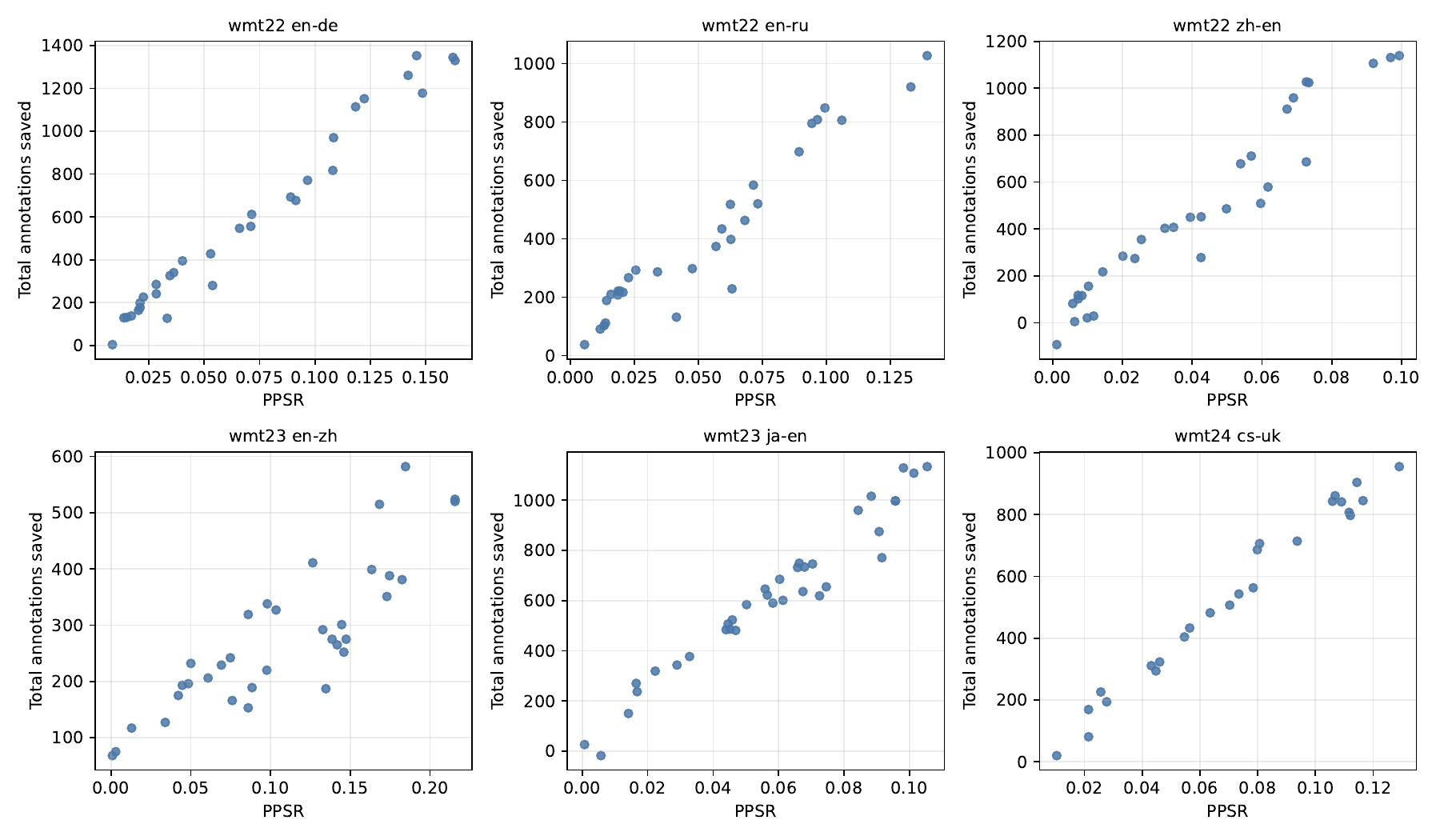}
\caption{Relationship between PPSR and the total number of human annotations saved.}
\label{fig:annotation_saving_ppsr_vs_total_saved}
\end{figure*}

\section{Actual Annotation Savings and PPSR}
\label{app:annotation_saving_vs_ppsr}

PPSR is derived as a theoretical saving ratio in \S\ref{sec:ppsr}. To check whether it reflects actual annotation savings, we estimate the number of human annotations required by human-only evaluation and prediction-powered evaluation via resampling. Specifically, for each dataset and each system pair, we treat the full-dataset human mean difference as the population effect size. We then determine the minimum labeled sample size needed to reach 80\% empirical power for the human-only paired $Z$-test, denoted $L_h$, and for the prediction-powered paired $Z$-test using metric $F_k$, denoted $L_k$. The annotations saved by $F_k$ on that system pair are defined as $L_h-L_k$. We exclude system pairs for which the human-only test does not reach the target power and sum $L_h-L_k$ over the remaining system pairs to obtain the total annotations saved by each metric. For consistency, PPSR is also recomputed for each metric using this same set of remaining system pairs, rather than all original system pairs.

\begin{table}[]
\centering
\small
\begin{tabular}{lrrr}
\toprule
Dataset & \#Metrics & Pearson $r$ & Spearman $\rho$ \\
\midrule
WMT22 en-de & 31 & 0.987 & 0.970 \\
WMT22 en-ru & 30 & 0.957 & 0.956 \\
WMT22 zh-en & 31 & 0.966 & 0.972 \\
WMT23 en-zh & 34 & 0.859 & 0.856 \\
WMT23 ja-en & 35 & 0.974 & 0.964 \\
WMT24 cs-uk & 25 & 0.989 & 0.978 \\
\bottomrule
\end{tabular}
\caption{Correlation between total annotations saved and PPSR across automatic metrics, computed separately for each dataset.}
\label{tab:annotation_saving_ppsr_corr}
\end{table}

Figure~\ref{fig:annotation_saving_ppsr_vs_total_saved} compares total annotations saved with PPSR across automatic metrics. Each point represents one automatic metric. Table~\ref{tab:annotation_saving_ppsr_corr} reports the corresponding Pearson and Spearman correlations. The correlations are high across all six datasets, indicating that PPSR is not only theoretically well-grounded, but also tracks the empirical annotation savings produced by prediction-powered evaluation using different automatic metrics. 

The relationship is not exact because actual annotation savings depend on factors beyond the relative variance reduction summarized by PPSR. In particular, "easy" system pairs offer limited room for savings because the human-only test already reaches the target power with few labels, whereas "harder" pairs benefit substantially more from a strong automatic metric. Furthermore, the choice of target power plays a role: shifting the $80\%$ threshold changes the required sample sizes and therefore the total savings.

\begin{table*}[t]
\centering
\small
\setlength{\tabcolsep}{4pt}
\begin{tabular}{lrrrrrr}
\toprule
Dataset & Max & Group-by-Item & No-Grouping & Group-by-System & PDP & PPSR \\
\midrule
WMT22 en-de & 465 & \textbf{421} & 365 & 358 & 403 & 412 \\
WMT22 en-ru & 435 & \textbf{388} & 330 & 319 & 372 & 373 \\
WMT22 zh-en & 465 & \textbf{431} & 382 & 379 & 418 & 422 \\
WMT23 en-zh & 561 & \textbf{508} & 494 & 421 & 507 & 507 \\
WMT23 ja-en & 595 & \textbf{506} & 483 & 468 & 483 & 500 \\
WMT24 cs-uk & 300 & \textbf{256} & 225 & 229 & 237 & 247 \\
\bottomrule
\end{tabular}
\caption{Discriminative power of segment-level meta-metrics and PPSR. Max gives the largest possible value for each criterion. Bold indicates the highest observed value in each row.}
\label{tab:segment_meta_metric_discriminative_power}
\vspace{-5mm}
\end{table*}

\begin{figure*}[t]
\centering
\includegraphics[width=\textwidth]{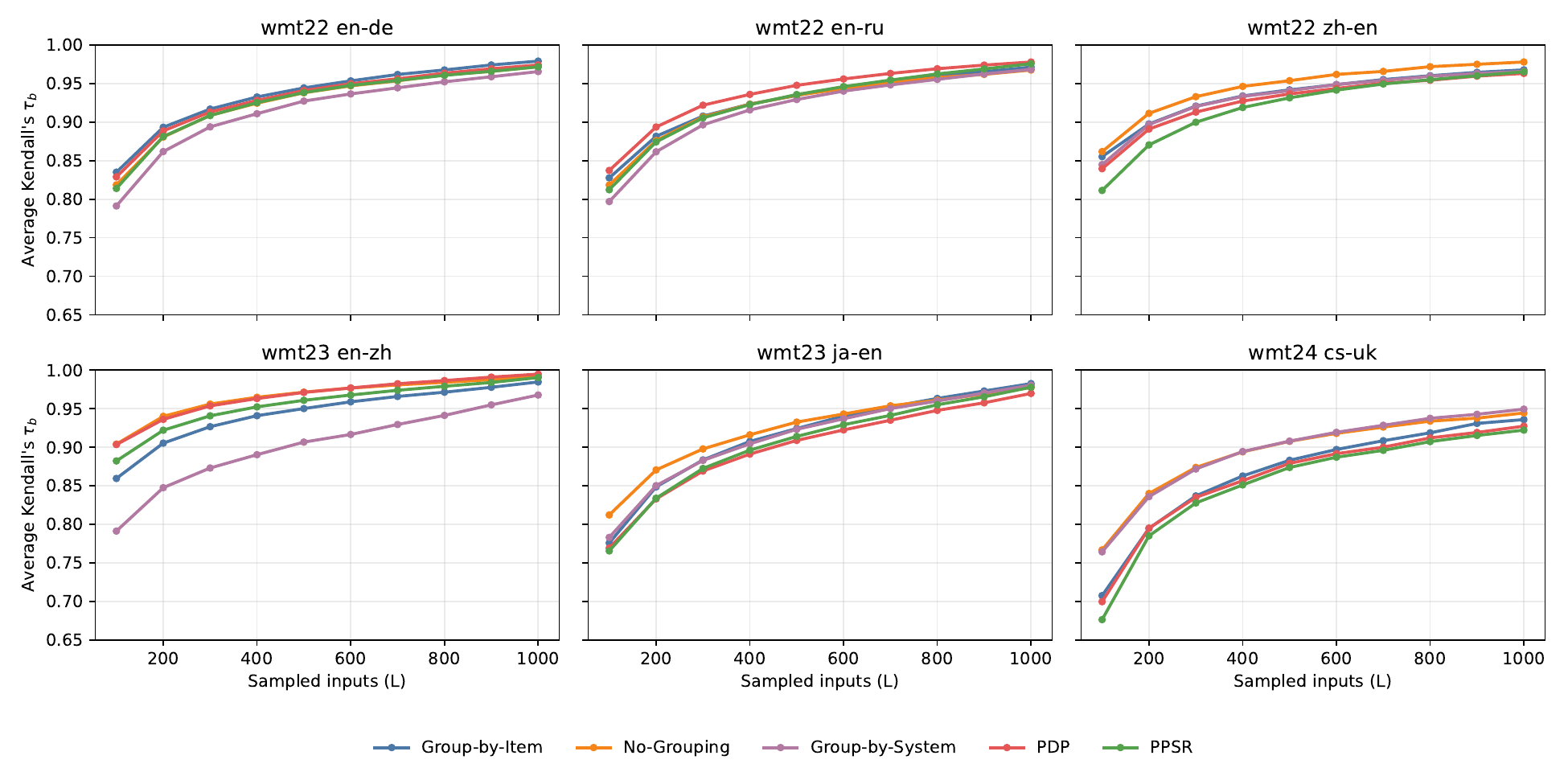}
\caption{Ranking stability of segment-level meta-metrics and PPSR. Higher values indicate that the metric ranking computed from a subsample is closer to the full-dataset ranking.}
\label{fig:ranking_stability_segment_all_datasets}
\vspace{-5mm}
\end{figure*}

\section{Comparison with Segment-level Meta-metrics}
\label{app:comparison_segment_metametric}

\subsection{Conceptual Comparison}

Following \citep{DBLP:conf/emnlp/DeutschFF23}, segment-level meta-metrics can be categorized according to three primary grouping strategies:
\begin{enumerate}
\item \textbf{Group-by-Item (Input level):}
$$\frac{1}{L}\sum_{j=1}^{L} r\!\left(\left\{\left(Y_{ij},\; F_k(X_{ij})\right)\right\}_{i=1}^{N}\right)$$
\item \textbf{No-Grouping (Global level):}
$$r\!\left(\left\{\left(Y_{ij},\; F_k(X_{ij})\right)\right\}_{i=1,j=1}^{N,L}\right)$$
\item \textbf{Group-by-System:}
$$\frac{1}{N}\sum_{i=1}^{N} r\!\left(\left\{\left(Y_{ij},\; F_k(X_{ij})\right)\right\}_{j=1}^{L}\right)$$
\end{enumerate}
To ensure consistency with PPSR, we uniformly adopt the sample Pearson correlation coefficient $r$ as the correlation metric. PPSR can be interpreted as following a \textit{Group-by-System-Pair} scheme if the square of the Pearson correlation is ignored:
\begin{equation*}
\begin{split}
\binom{N}{2}^{-1}\sum_{p=1}^{N-1}\sum_{q=p+1}^{N} r\!\left(\left\{\left(Y_{pj}-Y_{qj},\; F_k(X_{pj})-F_k(X_{qj})\right)\right\}_{j=1}^{L}\right).
\end{split}
\end{equation*}

Although PDP \citep{DBLP:conf/emnlp/DiIanniD25} also operates on score differences between systems on the same segment, its grouping strategy differs from that of PPSR. Specifically, PDP corresponds to a \textit{No-Grouping} scheme:
\begin{equation*}
\begin{split}
r\!\left(\left\{\left( Y_{pj}-Y_{qj},\; F_k(X_{pj})-F_k(X_{qj})\right)\right\}_{p=1,q=1,j=1}^{N,N,L}\right)
\end{split}
\end{equation*}
Furthermore, some readers may argue that PPSR is not a system-level meta-metric because it uses the score of each segment. We disagree with this restriction. Under that reasoning, SPA would also not be a system-level meta-metric, because it uses the score of each segment when calculating the $p$-values of system comparisons from either human or automatic scores. Nevertheless, SPA served as an official system-level meta-metric in both the 2024 \citep{DBLP:conf/wmt/FreitagMDLAR0BK24} and 2025 \citep{lavie-etal-2025-findings} WMT Metrics Shared Tasks. Ultimately, whether a meta-metric is system-level should depend on its meaning and intended use case.

\subsection{Discriminative Power}

To compare the discriminative power of segment-level meta-metrics and PPSR, we use the number of distinct values and the number of significant comparisons as in \S\ref{sec:exp-ppsr-disc}. Because PPSR and all segment-level meta-metrics attain the maximum possible number of distinct values, we focus on the number of significant comparisons. Table~\ref{tab:segment_meta_metric_discriminative_power} shows that \textit{Group-by-Item} (input level) is the most discriminative meta-metric on most datasets. PPSR is usually close to \textit{Group-by-Item}. This suggests that PPSR remains competitive even when compared against segment-level meta-metrics.

\subsection{Ranking Stability}

We also compare ranking stability under the same resampling protocol used in \S\ref{sec:exp-ppsr-stability}. Figure~\ref{fig:ranking_stability_segment_all_datasets} shows that the segment-level meta-metrics are often slightly more stable than PPSR, especially No-Grouping and PDP. However, PPSR remains competitive across datasets.

\section{Additional Figures and Tables}
\label{app:addition_alg_fig}

\begin{figure*}[]
\centering
\includegraphics[width=\textwidth]{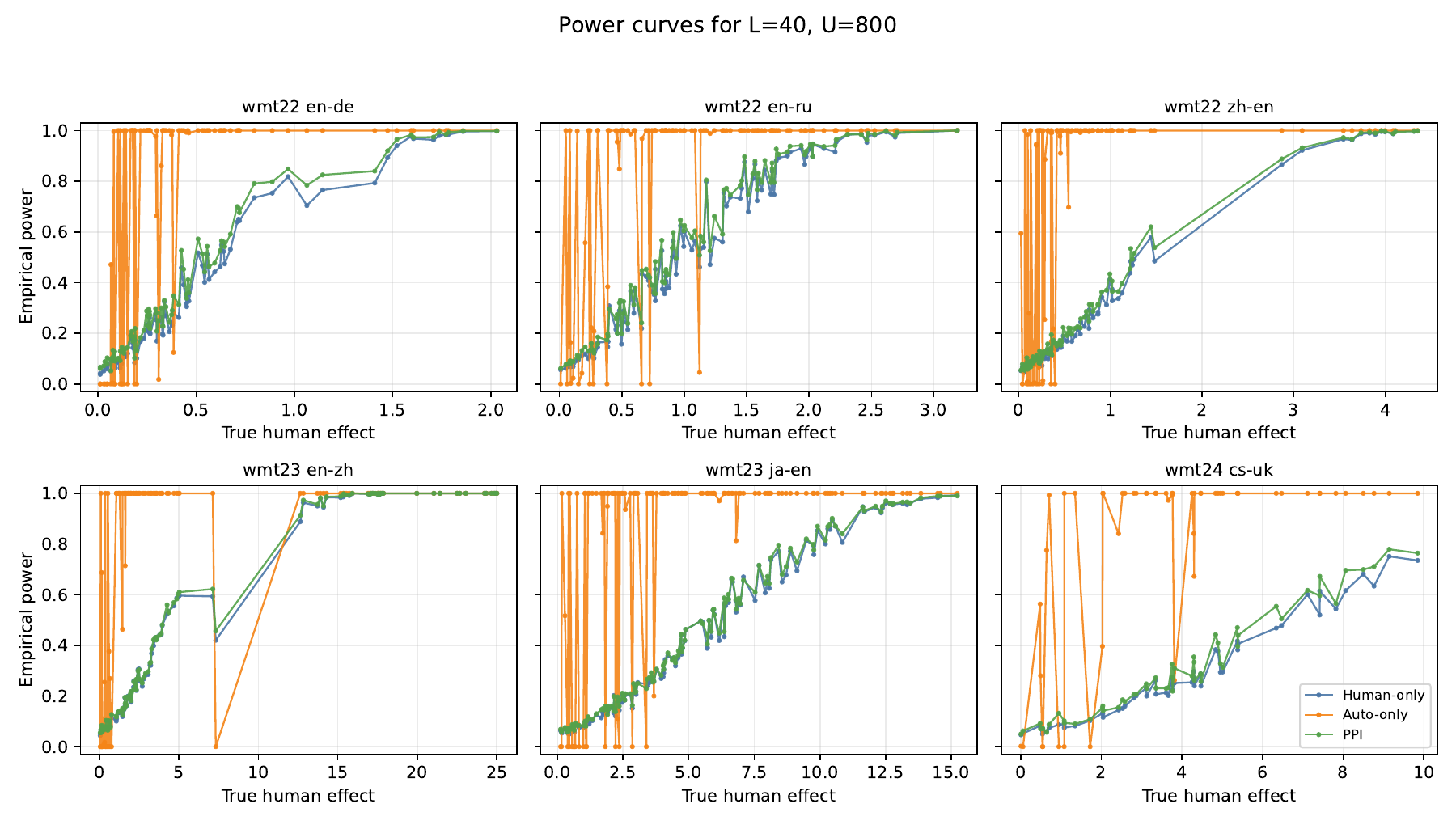}
\caption{Empirical power curves for $L=40$ and $U=800$. Each point is one ordered system pair, and the x-axis is the full-population human-score difference. MetricX is used as the automatic metric.} 
\label{fig:basic_power_l40}
\end{figure*}

\begin{figure*}[]
\centering
\includegraphics[width=\textwidth]{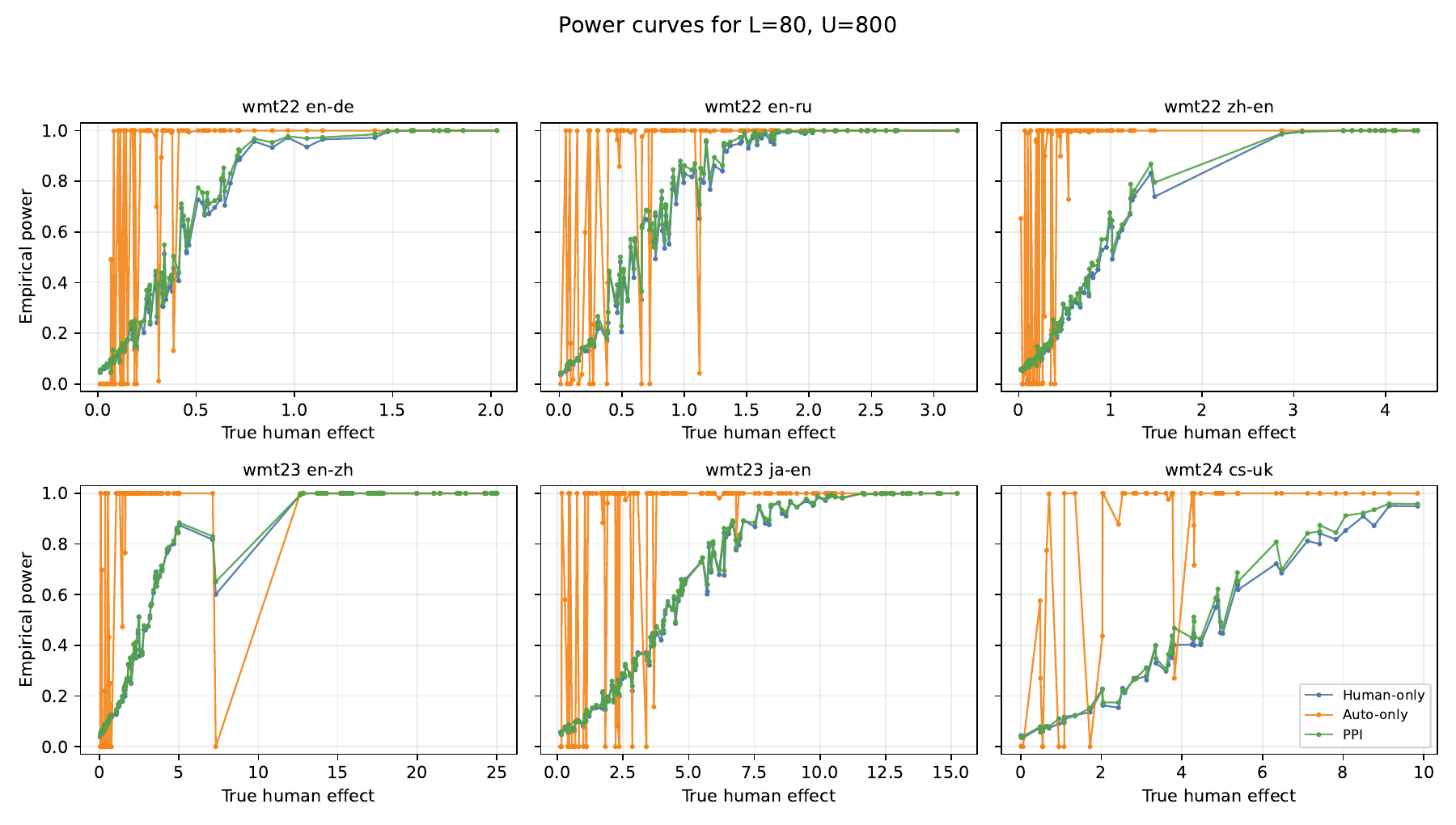}
\caption{Empirical power curves for $L=80$ and $U=800$. Increasing the labeled sample size improves both human-only and prediction-powered tests. MetricX is used as the automatic metric.}
\label{fig:basic_power_l80}
\end{figure*}

\begin{figure*}[]
\centering
\includegraphics[width=\textwidth]{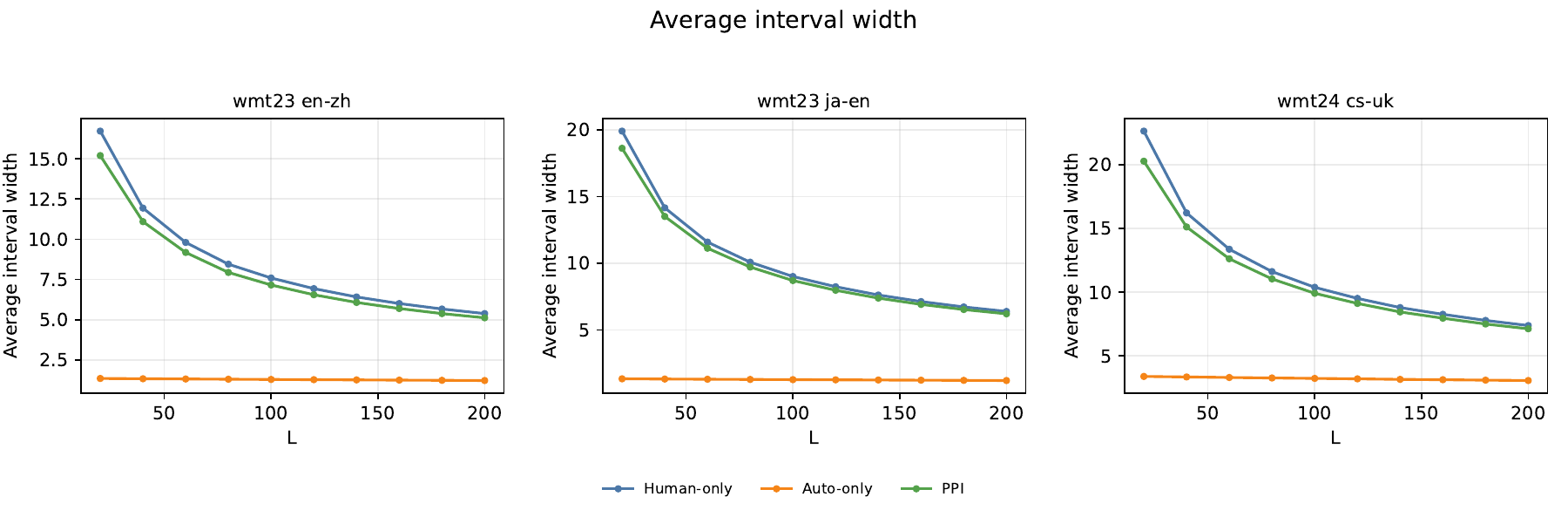}
\caption{Average two-sided 95\% confidence-interval width with $U=800$. GEMBA is used as the automatic metric. Results are averaged across system pairs.}
\label{fig:gemba_basic_interval_width}
\end{figure*}

\begin{figure*}[]
\centering
\includegraphics[width=\textwidth]{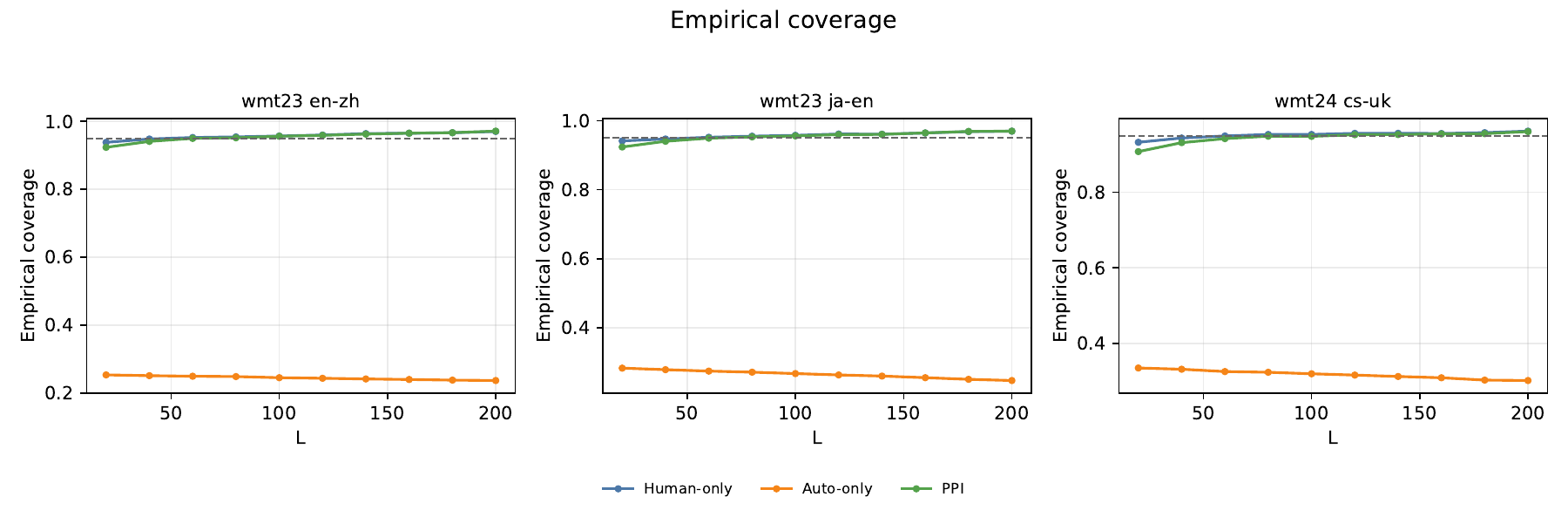}
\caption{Empirical coverage of the two-sided 95\% confidence intervals. GEMBA is used as the automatic metric. The dashed line marks the nominal 95\% level.}
\label{fig:gemba_basic_interval_coverage}
\end{figure*}

\begin{figure*}[]
\centering
\includegraphics[width=\textwidth]{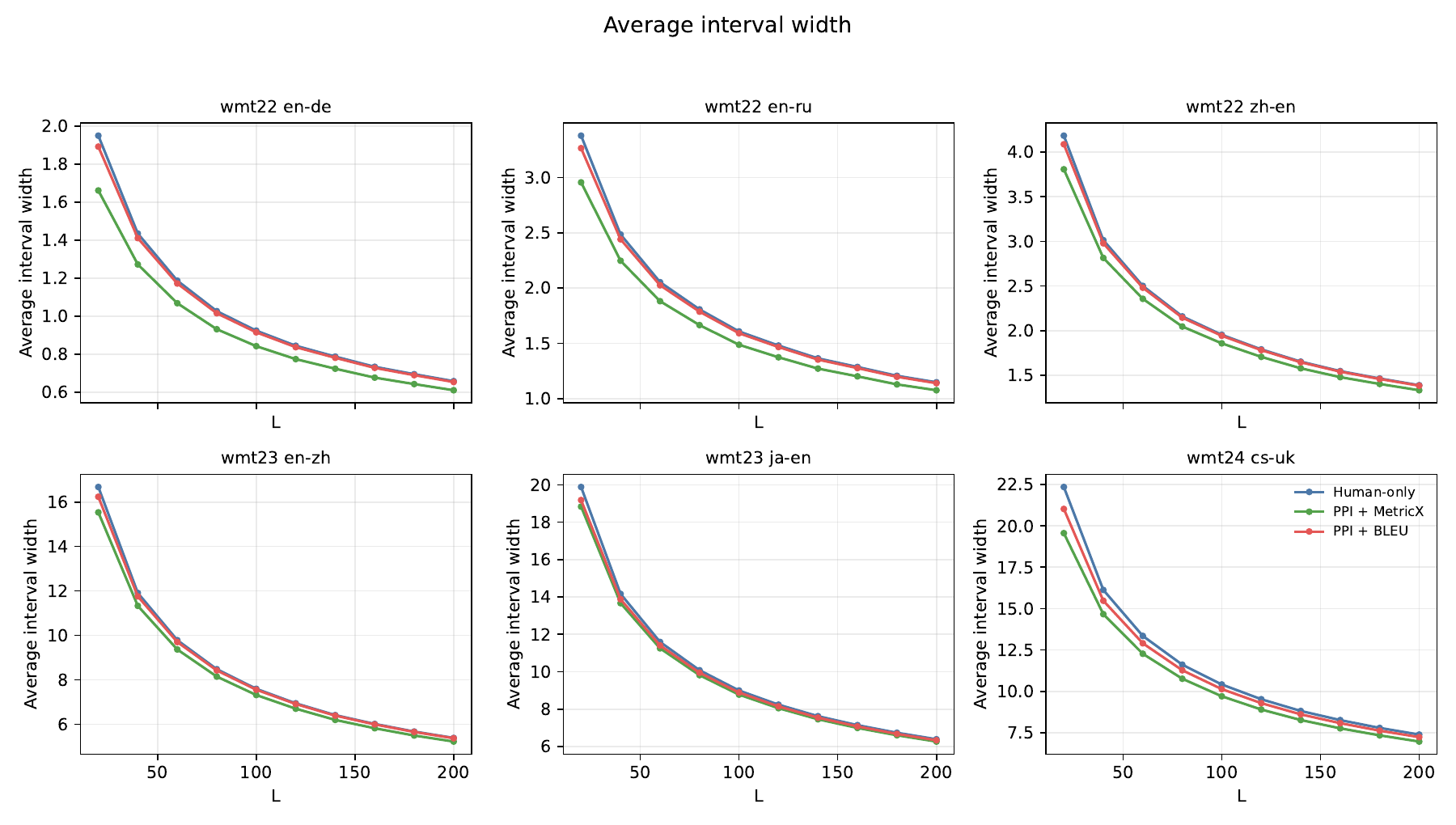}
\caption{Average two-sided 95\% confidence-interval width with $U=800$. Results are averaged across system pairs.}
\label{fig:basic_interval_width}
\end{figure*}

\begin{figure*}[]
\centering
\includegraphics[width=\textwidth]{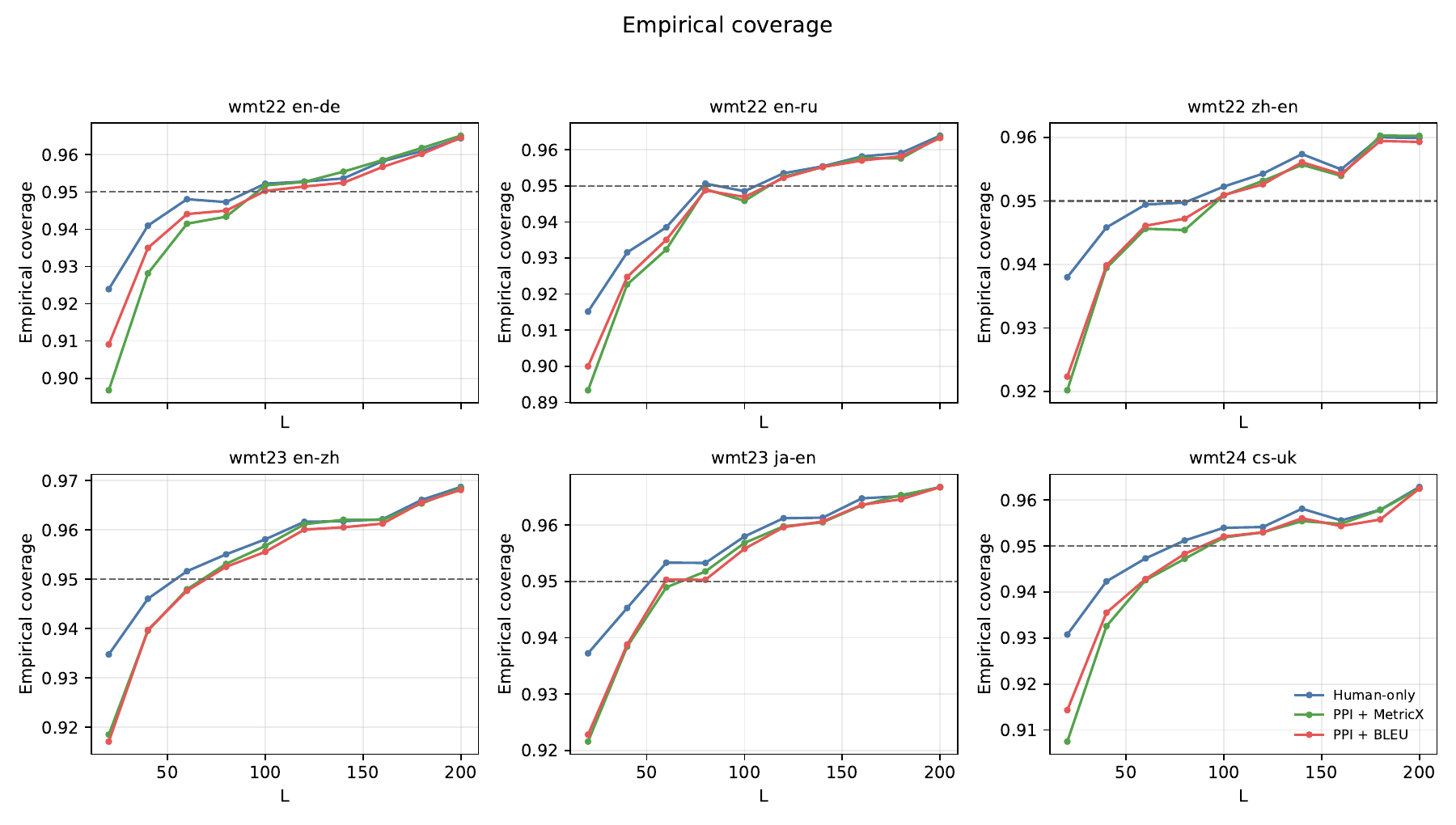}
\caption{Empirical coverage of the two-sided 95\% confidence intervals. The dashed line marks the nominal 95\% level.}
\label{fig:basic_interval_coverage}
\end{figure*}

\begin{figure*}[]
\centering
\includegraphics[width=\textwidth]{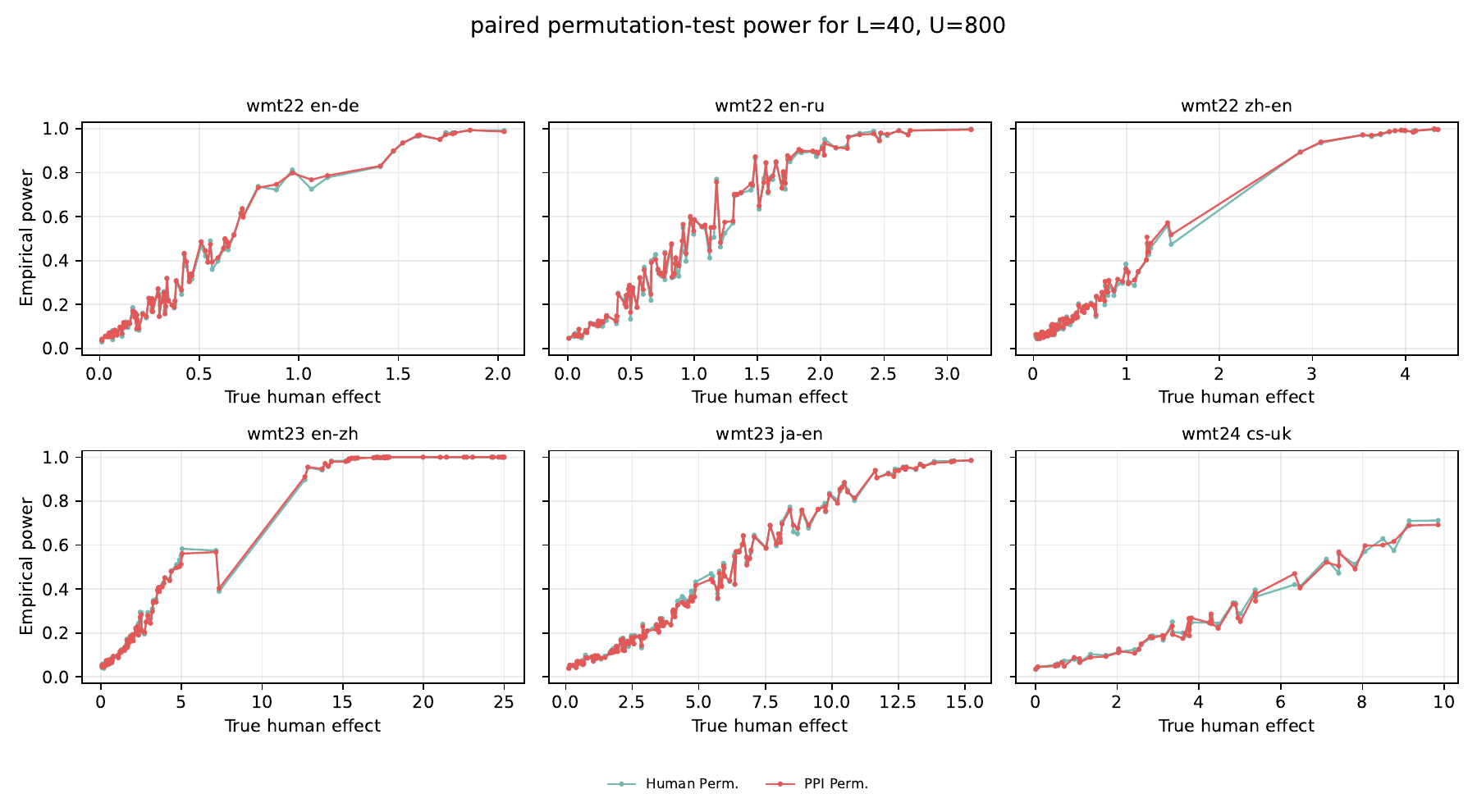}
\caption{Empirical power curves for $L=40$ and $U=800$. ``Human Perm'' means the human-only paired permutation test and ``PPI Perm'' means the prediction-powered paired permutation test. Each point is one ordered system pair, and the x-axis is the full-population human-score difference. MetricX is used as the automatic metric.}
\label{fig:perm_only_power_l40}
\end{figure*}

\begin{figure*}[]
\centering
\includegraphics[width=\textwidth]{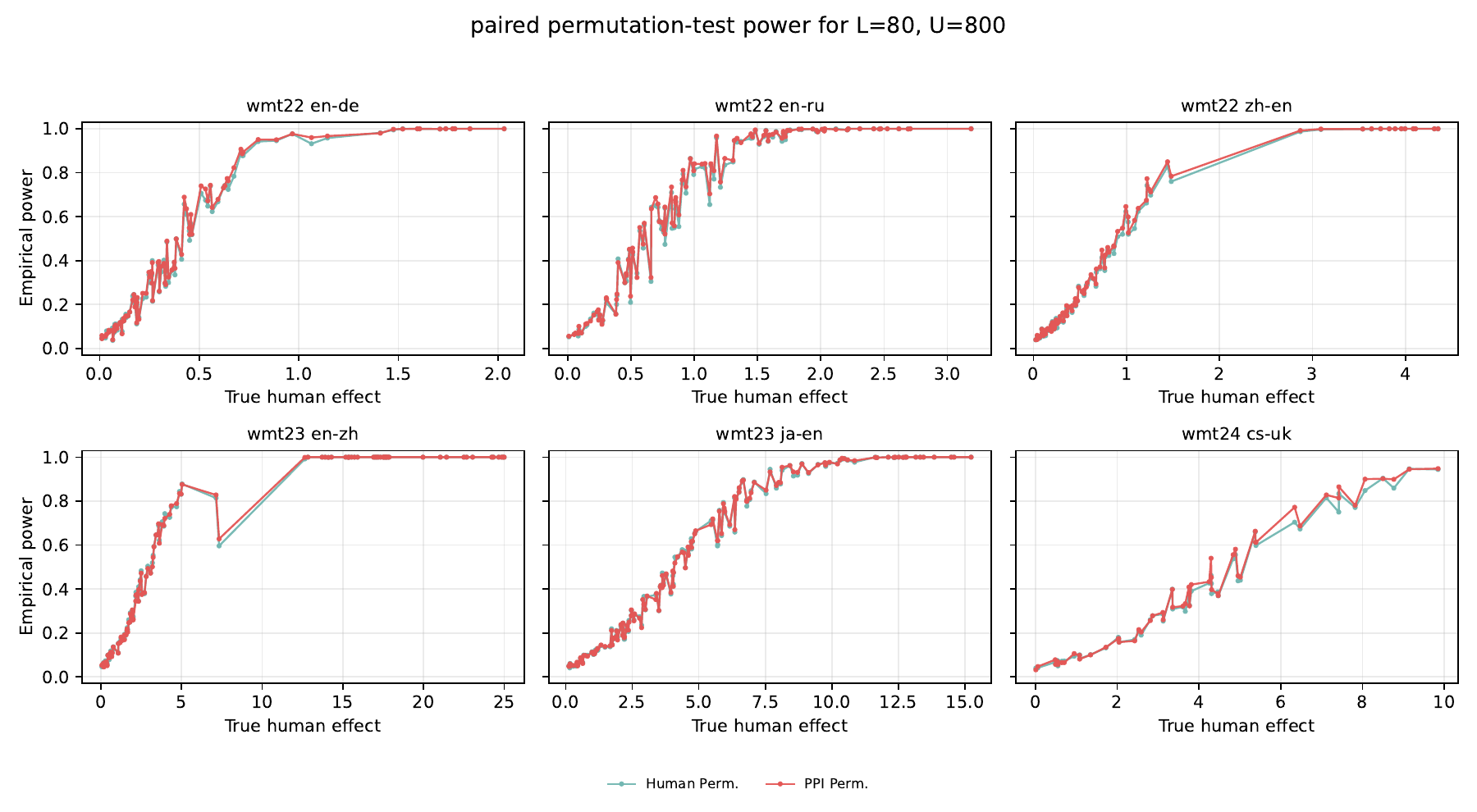}
\caption{Empirical power curves for $L=80$ and $U=800$. ``Human Perm'' means the human-only paired permutation test and ``PPI Perm'' means the prediction-powered paired permutation test. Each point is one ordered system pair, and the x-axis is the full-population human-score difference. MetricX is used as the automatic metric.}
\label{fig:perm_only_power_l80}
\end{figure*}

\begin{figure*}[]
\centering
\includegraphics[width=\textwidth]{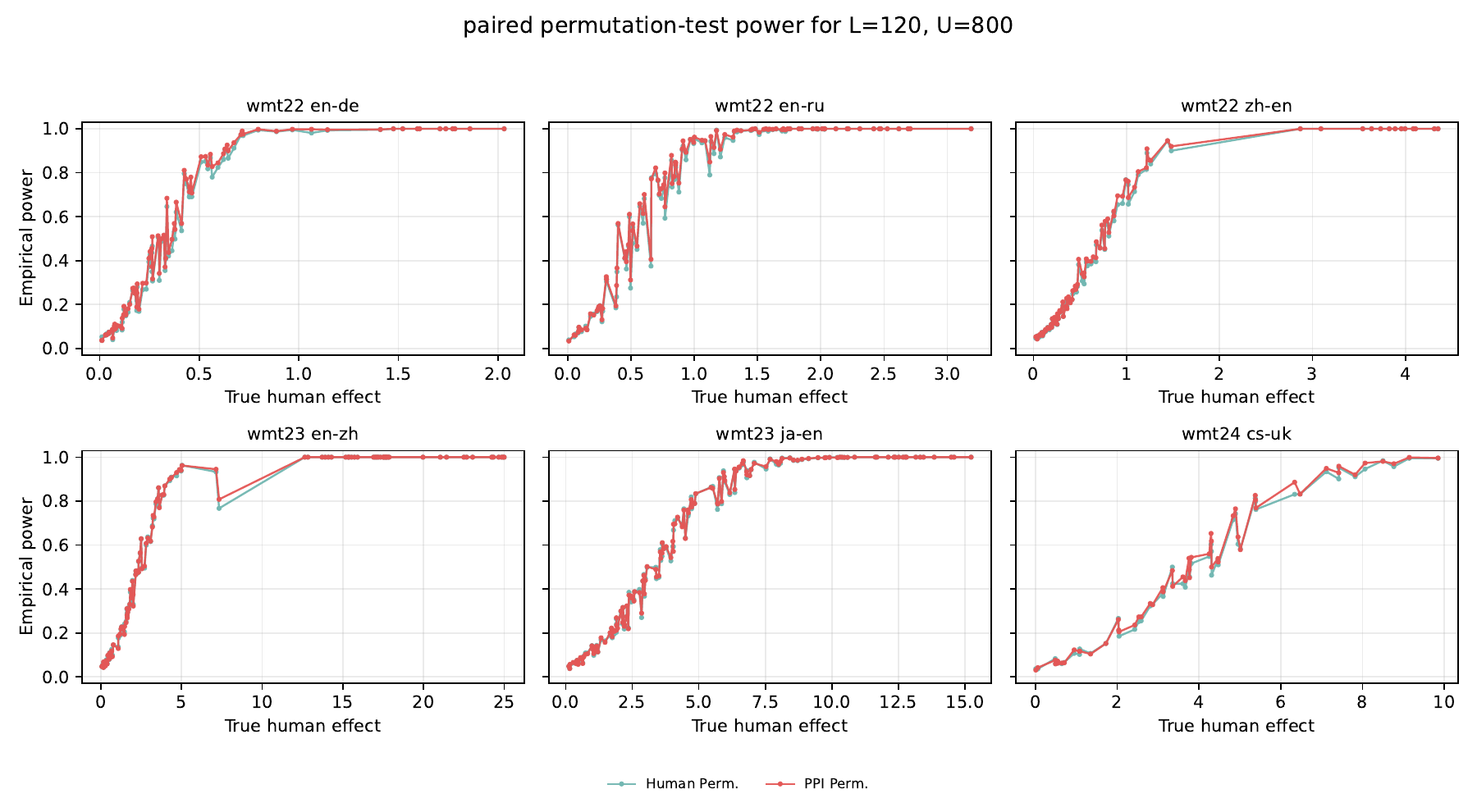}
\caption{Empirical power curves for $L=120$ and $U=800$. ``Human Perm'' means the human-only paired permutation test and ``PPI Perm'' means the prediction-powered paired permutation test. Each point is one ordered system pair, and the x-axis is the full-population human-score difference. MetricX is used as the automatic metric.}
\label{fig:perm_only_power_l120}
\end{figure*}

\begin{figure*}[]
\centering
\includegraphics[width=\textwidth]{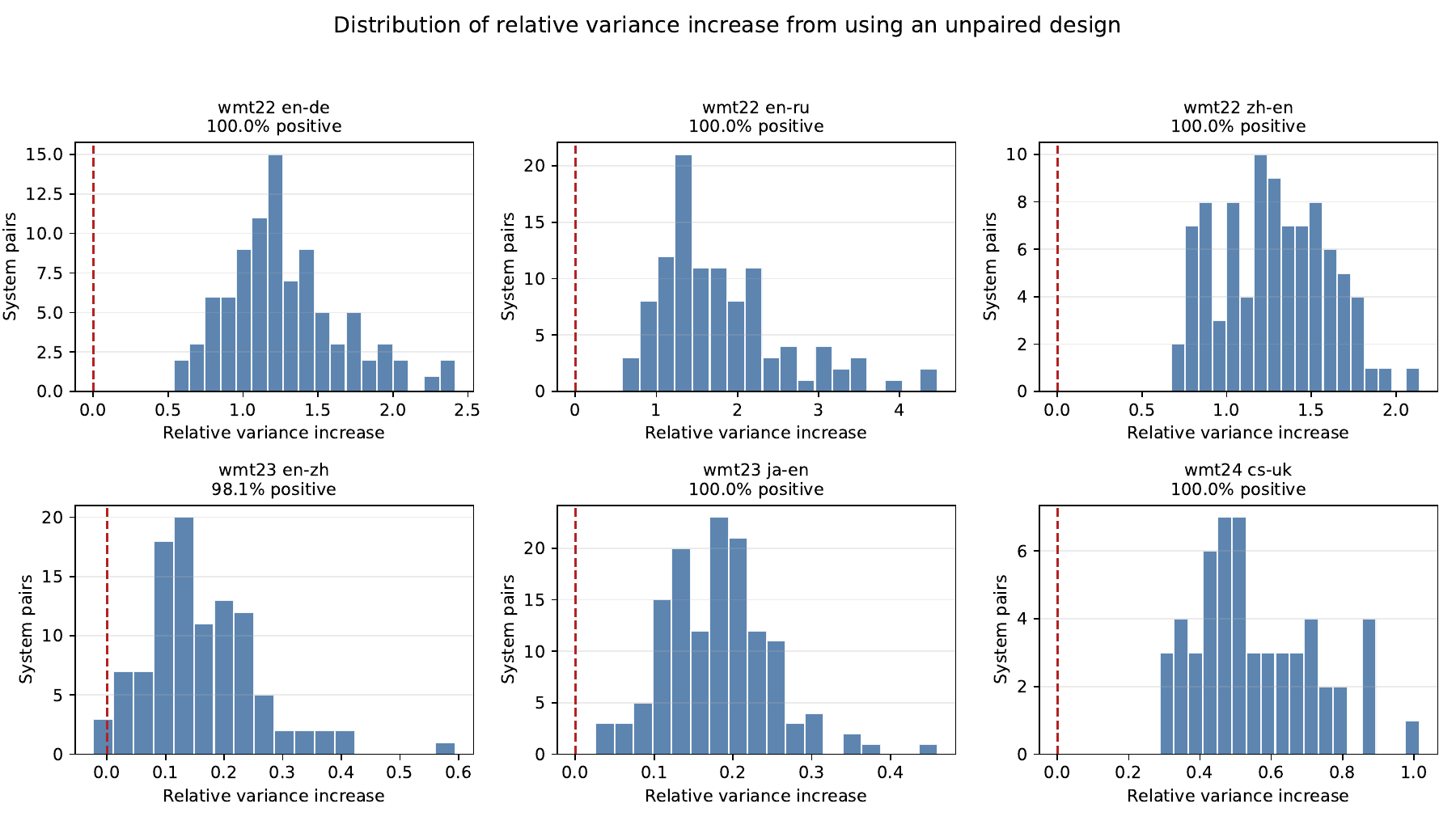}
\caption{Distribution of the relative variance increase from using an unpaired design instead of the paired design under the human-only cases. The dashed vertical line marks zero. Values to the right indicate larger variance under the unpaired design.}
\label{fig:paired_unpaired_human_variance}
\end{figure*}

\begin{figure*}[]
\centering
\includegraphics[width=\textwidth]{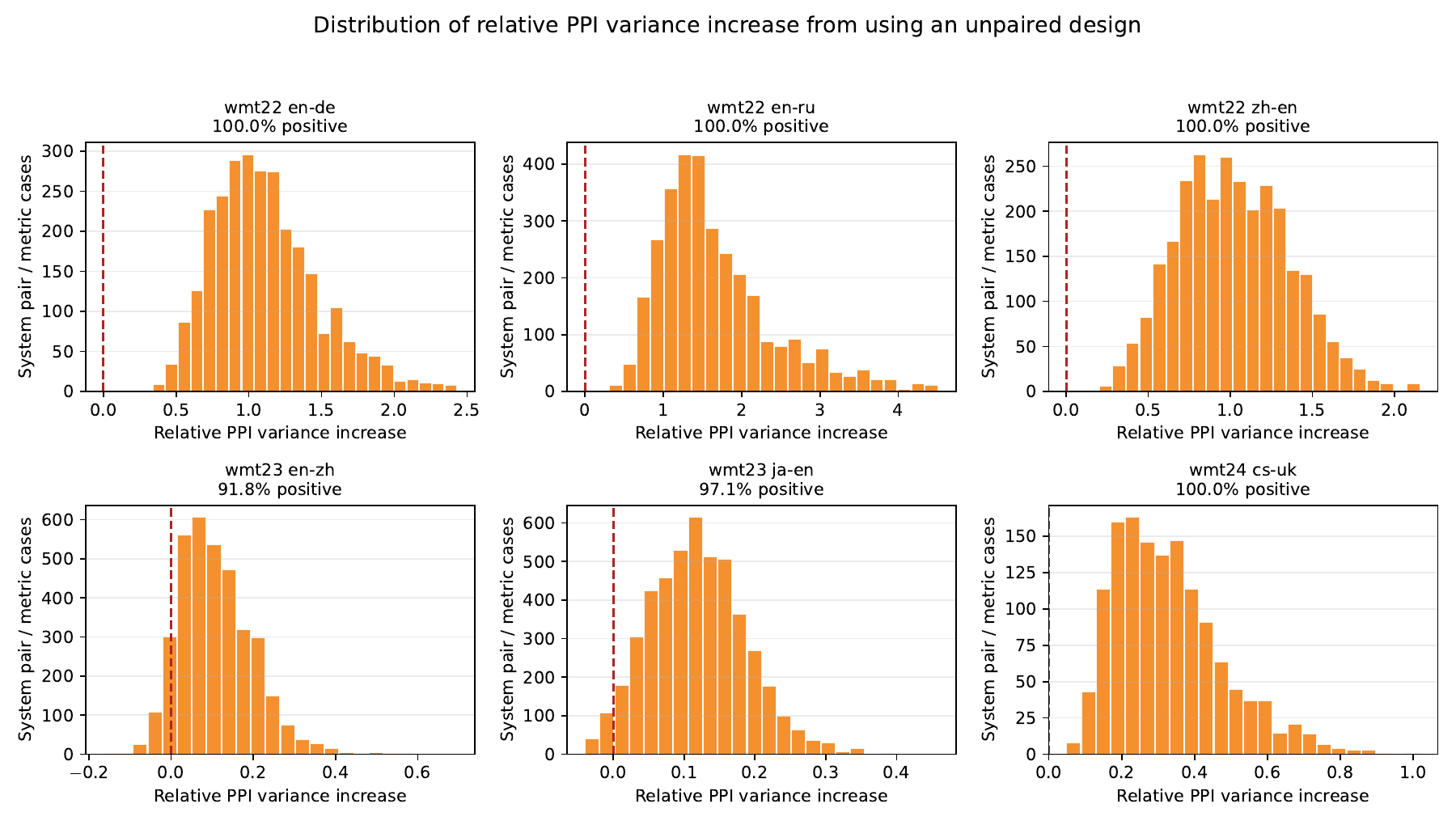}
\caption{Distribution of the relative variance increase from using an unpaired design instead of the paired design under the prediction-powered cases. Each observation is one system-pair/metric case. The dashed vertical line marks zero. Values to the right indicate larger variance under the unpaired design.}
\label{fig:paired_unpaired_pp_variance}
\end{figure*}

\begin{figure*}[]
\centering
\includegraphics[width=\textwidth]{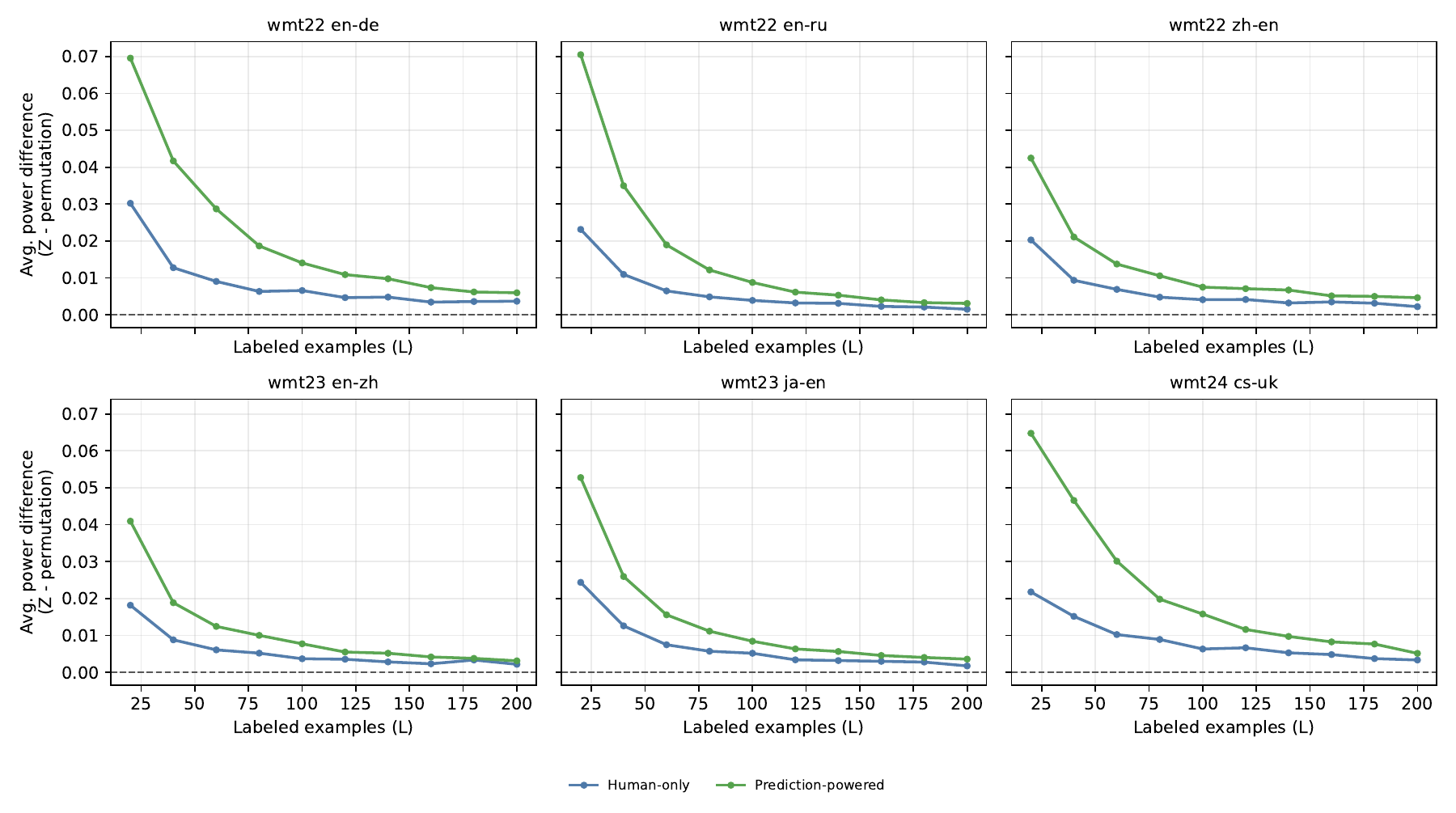}
\caption{Average power difference across all system pairs Average empirical power difference (paired Z-test minus paired permutation test). Positive values indicate higher empirical power for the paired Z-test.}
\label{fig:z_vs_perm_power_gain_by_l}
\end{figure*}

\begin{figure*}[]
\centering
\includegraphics[width=\textwidth]{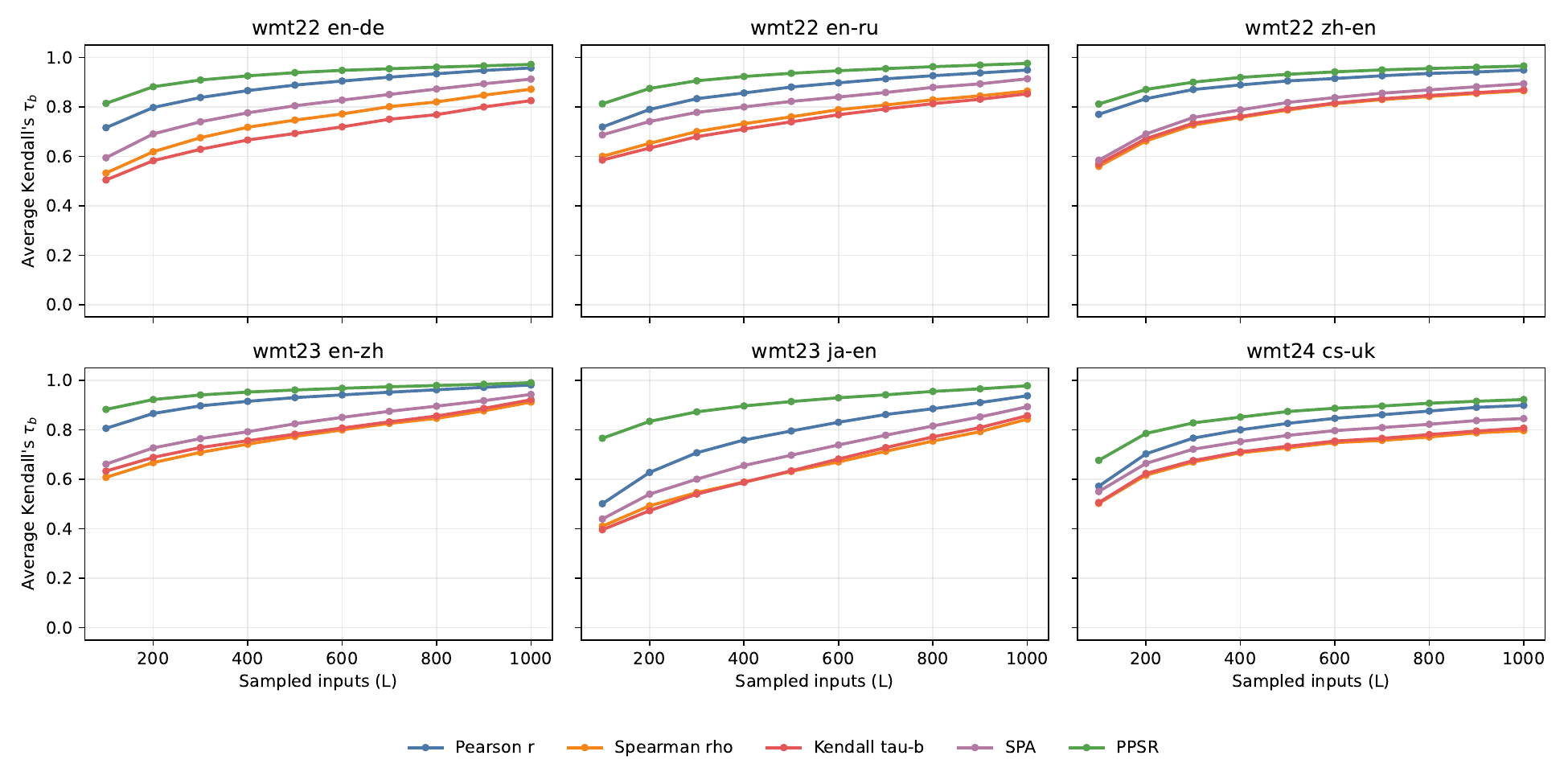}
\caption{Ranking stability of system-level meta-metrics. Higher values indicate more stable metric rankings. Note that for meta-evaluation, only labeled examples are used and the unlabeled examples are unused.}
\label{fig:saving_ratio_ranking_stability}
\end{figure*}

\begin{table*}[t]
\centering
\scriptsize
\setlength{\tabcolsep}{3pt}
\begin{adjustbox}{max width=\textwidth}
\begin{tabular}{lrrrrr}
\toprule
Metric & $r$ & $\rho$ & $\tau_b$ & SPA & PPSR \\
\midrule
metricx\_xl\_MQM\_2020-refA & 0.939 (\phantom{0}5) & 0.771 (\phantom{0}7) & 0.582 (\phantom{0}8) & 0.819 (\phantom{0}6) & 0.155 (\phantom{0}1) \\
metricx\_xxl\_MQM\_2020-refA & 0.953 (\phantom{0}4) & 0.837 (\phantom{0}2) & 0.626 (\phantom{0}2) & 0.831 (\phantom{0}3) & 0.155 (\phantom{0}2) \\
COMET-22-refA & 0.922 (\phantom{0}6) & 0.749 (11) & 0.604 (\phantom{0}6) & 0.810 (\phantom{0}7) & 0.140 (\phantom{0}3) \\
metricx\_xl\_DA\_2019-refA & 0.969 (\phantom{0}2) & 0.829 (\phantom{0}4) & 0.582 (\phantom{0}8) & 0.829 (\phantom{0}4) & 0.138 (\phantom{0}4) \\
metricx\_xxl\_DA\_2019-refA & 0.972 (\phantom{0}1) & 0.842 (\phantom{0}1) & 0.626 (\phantom{0}2) & 0.835 (\phantom{0}2) & 0.136 (\phantom{0}5) \\
UniTE-refA & 0.875 (11) & 0.763 (\phantom{0}8) & 0.538 (12) & 0.785 (11) & 0.116 (\phantom{0}6) \\
UniTE-ref-refA & 0.866 (14) & 0.763 (\phantom{0}8) & 0.538 (12) & 0.784 (12) & 0.112 (\phantom{0}7) \\
BLEURT-20-refA & 0.895 (\phantom{0}8) & 0.815 (\phantom{0}5) & 0.604 (\phantom{0}6) & 0.821 (\phantom{0}5) & 0.102 (\phantom{0}8) \\
COMETKiwi-src & 0.885 (10) & 0.705 (17) & 0.560 (10) & 0.785 (10) & 0.101 (\phantom{0}9) \\
COMET-20-refA & 0.964 (\phantom{0}3) & 0.833 (\phantom{0}3) & 0.670 (\phantom{0}1) & 0.848 (\phantom{0}1) & 0.091 (10) \\
Cross-QE-src & 0.895 (\phantom{0}9) & 0.710 (15) & 0.516 (16) & 0.775 (14) & 0.085 (11) \\
UniTE-src-src & 0.797 (21) & 0.631 (24) & 0.451 (25) & 0.737 (23) & 0.084 (12) \\
MS-COMET-22-refA & 0.875 (12) & 0.758 (10) & 0.626 (\phantom{0}2) & 0.804 (\phantom{0}8) & 0.068 (13) \\
COMET-QE-src & 0.676 (28) & 0.662 (22) & 0.495 (21) & 0.761 (18) & 0.066 (14) \\
SEScore-refA & 0.855 (15) & 0.675 (20) & 0.495 (21) & 0.766 (16) & 0.062 (15) \\
MATESE-refA & 0.872 (13) & 0.749 (11) & 0.516 (16) & 0.778 (13) & 0.051 (16) \\
MS-COMET-QE-22-src & 0.748 (24) & 0.675 (20) & 0.516 (16) & 0.740 (22) & 0.050 (17) \\
YiSi-1-refA & 0.897 (\phantom{0}7) & 0.815 (\phantom{0}5) & 0.626 (\phantom{0}2) & 0.800 (\phantom{0}9) & 0.038 (18) \\
MEE4-refA & 0.833 (17) & 0.719 (13) & 0.538 (12) & 0.770 (15) & 0.034 (19) \\
BERTScore-refA & 0.818 (19) & 0.679 (18) & 0.516 (16) & 0.757 (20) & 0.033 (20) \\
MATESE-QE-src & 0.752 (22) & 0.560 (28) & 0.341 (30) & 0.681 (30) & 0.032 (21) \\
MEE2-refA & 0.824 (18) & 0.719 (13) & 0.538 (12) & 0.758 (19) & 0.027 (22) \\
chrF-refA & 0.848 (16) & 0.710 (15) & 0.560 (10) & 0.763 (17) & 0.027 (23) \\
HWTSC-Teacher-Sim-src & 0.641 (29) & 0.565 (27) & 0.429 (27) & 0.702 (28) & 0.022 (24) \\
MEE-refA & 0.816 (20) & 0.679 (18) & 0.516 (16) & 0.743 (21) & 0.020 (25) \\
f200spBLEU-refA & 0.723 (26) & 0.618 (25) & 0.473 (24) & 0.717 (25) & 0.020 (26) \\
f101spBLEU-refA & 0.727 (25) & 0.662 (22) & 0.495 (21) & 0.721 (24) & 0.019 (27) \\
KG-BERTScore-src & 0.614 (30) & 0.481 (30) & 0.363 (29) & 0.694 (29) & 0.017 (28) \\
HWTSC-TLM-src & 0.749 (23) & 0.556 (29) & 0.429 (27) & 0.713 (26) & 0.014 (29) \\
BLEU-refA & 0.685 (27) & 0.609 (26) & 0.451 (25) & 0.708 (27) & 0.013 (30) \\
REUSE-src & -0.651 (31) & -0.490 (31) & -0.363 (31) & 0.327 (31) & 0.008 (31) \\
\bottomrule
\end{tabular}
\end{adjustbox}
\caption{Scores and ranks of automatic metrics under each system-level meta-metric for WMT22 en-de. Metrics are sorted by PPSR; tied scores receive the same rank.}
\label{tab:system_metric_score_ranks_wmt22_en_de}
\end{table*}

\begin{table*}[t]
\centering
\scriptsize
\setlength{\tabcolsep}{3pt}
\begin{adjustbox}{max width=\textwidth}
\begin{tabular}{lrrrrr}
\toprule
Metric & Group-by-Item $r$ & No-Grouping $r$ & Group-by-System $r$ & PDP & PPSR \\
\midrule
metricx\_xl\_MQM\_2020-refA & 0.461 (\phantom{0}2) & 0.570 (\phantom{0}2) & 0.540 (\phantom{0}2) & 0.436 (\phantom{0}2) & 0.155 (\phantom{0}1) \\
metricx\_xxl\_MQM\_2020-refA & 0.460 (\phantom{0}3) & 0.575 (\phantom{0}1) & 0.540 (\phantom{0}1) & 0.440 (\phantom{0}1) & 0.155 (\phantom{0}2) \\
COMET-22-refA & 0.446 (\phantom{0}5) & 0.534 (\phantom{0}3) & 0.498 (\phantom{0}3) & 0.420 (\phantom{0}3) & 0.140 (\phantom{0}3) \\
metricx\_xl\_DA\_2019-refA & 0.453 (\phantom{0}4) & 0.512 (\phantom{0}4) & 0.484 (\phantom{0}4) & 0.416 (\phantom{0}4) & 0.138 (\phantom{0}4) \\
metricx\_xxl\_DA\_2019-refA & 0.468 (\phantom{0}1) & 0.501 (\phantom{0}5) & 0.470 (\phantom{0}5) & 0.416 (\phantom{0}5) & 0.136 (\phantom{0}5) \\
UniTE-refA & 0.433 (\phantom{0}6) & 0.482 (\phantom{0}6) & 0.454 (\phantom{0}6) & 0.379 (\phantom{0}6) & 0.116 (\phantom{0}6) \\
UniTE-ref-refA & 0.419 (\phantom{0}7) & 0.473 (\phantom{0}8) & 0.443 (\phantom{0}8) & 0.372 (\phantom{0}7) & 0.112 (\phantom{0}7) \\
BLEURT-20-refA & 0.392 (\phantom{0}8) & 0.478 (\phantom{0}7) & 0.451 (\phantom{0}7) & 0.362 (\phantom{0}9) & 0.102 (\phantom{0}8) \\
COMETKiwi-src & 0.355 (10) & 0.461 (\phantom{0}9) & 0.418 (10) & 0.364 (\phantom{0}8) & 0.101 (\phantom{0}9) \\
COMET-20-refA & 0.372 (\phantom{0}9) & 0.430 (12) & 0.403 (11) & 0.342 (10) & 0.091 (10) \\
Cross-QE-src & 0.320 (14) & 0.433 (11) & 0.396 (12) & 0.333 (11) & 0.085 (11) \\
UniTE-src-src & 0.342 (12) & 0.425 (13) & 0.392 (13) & 0.323 (12) & 0.084 (12) \\
MS-COMET-22-refA & 0.327 (13) & 0.379 (15) & 0.343 (15) & 0.286 (13) & 0.068 (13) \\
COMET-QE-src & 0.265 (19) & 0.435 (10) & 0.420 (\phantom{0}9) & 0.276 (14) & 0.066 (14) \\
SEScore-refA & 0.318 (15) & 0.371 (16) & 0.337 (16) & 0.272 (15) & 0.062 (15) \\
MATESE-refA & 0.346 (11) & 0.412 (14) & 0.379 (14) & 0.250 (16) & 0.051 (16) \\
MS-COMET-QE-22-src & 0.215 (26) & 0.318 (17) & 0.280 (18) & 0.240 (17) & 0.050 (17) \\
YiSi-1-refA & 0.310 (16) & 0.293 (19) & 0.280 (19) & 0.220 (18) & 0.038 (18) \\
MEE4-refA & 0.269 (18) & 0.268 (22) & 0.259 (21) & 0.206 (19) & 0.034 (19) \\
BERTScore-refA & 0.263 (20) & 0.270 (21) & 0.259 (22) & 0.199 (20) & 0.033 (20) \\
MATESE-QE-src & 0.293 (17) & 0.317 (18) & 0.290 (17) & 0.193 (21) & 0.032 (21) \\
MEE2-refA & 0.254 (22) & 0.277 (20) & 0.267 (20) & 0.185 (23) & 0.027 (22) \\
chrF-refA & 0.256 (21) & 0.242 (23) & 0.231 (23) & 0.189 (22) & 0.027 (23) \\
HWTSC-Teacher-Sim-src & 0.172 (29) & 0.206 (27) & 0.189 (28) & 0.155 (25) & 0.022 (24) \\
MEE-refA & 0.224 (25) & 0.213 (25) & 0.201 (25) & 0.162 (24) & 0.020 (25) \\
f200spBLEU-refA & 0.229 (24) & 0.216 (24) & 0.206 (24) & 0.155 (26) & 0.020 (26) \\
f101spBLEU-refA & 0.232 (23) & 0.210 (26) & 0.200 (26) & 0.153 (27) & 0.019 (27) \\
KG-BERTScore-src & 0.158 (30) & 0.189 (29) & 0.183 (29) & 0.135 (28) & 0.017 (28) \\
HWTSC-TLM-src & 0.205 (27) & 0.091 (30) & 0.075 (30) & 0.132 (29) & 0.014 (29) \\
BLEU-refA & 0.203 (28) & 0.203 (28) & 0.195 (27) & 0.127 (30) & 0.013 (30) \\
REUSE-src & -0.104 (31) & 0.046 (31) & 0.063 (31) & -0.069 (31) & 0.008 (31) \\
\bottomrule
\end{tabular}
\end{adjustbox}
\caption{Scores and ranks of automatic metrics under each segment-level meta-metric for WMT22 en-de. Metrics are sorted by PPSR; tied scores receive the same rank.}
\label{tab:segment_metric_score_ranks_wmt22_en_de}
\end{table*}

\begin{table*}[t]
\centering
\scriptsize
\setlength{\tabcolsep}{3pt}
\begin{adjustbox}{max width=\textwidth}
\begin{tabular}{lrrrrr}
\toprule
Metric & $r$ & $\rho$ & $\tau_b$ & SPA & PPSR \\
\midrule
metricx\_xxl\_MQM\_2020-refA & 0.953 (\phantom{0}4) & 0.921 (\phantom{0}3) & 0.829 (\phantom{0}3) & 0.915 (\phantom{0}3) & 0.135 (\phantom{0}1) \\
metricx\_xl\_MQM\_2020-refA & 0.932 (\phantom{0}6) & 0.904 (\phantom{0}6) & 0.790 (\phantom{0}6) & 0.900 (\phantom{0}5) & 0.128 (\phantom{0}2) \\
COMET-22-refA & 0.908 (\phantom{0}7) & 0.875 (12) & 0.714 (12) & 0.880 (\phantom{0}7) & 0.102 (\phantom{0}3) \\
UniTE-refA & 0.894 (\phantom{0}8) & 0.886 (\phantom{0}7) & 0.752 (\phantom{0}7) & 0.878 (\phantom{0}8) & 0.096 (\phantom{0}4) \\
metricx\_xxl\_DA\_2019-refA & 0.979 (\phantom{0}1) & 0.943 (\phantom{0}2) & 0.848 (\phantom{0}1) & 0.933 (\phantom{0}1) & 0.094 (\phantom{0}5) \\
UniTE-ref-refA & 0.841 (11) & 0.854 (13) & 0.752 (\phantom{0}7) & 0.873 (\phantom{0}9) & 0.091 (\phantom{0}6) \\
metricx\_xl\_DA\_2019-refA & 0.975 (\phantom{0}2) & 0.946 (\phantom{0}1) & 0.848 (\phantom{0}1) & 0.923 (\phantom{0}2) & 0.087 (\phantom{0}7) \\
UniTE-src-src & 0.772 (22) & 0.789 (22) & 0.657 (19) & 0.825 (20) & 0.070 (\phantom{0}8) \\
BLEURT-20-refA & 0.956 (\phantom{0}3) & 0.921 (\phantom{0}3) & 0.829 (\phantom{0}3) & 0.903 (\phantom{0}4) & 0.069 (\phantom{0}9) \\
COMETKiwi-src & 0.763 (23) & 0.782 (23) & 0.619 (24) & 0.821 (23) & 0.065 (10) \\
MATESE-refA & 0.795 (21) & 0.832 (19) & 0.714 (12) & 0.847 (12) & 0.061 (11) \\
COMET-20-refA & 0.940 (\phantom{0}5) & 0.921 (\phantom{0}3) & 0.810 (\phantom{0}5) & 0.894 (\phantom{0}6) & 0.061 (12) \\
Cross-QE-src & 0.802 (20) & 0.814 (21) & 0.657 (19) & 0.826 (19) & 0.060 (13) \\
MS-COMET-22-refA & 0.832 (13) & 0.843 (16) & 0.695 (17) & 0.847 (11) & 0.057 (14) \\
COMET-QE-src & 0.483 (29) & 0.632 (27) & 0.600 (25) & 0.799 (25) & 0.055 (15) \\
MS-COMET-QE-22-src & 0.718 (24) & 0.779 (24) & 0.638 (23) & 0.812 (24) & 0.045 (16) \\
MATESE-QE-src & 0.665 (27) & 0.582 (29) & 0.467 (27) & 0.768 (26) & 0.040 (17) \\
YiSi-1-refA & 0.882 (\phantom{0}9) & 0.879 (11) & 0.714 (12) & 0.847 (10) & 0.033 (18) \\
BERTScore-refA & 0.817 (17) & 0.882 (\phantom{0}8) & 0.733 (\phantom{0}9) & 0.833 (18) & 0.025 (19) \\
MEE4-refA & 0.809 (19) & 0.882 (\phantom{0}8) & 0.733 (\phantom{0}9) & 0.843 (13) & 0.022 (20) \\
chrF-refA & 0.874 (10) & 0.836 (18) & 0.657 (19) & 0.837 (16) & 0.020 (21) \\
MEE2-refA & 0.825 (15) & 0.882 (\phantom{0}8) & 0.733 (\phantom{0}9) & 0.838 (14) & 0.019 (22) \\
f200spBLEU-refA & 0.827 (14) & 0.854 (13) & 0.714 (12) & 0.838 (15) & 0.018 (23) \\
f101spBLEU-refA & 0.821 (16) & 0.854 (13) & 0.714 (12) & 0.833 (17) & 0.018 (24) \\
MEE-refA & 0.835 (12) & 0.839 (17) & 0.657 (19) & 0.823 (22) & 0.015 (25) \\
BLEU-refA & 0.813 (18) & 0.821 (20) & 0.676 (18) & 0.824 (21) & 0.014 (26) \\
KG-BERTScore-src & 0.716 (25) & 0.661 (26) & 0.467 (27) & 0.759 (27) & 0.013 (27) \\
HWTSC-Teacher-Sim-src & 0.703 (26) & 0.718 (25) & 0.543 (26) & 0.747 (28) & 0.013 (28) \\
HWTSC-TLM-src & 0.642 (28) & 0.632 (27) & 0.467 (27) & 0.729 (29) & 0.011 (29) \\
REUSE-src & -0.381 (30) & -0.493 (30) & -0.429 (30) & 0.286 (30) & 0.005 (30) \\
\bottomrule
\end{tabular}
\end{adjustbox}
\caption{Scores and ranks of automatic metrics under each system-level meta-metric for WMT22 en-ru. Metrics are sorted by PPSR; tied scores receive the same rank.}
\label{tab:system_metric_score_ranks_wmt22_en_ru}
\end{table*}

\begin{table*}[t]
\centering
\scriptsize
\setlength{\tabcolsep}{3pt}
\begin{adjustbox}{max width=\textwidth}
\begin{tabular}{lrrrrr}
\toprule
Metric & Group-by-Item $r$ & No-Grouping $r$ & Group-by-System $r$ & PDP & PPSR \\
\midrule
metricx\_xxl\_MQM\_2020-refA & 0.461 (\phantom{0}2) & 0.495 (\phantom{0}1) & 0.455 (\phantom{0}1) & 0.400 (\phantom{0}1) & 0.135 (\phantom{0}1) \\
metricx\_xl\_MQM\_2020-refA & 0.431 (\phantom{0}6) & 0.492 (\phantom{0}2) & 0.452 (\phantom{0}2) & 0.391 (\phantom{0}2) & 0.128 (\phantom{0}2) \\
COMET-22-refA & 0.442 (\phantom{0}5) & 0.469 (\phantom{0}3) & 0.442 (\phantom{0}3) & 0.344 (\phantom{0}3) & 0.102 (\phantom{0}3) \\
UniTE-refA & 0.448 (\phantom{0}4) & 0.414 (\phantom{0}8) & 0.389 (\phantom{0}6) & 0.335 (\phantom{0}4) & 0.096 (\phantom{0}4) \\
metricx\_xxl\_DA\_2019-refA & 0.471 (\phantom{0}1) & 0.416 (\phantom{0}6) & 0.385 (\phantom{0}8) & 0.333 (\phantom{0}5) & 0.094 (\phantom{0}5) \\
UniTE-ref-refA & 0.425 (\phantom{0}7) & 0.415 (\phantom{0}7) & 0.392 (\phantom{0}5) & 0.323 (\phantom{0}7) & 0.091 (\phantom{0}6) \\
metricx\_xl\_DA\_2019-refA & 0.454 (\phantom{0}3) & 0.417 (\phantom{0}5) & 0.385 (\phantom{0}7) & 0.324 (\phantom{0}6) & 0.087 (\phantom{0}7) \\
UniTE-src-src & 0.367 (10) & 0.392 (11) & 0.368 (10) & 0.283 (\phantom{0}9) & 0.070 (\phantom{0}8) \\
BLEURT-20-refA & 0.396 (\phantom{0}8) & 0.397 (\phantom{0}9) & 0.372 (\phantom{0}9) & 0.286 (\phantom{0}8) & 0.069 (\phantom{0}9) \\
COMETKiwi-src & 0.358 (11) & 0.385 (12) & 0.359 (12) & 0.272 (10) & 0.065 (10) \\
MATESE-refA & 0.329 (15) & 0.345 (14) & 0.301 (16) & 0.260 (13) & 0.061 (11) \\
COMET-20-refA & 0.384 (\phantom{0}9) & 0.343 (15) & 0.321 (14) & 0.263 (11) & 0.061 (12) \\
Cross-QE-src & 0.357 (12) & 0.397 (10) & 0.364 (11) & 0.260 (12) & 0.060 (13) \\
MS-COMET-22-refA & 0.346 (13) & 0.356 (13) & 0.334 (13) & 0.249 (14) & 0.057 (14) \\
COMET-QE-src & 0.245 (23) & 0.439 (\phantom{0}4) & 0.429 (\phantom{0}4) & 0.227 (15) & 0.055 (15) \\
MS-COMET-QE-22-src & 0.265 (20) & 0.343 (16) & 0.313 (15) & 0.211 (16) & 0.045 (16) \\
MATESE-QE-src & 0.267 (19) & 0.306 (17) & 0.265 (17) & 0.204 (17) & 0.040 (17) \\
YiSi-1-refA & 0.336 (14) & 0.234 (18) & 0.214 (18) & 0.199 (18) & 0.033 (18) \\
BERTScore-refA & 0.274 (18) & 0.197 (19) & 0.180 (20) & 0.171 (19) & 0.025 (19) \\
MEE4-refA & 0.301 (16) & 0.188 (21) & 0.175 (21) & 0.163 (20) & 0.022 (20) \\
chrF-refA & 0.264 (21) & 0.168 (23) & 0.150 (23) & 0.157 (21) & 0.020 (21) \\
MEE2-refA & 0.278 (17) & 0.194 (20) & 0.180 (19) & 0.150 (22) & 0.019 (22) \\
f200spBLEU-refA & 0.246 (22) & 0.172 (22) & 0.155 (22) & 0.146 (23) & 0.018 (23) \\
f101spBLEU-refA & 0.238 (25) & 0.157 (25) & 0.140 (25) & 0.145 (24) & 0.018 (24) \\
MEE-refA & 0.239 (24) & 0.126 (26) & 0.110 (27) & 0.135 (25) & 0.015 (25) \\
BLEU-refA & 0.201 (28) & 0.160 (24) & 0.148 (24) & 0.124 (26) & 0.014 (26) \\
KG-BERTScore-src & 0.207 (27) & 0.112 (27) & 0.111 (26) & 0.119 (28) & 0.013 (27) \\
HWTSC-Teacher-Sim-src & 0.208 (26) & 0.111 (28) & 0.106 (28) & 0.120 (27) & 0.013 (28) \\
HWTSC-TLM-src & 0.201 (29) & 0.078 (29) & 0.069 (29) & 0.103 (29) & 0.011 (29) \\
REUSE-src & -0.037 (30) & 0.052 (30) & 0.065 (30) & -0.025 (30) & 0.005 (30) \\
\bottomrule
\end{tabular}
\end{adjustbox}
\caption{Scores and ranks of automatic metrics under each segment-level meta-metric for WMT22 en-ru. Metrics are sorted by PPSR; tied scores receive the same rank.}
\label{tab:segment_metric_score_ranks_wmt22_en_ru}
\end{table*}

\begin{table*}[t]
\centering
\scriptsize
\setlength{\tabcolsep}{3pt}
\begin{adjustbox}{max width=\textwidth}
\begin{tabular}{lrrrrr}
\toprule
Metric & $r$ & $\rho$ & $\tau_b$ & SPA & PPSR \\
\midrule
COMET-22-refA & 0.942 (\phantom{0}5) & 0.859 (\phantom{0}4) & 0.736 (\phantom{0}3) & 0.876 (\phantom{0}3) & 0.093 (\phantom{0}1) \\
metricx\_xl\_MQM\_2020-refA & 0.914 (10) & 0.877 (\phantom{0}3) & 0.736 (\phantom{0}3) & 0.873 (\phantom{0}4) & 0.091 (\phantom{0}2) \\
metricx\_xxl\_MQM\_2020-refA & 0.920 (\phantom{0}9) & 0.842 (\phantom{0}7) & 0.692 (\phantom{0}7) & 0.862 (\phantom{0}5) & 0.087 (\phantom{0}3) \\
metricx\_xl\_DA\_2019-refA & 0.982 (\phantom{0}2) & 0.895 (\phantom{0}2) & 0.758 (\phantom{0}1) & 0.895 (\phantom{0}1) & 0.070 (\phantom{0}4) \\
UniTE-refA & 0.914 (11) & 0.837 (\phantom{0}8) & 0.692 (\phantom{0}7) & 0.849 (10) & 0.069 (\phantom{0}5) \\
COMETKiwi-src & 0.866 (19) & 0.785 (12) & 0.648 (12) & 0.828 (12) & 0.067 (\phantom{0}6) \\
UniTE-ref-refA & 0.892 (15) & 0.820 (11) & 0.670 (11) & 0.839 (11) & 0.065 (\phantom{0}7) \\
metricx\_xxl\_DA\_2019-refA & 0.984 (\phantom{0}1) & 0.908 (\phantom{0}1) & 0.758 (\phantom{0}1) & 0.894 (\phantom{0}2) & 0.064 (\phantom{0}8) \\
Cross-QE-src & 0.870 (18) & 0.736 (16) & 0.582 (16) & 0.799 (16) & 0.057 (\phantom{0}9) \\
MATESE-refA & 0.856 (20) & 0.855 (\phantom{0}5) & 0.714 (\phantom{0}6) & 0.858 (\phantom{0}6) & 0.055 (10) \\
UniTE-src-src & 0.874 (16) & 0.635 (17) & 0.516 (17) & 0.775 (17) & 0.053 (11) \\
BLEURT-20-refA & 0.938 (\phantom{0}6) & 0.837 (\phantom{0}8) & 0.692 (\phantom{0}7) & 0.853 (\phantom{0}7) & 0.051 (12) \\
SEScore-refA & 0.944 (\phantom{0}4) & 0.754 (13) & 0.604 (15) & 0.815 (15) & 0.045 (13) \\
MATESE-QE-src & 0.767 (23) & 0.837 (\phantom{0}8) & 0.692 (\phantom{0}7) & 0.850 (\phantom{0}9) & 0.039 (14) \\
COMET-QE-src & 0.569 (27) & 0.749 (15) & 0.626 (13) & 0.816 (14) & 0.039 (15) \\
COMET-20-refA & 0.970 (\phantom{0}3) & 0.754 (13) & 0.626 (13) & 0.827 (13) & 0.037 (16) \\
BERTScore-refA & 0.924 (\phantom{0}8) & 0.578 (20) & 0.451 (19) & 0.747 (19) & 0.032 (17) \\
YiSi-1-refA & 0.935 (\phantom{0}7) & 0.591 (19) & 0.451 (19) & 0.737 (20) & 0.030 (18) \\
MS-COMET-22-refA & 0.909 (12) & 0.846 (\phantom{0}6) & 0.736 (\phantom{0}3) & 0.852 (\phantom{0}8) & 0.024 (19) \\
MS-COMET-QE-22-src & 0.897 (14) & 0.604 (18) & 0.473 (18) & 0.769 (18) & 0.022 (20) \\
MEE4-refA & 0.905 (13) & 0.569 (21) & 0.451 (19) & 0.723 (21) & 0.019 (21) \\
MEE2-refA & 0.872 (17) & 0.446 (25) & 0.363 (23) & 0.684 (23) & 0.014 (22) \\
HWTSC-Teacher-Sim-src & 0.356 (30) & 0.499 (22) & 0.363 (23) & 0.678 (24) & 0.011 (23) \\
chrF-refA & 0.812 (22) & 0.429 (28) & 0.297 (28) & 0.660 (28) & 0.010 (24) \\
KG-BERTScore-src & 0.553 (28) & 0.486 (23) & 0.385 (22) & 0.696 (22) & 0.009 (25) \\
MEE-refA & 0.824 (21) & 0.349 (29) & 0.297 (28) & 0.650 (29) & 0.008 (26) \\
f200spBLEU-refA & 0.728 (24) & 0.437 (26) & 0.341 (26) & 0.668 (26) & 0.007 (27) \\
f101spBLEU-refA & 0.718 (25) & 0.437 (26) & 0.341 (26) & 0.665 (27) & 0.007 (28) \\
HWTSC-TLM-src & 0.460 (29) & 0.455 (24) & 0.363 (23) & 0.669 (25) & 0.006 (29) \\
BLEU-refA & 0.658 (26) & 0.345 (30) & 0.275 (30) & 0.646 (30) & 0.005 (30) \\
REUSE-src & -0.142 (31) & -0.332 (31) & -0.231 (31) & 0.386 (31) & 0.001 (31) \\
\bottomrule
\end{tabular}
\end{adjustbox}
\caption{Scores and ranks of automatic metrics under each system-level meta-metric for WMT22 zh-en. Metrics are sorted by PPSR; tied scores receive the same rank.}
\label{tab:system_metric_score_ranks_wmt22_zh_en}
\end{table*}

\begin{table*}[t]
\centering
\scriptsize
\setlength{\tabcolsep}{3pt}
\begin{adjustbox}{max width=\textwidth}
\begin{tabular}{lrrrrr}
\toprule
Metric & Group-by-Item $r$ & No-Grouping $r$ & Group-by-System $r$ & PDP & PPSR \\
\midrule
COMET-22-refA & 0.397 (\phantom{0}3) & 0.585 (\phantom{0}1) & 0.564 (\phantom{0}2) & 0.360 (\phantom{0}1) & 0.093 (\phantom{0}1) \\
metricx\_xl\_MQM\_2020-refA & 0.382 (\phantom{0}5) & 0.575 (\phantom{0}3) & 0.556 (\phantom{0}3) & 0.348 (\phantom{0}2) & 0.091 (\phantom{0}2) \\
metricx\_xxl\_MQM\_2020-refA & 0.381 (\phantom{0}6) & 0.581 (\phantom{0}2) & 0.566 (\phantom{0}1) & 0.342 (\phantom{0}3) & 0.087 (\phantom{0}3) \\
metricx\_xl\_DA\_2019-refA & 0.412 (\phantom{0}1) & 0.468 (\phantom{0}8) & 0.442 (\phantom{0}9) & 0.319 (\phantom{0}4) & 0.070 (\phantom{0}4) \\
UniTE-refA & 0.383 (\phantom{0}4) & 0.405 (14) & 0.385 (14) & 0.307 (\phantom{0}6) & 0.069 (\phantom{0}5) \\
COMETKiwi-src & 0.325 (11) & 0.509 (\phantom{0}6) & 0.488 (\phantom{0}7) & 0.305 (\phantom{0}7) & 0.067 (\phantom{0}6) \\
UniTE-ref-refA & 0.372 (\phantom{0}7) & 0.410 (13) & 0.391 (12) & 0.295 (\phantom{0}8) & 0.065 (\phantom{0}7) \\
metricx\_xxl\_DA\_2019-refA & 0.402 (\phantom{0}2) & 0.464 (10) & 0.439 (10) & 0.309 (\phantom{0}5) & 0.064 (\phantom{0}8) \\
Cross-QE-src & 0.318 (13) & 0.546 (\phantom{0}4) & 0.529 (\phantom{0}4) & 0.282 (\phantom{0}9) & 0.057 (\phantom{0}9) \\
MATESE-refA & 0.269 (18) & 0.528 (\phantom{0}5) & 0.507 (\phantom{0}5) & 0.270 (12) & 0.055 (10) \\
UniTE-src-src & 0.335 (\phantom{0}9) & 0.404 (15) & 0.382 (15) & 0.273 (11) & 0.053 (11) \\
BLEURT-20-refA & 0.364 (\phantom{0}8) & 0.430 (11) & 0.408 (11) & 0.275 (10) & 0.051 (12) \\
SEScore-refA & 0.311 (15) & 0.422 (12) & 0.389 (13) & 0.261 (13) & 0.045 (13) \\
MATESE-QE-src & 0.223 (22) & 0.468 (\phantom{0}9) & 0.444 (\phantom{0}8) & 0.222 (15) & 0.039 (14) \\
COMET-QE-src & 0.236 (21) & 0.505 (\phantom{0}7) & 0.501 (\phantom{0}6) & 0.215 (18) & 0.039 (15) \\
COMET-20-refA & 0.328 (10) & 0.386 (16) & 0.359 (16) & 0.241 (14) & 0.037 (16) \\
BERTScore-refA & 0.311 (14) & 0.376 (17) & 0.356 (17) & 0.217 (17) & 0.032 (17) \\
YiSi-1-refA & 0.319 (12) & 0.351 (19) & 0.328 (20) & 0.218 (16) & 0.030 (18) \\
MS-COMET-22-refA & 0.294 (16) & 0.369 (18) & 0.355 (18) & 0.188 (19) & 0.024 (19) \\
MS-COMET-QE-22-src & 0.246 (20) & 0.321 (21) & 0.302 (21) & 0.172 (21) & 0.022 (20) \\
MEE4-refA & 0.276 (17) & 0.192 (24) & 0.172 (24) & 0.178 (20) & 0.019 (21) \\
MEE2-refA & 0.253 (19) & 0.208 (23) & 0.191 (23) & 0.156 (22) & 0.014 (22) \\
HWTSC-Teacher-Sim-src & 0.149 (30) & 0.350 (20) & 0.349 (19) & 0.106 (27) & 0.011 (23) \\
chrF-refA & 0.222 (23) & 0.154 (28) & 0.138 (29) & 0.139 (23) & 0.010 (24) \\
KG-BERTScore-src & 0.159 (28) & 0.303 (22) & 0.302 (22) & 0.105 (28) & 0.009 (25) \\
MEE-refA & 0.195 (24) & 0.139 (29) & 0.122 (30) & 0.124 (24) & 0.008 (26) \\
f200spBLEU-refA & 0.176 (26) & 0.172 (27) & 0.158 (27) & 0.112 (25) & 0.007 (27) \\
f101spBLEU-refA & 0.177 (25) & 0.177 (25) & 0.163 (26) & 0.112 (26) & 0.007 (28) \\
HWTSC-TLM-src & 0.165 (27) & 0.010 (31) & 0.002 (31) & 0.081 (30) & 0.006 (29) \\
BLEU-refA & 0.159 (29) & 0.175 (26) & 0.164 (25) & 0.099 (29) & 0.005 (30) \\
REUSE-src & 0.006 (31) & 0.131 (30) & 0.140 (28) & 0.008 (31) & 0.001 (31) \\
\bottomrule
\end{tabular}
\end{adjustbox}
\caption{Scores and ranks of automatic metrics under each segment-level meta-metric for WMT22 zh-en. Metrics are sorted by PPSR; tied scores receive the same rank.}
\label{tab:segment_metric_score_ranks_wmt22_zh_en}
\end{table*}

\begin{table*}[t]
\centering
\scriptsize
\setlength{\tabcolsep}{3pt}
\begin{adjustbox}{max width=\textwidth}
\begin{tabular}{lrrrrr}
\toprule
Metric & $r$ & $\rho$ & $\tau_b$ & SPA & PPSR \\
\midrule
CometKiwi-src & 0.992 (\phantom{0}5) & 0.950 (\phantom{0}4) & 0.867 (\phantom{0}2) & 0.925 (\phantom{0}4) & 0.186 (\phantom{0}1) \\
KG-BERTScore-src & 0.992 (\phantom{0}4) & 0.950 (\phantom{0}4) & 0.867 (\phantom{0}2) & 0.925 (\phantom{0}5) & 0.186 (\phantom{0}2) \\
CometKiwi-XL-src & 0.985 (10) & 0.921 (10) & 0.810 (\phantom{0}9) & 0.916 (11) & 0.160 (\phantom{0}3) \\
MS-COMET-QE-22-src & 0.993 (\phantom{0}3) & 0.950 (\phantom{0}4) & 0.848 (\phantom{0}5) & 0.930 (\phantom{0}2) & 0.157 (\phantom{0}4) \\
COMET-refA & 0.995 (\phantom{0}2) & 0.954 (\phantom{0}2) & 0.867 (\phantom{0}2) & 0.920 (\phantom{0}7) & 0.149 (\phantom{0}5) \\
cometoid22-wmt23-src & 0.996 (\phantom{0}1) & 0.971 (\phantom{0}1) & 0.886 (\phantom{0}1) & 0.946 (\phantom{0}1) & 0.148 (\phantom{0}6) \\
CometKiwi-XXL-src & 0.980 (12) & 0.925 (\phantom{0}9) & 0.829 (\phantom{0}7) & 0.917 (10) & 0.145 (\phantom{0}7) \\
BLEURT-20-refA & 0.986 (\phantom{0}9) & 0.875 (22) & 0.771 (14) & 0.869 (22) & 0.140 (\phantom{0}8) \\
cometoid22-wmt22-src & 0.991 (\phantom{0}6) & 0.921 (10) & 0.810 (\phantom{0}9) & 0.922 (\phantom{0}6) & 0.126 (\phantom{0}9) \\
XLsim-refA & 0.981 (11) & 0.807 (23) & 0.695 (23) & 0.826 (23) & 0.124 (10) \\
YiSi-1-refA & 0.972 (14) & 0.696 (25) & 0.562 (25) & 0.773 (25) & 0.123 (11) \\
cometoid22-wmt21-src & 0.990 (\phantom{0}7) & 0.932 (\phantom{0}8) & 0.810 (\phantom{0}9) & 0.919 (\phantom{0}9) & 0.121 (12) \\
prismRef-refA & 0.970 (15) & 0.725 (24) & 0.619 (24) & 0.795 (24) & 0.118 (13) \\
MetricX-23-c-refA & 0.936 (22) & 0.954 (\phantom{0}2) & 0.829 (\phantom{0}7) & 0.927 (\phantom{0}3) & 0.115 (14) \\
BERTscore-refA & 0.972 (13) & 0.586 (29) & 0.486 (28) & 0.744 (27) & 0.113 (15) \\
XCOMET-Ensemble-refA & 0.944 (20) & 0.911 (17) & 0.771 (14) & 0.899 (18) & 0.110 (16) \\
GEMBA-MQM-src & 0.967 (16) & 0.946 (\phantom{0}7) & 0.848 (\phantom{0}5) & 0.919 (\phantom{0}8) & 0.090 (17) \\
XCOMET-QE-Ensemble-src & 0.908 (29) & 0.918 (12) & 0.790 (12) & 0.905 (14) & 0.085 (18) \\
mre-score-labse-regular-refA & 0.989 (\phantom{0}8) & 0.896 (20) & 0.733 (21) & 0.877 (20) & 0.085 (19) \\
MetricX-23-QE-b-src & 0.967 (17) & 0.904 (18) & 0.771 (14) & 0.910 (12) & 0.076 (20) \\
MetricX-23-QE-c-src & 0.909 (28) & 0.914 (13) & 0.771 (14) & 0.900 (17) & 0.075 (21) \\
prismSrc-src & 0.931 (23) & 0.168 (34) & 0.105 (34) & 0.549 (34) & 0.074 (22) \\
MetricX-23-QE-src & 0.950 (18) & 0.914 (13) & 0.790 (12) & 0.904 (15) & 0.066 (23) \\
MetricX-23-b-refA & 0.930 (24) & 0.914 (13) & 0.771 (14) & 0.900 (16) & 0.065 (24) \\
MetricX-23-refA & 0.892 (30) & 0.914 (13) & 0.771 (14) & 0.908 (13) & 0.061 (25) \\
XCOMET-XXL-refA & 0.886 (31) & 0.904 (18) & 0.752 (20) & 0.891 (19) & 0.053 (26) \\
XCOMET-XL-refA & 0.796 (32) & 0.893 (21) & 0.733 (21) & 0.876 (21) & 0.044 (27) \\
tokengram\_F-refA & 0.943 (21) & 0.614 (27) & 0.505 (27) & 0.740 (28) & 0.042 (28) \\
f200spBLEU-refA & 0.912 (26) & 0.532 (31) & 0.429 (30) & 0.705 (31) & 0.039 (29) \\
chrF-refA & 0.945 (19) & 0.596 (28) & 0.467 (29) & 0.726 (29) & 0.037 (30) \\
embed\_llama-refA & 0.924 (25) & 0.514 (32) & 0.410 (32) & 0.708 (30) & 0.030 (31) \\
eBLEU-refA & 0.910 (27) & 0.554 (30) & 0.429 (30) & 0.704 (32) & 0.011 (32) \\
BLEU-refA & 0.764 (33) & 0.671 (26) & 0.524 (26) & 0.753 (26) & 0.003 (33) \\
Random-sysname-src & 0.045 (34) & 0.336 (33) & 0.200 (33) & 0.598 (33) & 0.001 (34) \\
\bottomrule
\end{tabular}
\end{adjustbox}
\caption{Scores and ranks of automatic metrics under each system-level meta-metric for WMT23 en-zh. Metrics are sorted by PPSR; tied scores receive the same rank.}
\label{tab:system_metric_score_ranks_wmt23_en_zh}
\end{table*}

\begin{table*}[t]
\centering
\scriptsize
\setlength{\tabcolsep}{3pt}
\begin{adjustbox}{max width=\textwidth}
\begin{tabular}{lrrrrr}
\toprule
Metric & Group-by-Item $r$ & No-Grouping $r$ & Group-by-System $r$ & PDP & PPSR \\
\midrule
CometKiwi-src & 0.551 (\phantom{0}1) & 0.635 (\phantom{0}1) & 0.396 (\phantom{0}2) & 0.600 (\phantom{0}1) & 0.186 (\phantom{0}1) \\
KG-BERTScore-src & 0.551 (\phantom{0}2) & 0.635 (\phantom{0}2) & 0.396 (\phantom{0}3) & 0.600 (\phantom{0}2) & 0.186 (\phantom{0}2) \\
CometKiwi-XL-src & 0.528 (\phantom{0}3) & 0.588 (\phantom{0}4) & 0.397 (\phantom{0}1) & 0.549 (\phantom{0}4) & 0.160 (\phantom{0}3) \\
MS-COMET-QE-22-src & 0.500 (\phantom{0}7) & 0.610 (\phantom{0}3) & 0.384 (\phantom{0}5) & 0.558 (\phantom{0}3) & 0.157 (\phantom{0}4) \\
COMET-refA & 0.514 (\phantom{0}5) & 0.575 (\phantom{0}6) & 0.384 (\phantom{0}6) & 0.539 (\phantom{0}6) & 0.149 (\phantom{0}5) \\
cometoid22-wmt23-src & 0.527 (\phantom{0}4) & 0.588 (\phantom{0}5) & 0.384 (\phantom{0}4) & 0.545 (\phantom{0}5) & 0.148 (\phantom{0}6) \\
CometKiwi-XXL-src & 0.514 (\phantom{0}6) & 0.559 (\phantom{0}7) & 0.376 (\phantom{0}7) & 0.522 (\phantom{0}7) & 0.145 (\phantom{0}7) \\
BLEURT-20-refA & 0.495 (\phantom{0}8) & 0.550 (\phantom{0}8) & 0.362 (\phantom{0}9) & 0.516 (\phantom{0}8) & 0.140 (\phantom{0}8) \\
cometoid22-wmt22-src & 0.486 (11) & 0.537 (\phantom{0}9) & 0.344 (12) & 0.505 (\phantom{0}9) & 0.126 (\phantom{0}9) \\
XLsim-refA & 0.481 (14) & 0.524 (11) & 0.308 (16) & 0.503 (10) & 0.124 (10) \\
YiSi-1-refA & 0.484 (12) & 0.493 (15) & 0.294 (23) & 0.493 (13) & 0.123 (11) \\
cometoid22-wmt21-src & 0.484 (13) & 0.527 (10) & 0.338 (13) & 0.495 (12) & 0.121 (12) \\
prismRef-refA & 0.490 (\phantom{0}9) & 0.496 (13) & 0.300 (22) & 0.501 (11) & 0.118 (13) \\
MetricX-23-c-refA & 0.465 (17) & 0.507 (12) & 0.292 (24) & 0.476 (14) & 0.115 (14) \\
BERTscore-refA & 0.470 (16) & 0.474 (17) & 0.287 (25) & 0.476 (15) & 0.113 (15) \\
XCOMET-Ensemble-refA & 0.480 (15) & 0.493 (14) & 0.372 (\phantom{0}8) & 0.447 (16) & 0.110 (16) \\
GEMBA-MQM-src & 0.464 (18) & 0.489 (16) & 0.334 (14) & 0.402 (19) & 0.090 (17) \\
XCOMET-QE-Ensemble-src & 0.427 (21) & 0.450 (21) & 0.345 (11) & 0.383 (21) & 0.085 (18) \\
mre-score-labse-regular-refA & 0.487 (10) & 0.177 (32) & 0.105 (32) & 0.427 (17) & 0.085 (19) \\
MetricX-23-QE-b-src & 0.452 (19) & 0.456 (19) & 0.302 (20) & 0.395 (20) & 0.076 (20) \\
MetricX-23-QE-c-src & 0.415 (23) & 0.468 (18) & 0.358 (10) & 0.365 (22) & 0.075 (21) \\
prismSrc-src & 0.383 (26) & 0.452 (20) & 0.228 (26) & 0.419 (18) & 0.074 (22) \\
MetricX-23-QE-src & 0.418 (22) & 0.439 (22) & 0.302 (21) & 0.360 (23) & 0.066 (23) \\
MetricX-23-b-refA & 0.437 (20) & 0.420 (23) & 0.303 (19) & 0.352 (24) & 0.065 (24) \\
MetricX-23-refA & 0.400 (24) & 0.411 (24) & 0.312 (15) & 0.327 (25) & 0.061 (25) \\
XCOMET-XXL-refA & 0.380 (28) & 0.391 (25) & 0.306 (17) & 0.309 (27) & 0.053 (26) \\
XCOMET-XL-refA & 0.343 (30) & 0.366 (26) & 0.305 (18) & 0.268 (30) & 0.044 (27) \\
tokengram\_F-refA & 0.397 (25) & 0.343 (27) & 0.200 (28) & 0.319 (26) & 0.042 (28) \\
f200spBLEU-refA & 0.369 (29) & 0.327 (28) & 0.206 (27) & 0.297 (29) & 0.039 (29) \\
chrF-refA & 0.382 (27) & 0.326 (29) & 0.190 (30) & 0.301 (28) & 0.037 (30) \\
embed\_llama-refA & 0.305 (31) & 0.297 (30) & 0.197 (29) & 0.242 (31) & 0.030 (31) \\
eBLEU-refA & 0.277 (32) & 0.210 (31) & 0.106 (31) & 0.199 (32) & 0.011 (32) \\
BLEU-refA & 0.182 (33) & 0.093 (33) & 0.048 (33) & 0.057 (33) & 0.003 (33) \\
Random-sysname-src & 0.028 (34) & 0.018 (34) & 0.005 (34) & -0.048 (34) & 0.001 (34) \\
\bottomrule
\end{tabular}
\end{adjustbox}
\caption{Scores and ranks of automatic metrics under each segment-level meta-metric for WMT23 en-zh. Metrics are sorted by PPSR; tied scores receive the same rank.}
\label{tab:segment_metric_score_ranks_wmt23_en_zh}
\end{table*}

\begin{table*}[t]
\centering
\scriptsize
\setlength{\tabcolsep}{3pt}
\begin{adjustbox}{max width=\textwidth}
\begin{tabular}{lrrrrr}
\toprule
Metric & $r$ & $\rho$ & $\tau_b$ & SPA & PPSR \\
\midrule
CometKiwi-XXL-src & 0.987 (\phantom{0}2) & 0.975 (\phantom{0}1) & 0.912 (\phantom{0}1) & 0.948 (\phantom{0}2) & 0.101 (\phantom{0}1) \\
COMET-refA & 0.967 (13) & 0.953 (10) & 0.853 (\phantom{0}8) & 0.923 (14) & 0.096 (\phantom{0}2) \\
BLEURT-20-refA & 0.959 (21) & 0.953 (10) & 0.853 (\phantom{0}8) & 0.920 (17) & 0.093 (\phantom{0}3) \\
KG-BERTScore-src & 0.973 (\phantom{0}8) & 0.973 (\phantom{0}3) & 0.897 (\phantom{0}4) & 0.938 (\phantom{0}6) & 0.091 (\phantom{0}4) \\
CometKiwi-src & 0.973 (\phantom{0}9) & 0.973 (\phantom{0}3) & 0.897 (\phantom{0}4) & 0.938 (\phantom{0}7) & 0.091 (\phantom{0}5) \\
MetricX-23-c-refA & 0.962 (18) & 0.961 (\phantom{0}7) & 0.868 (\phantom{0}6) & 0.939 (\phantom{0}5) & 0.087 (\phantom{0}6) \\
cometoid22-wmt23-src & 0.967 (14) & 0.958 (\phantom{0}8) & 0.853 (\phantom{0}8) & 0.924 (11) & 0.086 (\phantom{0}7) \\
CometKiwi-XL-src & 0.986 (\phantom{0}3) & 0.975 (\phantom{0}1) & 0.912 (\phantom{0}1) & 0.949 (\phantom{0}1) & 0.085 (\phantom{0}8) \\
XCOMET-Ensemble-refA & 0.951 (22) & 0.949 (14) & 0.838 (15) & 0.922 (16) & 0.081 (\phantom{0}9) \\
cometoid22-wmt22-src & 0.944 (25) & 0.944 (19) & 0.824 (19) & 0.912 (20) & 0.071 (10) \\
cometoid22-wmt21-src & 0.947 (23) & 0.941 (20) & 0.809 (25) & 0.913 (19) & 0.069 (11) \\
YiSi-1-refA & 0.974 (\phantom{0}7) & 0.941 (20) & 0.838 (15) & 0.924 (13) & 0.066 (12) \\
MS-COMET-QE-22-src & 0.905 (33) & 0.949 (14) & 0.838 (15) & 0.909 (24) & 0.065 (13) \\
prismRef-refA & 0.970 (12) & 0.931 (25) & 0.794 (27) & 0.902 (30) & 0.064 (14) \\
MetricX-23-QE-c-src & 0.966 (15) & 0.953 (10) & 0.853 (\phantom{0}8) & 0.932 (\phantom{0}8) & 0.062 (15) \\
GEMBA-MQM-src & 0.985 (\phantom{0}4) & 0.963 (\phantom{0}6) & 0.868 (\phantom{0}6) & 0.945 (\phantom{0}3) & 0.062 (16) \\
MetricX-23-QE-b-src & 0.977 (\phantom{0}6) & 0.949 (14) & 0.853 (\phantom{0}8) & 0.929 (\phantom{0}9) & 0.058 (17) \\
BERTscore-refA & 0.971 (11) & 0.951 (13) & 0.853 (\phantom{0}8) & 0.924 (12) & 0.057 (18) \\
XLsim-refA & 0.987 (\phantom{0}1) & 0.968 (\phantom{0}5) & 0.912 (\phantom{0}1) & 0.939 (\phantom{0}4) & 0.056 (19) \\
XCOMET-XXL-refA & 0.942 (26) & 0.934 (24) & 0.794 (27) & 0.906 (28) & 0.054 (20) \\
XCOMET-QE-Ensemble-src & 0.944 (24) & 0.949 (14) & 0.838 (15) & 0.925 (10) & 0.054 (21) \\
MetricX-23-b-refA & 0.940 (29) & 0.931 (25) & 0.779 (31) & 0.902 (29) & 0.048 (22) \\
MetricX-23-QE-src & 0.941 (28) & 0.939 (22) & 0.809 (25) & 0.910 (22) & 0.045 (23) \\
chrF-refA & 0.965 (16) & 0.931 (25) & 0.824 (19) & 0.912 (21) & 0.044 (24) \\
XCOMET-XL-refA & 0.926 (30) & 0.946 (18) & 0.824 (19) & 0.908 (25) & 0.043 (25) \\
tokengram\_F-refA & 0.965 (17) & 0.931 (25) & 0.824 (19) & 0.913 (18) & 0.042 (26) \\
MetricX-23-refA & 0.918 (31) & 0.919 (31) & 0.765 (32) & 0.892 (31) & 0.042 (27) \\
mre-score-labse-regular-refA & 0.981 (\phantom{0}5) & 0.956 (\phantom{0}9) & 0.853 (\phantom{0}8) & 0.922 (15) & 0.032 (28) \\
eBLEU-refA & 0.961 (20) & 0.936 (23) & 0.824 (19) & 0.909 (23) & 0.028 (29) \\
f200spBLEU-refA & 0.962 (19) & 0.926 (29) & 0.824 (19) & 0.906 (27) & 0.021 (30) \\
embed\_llama-refA & 0.971 (10) & 0.926 (29) & 0.794 (27) & 0.908 (26) & 0.016 (31) \\
BLEU-refA & 0.942 (27) & 0.919 (31) & 0.794 (27) & 0.891 (32) & 0.016 (32) \\
MaTESe-refA & 0.911 (32) & 0.895 (33) & 0.735 (33) & 0.884 (33) & 0.013 (33) \\
prismSrc-src & -0.773 (35) & -0.706 (35) & -0.559 (35) & 0.240 (35) & 0.006 (34) \\
Random-sysname-src & 0.288 (34) & 0.279 (34) & 0.191 (34) & 0.605 (34) & 0.001 (35) \\
\bottomrule
\end{tabular}
\end{adjustbox}
\caption{Scores and ranks of automatic metrics under each system-level meta-metric for WMT23 ja-en. Metrics are sorted by PPSR; tied scores receive the same rank.}
\label{tab:system_metric_score_ranks_wmt23_ja_en}
\end{table*}

\begin{table*}[t]
\centering
\scriptsize
\setlength{\tabcolsep}{3pt}
\begin{adjustbox}{max width=\textwidth}
\begin{tabular}{lrrrrr}
\toprule
Metric & Group-by-Item $r$ & No-Grouping $r$ & Group-by-System $r$ & PDP & PPSR \\
\midrule
CometKiwi-XXL-src & 0.349 (\phantom{0}1) & 0.474 (\phantom{0}1) & 0.411 (\phantom{0}1) & 0.352 (\phantom{0}1) & 0.101 (\phantom{0}1) \\
COMET-refA & 0.336 (\phantom{0}4) & 0.445 (\phantom{0}5) & 0.386 (\phantom{0}4) & 0.341 (\phantom{0}2) & 0.096 (\phantom{0}2) \\
BLEURT-20-refA & 0.340 (\phantom{0}3) & 0.436 (\phantom{0}6) & 0.383 (\phantom{0}6) & 0.335 (\phantom{0}3) & 0.093 (\phantom{0}3) \\
KG-BERTScore-src & 0.327 (\phantom{0}6) & 0.455 (\phantom{0}2) & 0.390 (\phantom{0}2) & 0.334 (\phantom{0}4) & 0.091 (\phantom{0}4) \\
CometKiwi-src & 0.326 (\phantom{0}7) & 0.455 (\phantom{0}3) & 0.390 (\phantom{0}3) & 0.334 (\phantom{0}5) & 0.091 (\phantom{0}5) \\
MetricX-23-c-refA & 0.311 (12) & 0.342 (22) & 0.265 (25) & 0.325 (\phantom{0}6) & 0.087 (\phantom{0}6) \\
cometoid22-wmt23-src & 0.320 (\phantom{0}9) & 0.435 (\phantom{0}7) & 0.379 (\phantom{0}7) & 0.323 (\phantom{0}8) & 0.086 (\phantom{0}7) \\
CometKiwi-XL-src & 0.332 (\phantom{0}5) & 0.446 (\phantom{0}4) & 0.383 (\phantom{0}5) & 0.325 (\phantom{0}7) & 0.085 (\phantom{0}8) \\
XCOMET-Ensemble-refA & 0.342 (\phantom{0}2) & 0.410 (12) & 0.351 (10) & 0.316 (\phantom{0}9) & 0.081 (\phantom{0}9) \\
cometoid22-wmt22-src & 0.296 (18) & 0.432 (\phantom{0}8) & 0.374 (\phantom{0}8) & 0.293 (10) & 0.071 (10) \\
cometoid22-wmt21-src & 0.292 (19) & 0.431 (\phantom{0}9) & 0.372 (\phantom{0}9) & 0.290 (11) & 0.069 (11) \\
YiSi-1-refA & 0.307 (14) & 0.383 (15) & 0.332 (13) & 0.276 (16) & 0.066 (12) \\
MS-COMET-QE-22-src & 0.274 (24) & 0.388 (14) & 0.332 (14) & 0.274 (17) & 0.065 (13) \\
prismRef-refA & 0.296 (17) & 0.351 (19) & 0.302 (18) & 0.285 (14) & 0.064 (14) \\
MetricX-23-QE-c-src & 0.317 (10) & 0.418 (11) & 0.347 (11) & 0.287 (12) & 0.062 (15) \\
GEMBA-MQM-src & 0.323 (\phantom{0}8) & 0.421 (10) & 0.340 (12) & 0.285 (13) & 0.062 (16) \\
MetricX-23-QE-b-src & 0.315 (11) & 0.383 (16) & 0.309 (17) & 0.277 (15) & 0.058 (17) \\
BERTscore-refA & 0.285 (23) & 0.357 (17) & 0.311 (16) & 0.252 (21) & 0.057 (18) \\
XLsim-refA & 0.256 (28) & 0.342 (23) & 0.292 (20) & 0.263 (19) & 0.056 (19) \\
XCOMET-XXL-refA & 0.306 (15) & 0.352 (18) & 0.293 (19) & 0.264 (18) & 0.054 (20) \\
XCOMET-QE-Ensemble-src & 0.301 (16) & 0.388 (13) & 0.332 (15) & 0.260 (20) & 0.054 (21) \\
MetricX-23-b-refA & 0.310 (13) & 0.343 (21) & 0.285 (21) & 0.249 (22) & 0.048 (22) \\
MetricX-23-QE-src & 0.286 (21) & 0.344 (20) & 0.283 (22) & 0.239 (23) & 0.045 (23) \\
chrF-refA & 0.258 (26) & 0.292 (26) & 0.248 (26) & 0.234 (24) & 0.044 (24) \\
XCOMET-XL-refA & 0.285 (22) & 0.327 (25) & 0.276 (24) & 0.234 (25) & 0.043 (25) \\
tokengram\_F-refA & 0.257 (27) & 0.290 (27) & 0.246 (27) & 0.230 (27) & 0.042 (26) \\
MetricX-23-refA & 0.290 (20) & 0.332 (24) & 0.279 (23) & 0.232 (26) & 0.042 (27) \\
mre-score-labse-regular-refA & 0.266 (25) & 0.186 (33) & 0.157 (34) & 0.169 (29) & 0.032 (28) \\
eBLEU-refA & 0.212 (29) & 0.202 (32) & 0.166 (33) & 0.183 (28) & 0.028 (29) \\
f200spBLEU-refA & 0.202 (30) & 0.226 (29) & 0.192 (31) & 0.163 (30) & 0.021 (30) \\
embed\_llama-refA & 0.147 (33) & 0.203 (31) & 0.176 (32) & 0.135 (32) & 0.016 (31) \\
BLEU-refA & 0.186 (32) & 0.221 (30) & 0.192 (30) & 0.140 (31) & 0.016 (32) \\
MaTESe-refA & 0.199 (31) & 0.242 (28) & 0.195 (29) & 0.134 (33) & 0.013 (33) \\
prismSrc-src & 0.008 (35) & 0.171 (34) & 0.203 (28) & 0.034 (34) & 0.006 (34) \\
Random-sysname-src & 0.071 (34) & 0.061 (35) & 0.003 (35) & 0.033 (35) & 0.001 (35) \\
\bottomrule
\end{tabular}
\end{adjustbox}
\caption{Scores and ranks of automatic metrics under each segment-level meta-metric for WMT23 ja-en. Metrics are sorted by PPSR; tied scores receive the same rank.}
\label{tab:segment_metric_score_ranks_wmt23_ja_en}
\end{table*}

\begin{table*}[t]
\centering
\scriptsize
\setlength{\tabcolsep}{3pt}
\begin{adjustbox}{max width=\textwidth}
\begin{tabular}{lrrrrr}
\toprule
Metric & $r$ & $\rho$ & $\tau_b$ & SPA & PPSR \\
\midrule
MetricX-24-refA & 0.924 (\phantom{0}6) & 0.900 (\phantom{0}5) & 0.745 (\phantom{0}8) & 0.890 (11) & 0.124 (\phantom{0}1) \\
PrismRefMedium-refA & 0.949 (\phantom{0}2) & 0.909 (\phantom{0}4) & 0.782 (\phantom{0}4) & 0.926 (\phantom{0}2) & 0.114 (\phantom{0}2) \\
COMET-22-refA & 0.921 (\phantom{0}8) & 0.973 (\phantom{0}1) & 0.891 (\phantom{0}1) & 0.917 (\phantom{0}3) & 0.112 (\phantom{0}3) \\
BLEURT-20-refA & 0.981 (\phantom{0}1) & 0.973 (\phantom{0}1) & 0.891 (\phantom{0}1) & 0.960 (\phantom{0}1) & 0.111 (\phantom{0}4) \\
PrismRefSmall-refA & 0.931 (\phantom{0}3) & 0.900 (\phantom{0}5) & 0.745 (\phantom{0}8) & 0.906 (\phantom{0}6) & 0.110 (\phantom{0}5) \\
MetricX-24-Hybrid-refA & 0.898 (11) & 0.882 (10) & 0.709 (10) & 0.883 (12) & 0.106 (\phantom{0}6) \\
metametrics\_mt\_mqm\_kendall-refA & 0.801 (16) & 0.773 (15) & 0.600 (15) & 0.821 (16) & 0.104 (\phantom{0}7) \\
metametrics\_mt\_mqm\_hybrid\_kendall-refA & 0.804 (15) & 0.773 (15) & 0.600 (15) & 0.821 (15) & 0.103 (\phantom{0}8) \\
chrfS-refA & 0.908 (\phantom{0}9) & 0.900 (\phantom{0}5) & 0.782 (\phantom{0}4) & 0.896 (\phantom{0}8) & 0.092 (\phantom{0}9) \\
chrF-refA & 0.931 (\phantom{0}4) & 0.900 (\phantom{0}5) & 0.782 (\phantom{0}4) & 0.901 (\phantom{0}7) & 0.079 (10) \\
BERTScore-refA & 0.819 (13) & 0.809 (14) & 0.673 (14) & 0.826 (14) & 0.079 (11) \\
damonmonli-refA & 0.865 (12) & 0.855 (11) & 0.709 (10) & 0.892 (\phantom{0}9) & 0.077 (12) \\
YiSi-1-refA & 0.925 (\phantom{0}5) & 0.955 (\phantom{0}3) & 0.855 (\phantom{0}3) & 0.914 (\phantom{0}4) & 0.072 (13) \\
gemba\_esa-src & 0.785 (17) & 0.745 (18) & 0.564 (18) & 0.777 (20) & 0.069 (14) \\
XCOMET-refA & 0.766 (19) & 0.709 (20) & 0.527 (19) & 0.782 (19) & 0.063 (15) \\
monmonli-refA & 0.818 (14) & 0.845 (12) & 0.709 (10) & 0.848 (13) & 0.055 (16) \\
spBLEU-refA & 0.922 (\phantom{0}7) & 0.900 (\phantom{0}5) & 0.782 (\phantom{0}4) & 0.912 (\phantom{0}5) & 0.055 (17) \\
MetricX-24-Hybrid-QE-src & 0.748 (20) & 0.736 (19) & 0.527 (19) & 0.790 (18) & 0.045 (18) \\
MetricX-24-QE-src & 0.776 (18) & 0.773 (15) & 0.600 (15) & 0.820 (17) & 0.044 (19) \\
BLEU-refA & 0.899 (10) & 0.845 (12) & 0.709 (10) & 0.892 (10) & 0.043 (20) \\
XCOMET-QE-src & 0.566 (22) & 0.482 (22) & 0.345 (22) & 0.698 (22) & 0.027 (21) \\
CometKiwi-src & 0.469 (23) & 0.427 (23) & 0.273 (23) & 0.683 (23) & 0.025 (22) \\
CometKiwi-XXL-src & 0.699 (21) & 0.627 (21) & 0.455 (21) & 0.774 (21) & 0.022 (23) \\
metametrics\_mt\_mqm\_qe\_kendall.seg.s-src & 0.330 (24) & 0.145 (24) & 0.127 (24) & 0.595 (24) & 0.020 (24) \\
XLsimDA-src & -0.318 (25) & -0.009 (25) & -0.018 (25) & 0.475 (25) & 0.009 (25) \\
\bottomrule
\end{tabular}
\end{adjustbox}
\caption{Scores and ranks of automatic metrics under each system-level meta-metric for WMT24 cs-uk. Metrics are sorted by PPSR; tied scores receive the same rank.}
\label{tab:system_metric_score_ranks_wmt24_cs_uk}
\end{table*}

\begin{table*}[t]
\centering
\scriptsize
\setlength{\tabcolsep}{3pt}
\begin{adjustbox}{max width=\textwidth}
\begin{tabular}{lrrrrr}
\toprule
Metric & Group-by-Item $r$ & No-Grouping $r$ & Group-by-System $r$ & PDP & PPSR \\
\midrule
MetricX-24-refA & 0.264 (\phantom{0}4) & 0.576 (\phantom{0}1) & 0.564 (\phantom{0}1) & 0.362 (\phantom{0}1) & 0.124 (\phantom{0}1) \\
PrismRefMedium-refA & 0.229 (\phantom{0}9) & 0.522 (\phantom{0}5) & 0.508 (\phantom{0}7) & 0.343 (\phantom{0}2) & 0.114 (\phantom{0}2) \\
COMET-22-refA & 0.245 (\phantom{0}5) & 0.548 (\phantom{0}2) & 0.538 (\phantom{0}2) & 0.340 (\phantom{0}4) & 0.112 (\phantom{0}3) \\
BLEURT-20-refA & 0.240 (\phantom{0}7) & 0.536 (\phantom{0}4) & 0.521 (\phantom{0}6) & 0.341 (\phantom{0}3) & 0.111 (\phantom{0}4) \\
PrismRefSmall-refA & 0.221 (10) & 0.520 (\phantom{0}6) & 0.507 (\phantom{0}8) & 0.336 (\phantom{0}5) & 0.110 (\phantom{0}5) \\
MetricX-24-Hybrid-refA & 0.266 (\phantom{0}3) & 0.544 (\phantom{0}3) & 0.531 (\phantom{0}3) & 0.335 (\phantom{0}6) & 0.106 (\phantom{0}6) \\
metametrics\_mt\_mqm\_kendall-refA & 0.273 (\phantom{0}2) & 0.519 (\phantom{0}7) & 0.523 (\phantom{0}4) & 0.323 (\phantom{0}7) & 0.104 (\phantom{0}7) \\
metametrics\_mt\_mqm\_hybrid\_kendall-refA & 0.273 (\phantom{0}1) & 0.519 (\phantom{0}8) & 0.523 (\phantom{0}5) & 0.322 (\phantom{0}8) & 0.103 (\phantom{0}8) \\
chrfS-refA & 0.200 (12) & 0.496 (11) & 0.488 (\phantom{0}9) & 0.308 (\phantom{0}9) & 0.092 (\phantom{0}9) \\
chrF-refA & 0.196 (13) & 0.424 (16) & 0.413 (18) & 0.286 (10) & 0.079 (10) \\
BERTScore-refA & 0.188 (16) & 0.421 (18) & 0.416 (16) & 0.283 (12) & 0.079 (11) \\
damonmonli-refA & 0.183 (17) & 0.498 (\phantom{0}9) & 0.486 (11) & 0.285 (11) & 0.077 (12) \\
YiSi-1-refA & 0.209 (11) & 0.468 (13) & 0.461 (12) & 0.271 (13) & 0.072 (13) \\
gemba\_esa-src & 0.245 (\phantom{0}6) & 0.470 (12) & 0.449 (13) & 0.269 (14) & 0.069 (14) \\
XCOMET-refA & 0.229 (\phantom{0}8) & 0.497 (10) & 0.487 (10) & 0.257 (15) & 0.063 (15) \\
monmonli-refA & 0.176 (19) & 0.456 (14) & 0.444 (14) & 0.244 (16) & 0.055 (16) \\
spBLEU-refA & 0.178 (18) & 0.328 (22) & 0.318 (22) & 0.239 (17) & 0.055 (17) \\
MetricX-24-Hybrid-QE-src & 0.194 (15) & 0.396 (19) & 0.385 (19) & 0.217 (18) & 0.045 (18) \\
MetricX-24-QE-src & 0.195 (14) & 0.424 (17) & 0.413 (17) & 0.215 (19) & 0.044 (19) \\
BLEU-refA & 0.149 (23) & 0.277 (24) & 0.268 (24) & 0.209 (20) & 0.043 (20) \\
XCOMET-QE-src & 0.163 (21) & 0.371 (20) & 0.361 (20) & 0.151 (21) & 0.027 (21) \\
CometKiwi-src & 0.163 (22) & 0.424 (15) & 0.422 (15) & 0.140 (22) & 0.025 (22) \\
CometKiwi-XXL-src & 0.171 (20) & 0.348 (21) & 0.335 (21) & 0.139 (23) & 0.022 (23) \\
metametrics\_mt\_mqm\_qe\_kendall.seg.s-src & 0.091 (24) & 0.294 (23) & 0.289 (23) & 0.123 (24) & 0.020 (24) \\
XLsimDA-src & 0.037 (25) & 0.005 (25) & 0.009 (25) & 0.072 (25) & 0.009 (25) \\
\bottomrule
\end{tabular}
\end{adjustbox}
\caption{Scores and ranks of automatic metrics under each segment-level meta-metric for WMT24 cs-uk. Metrics are sorted by PPSR; tied scores receive the same rank.}
\label{tab:segment_metric_score_ranks_wmt24_cs_uk}
\end{table*}

\end{document}